\documentclass[10pt]{article} %
\usepackage[accepted]{tmlr}

\usepackage{amsmath,amsfonts,bm}

\def\eqref#1{equation~\ref{#1}}

\def\1{\bm{1}}

\DeclareMathAlphabet{\mathsfit}{\encodingdefault}{\sfdefault}{m}{sl}
\SetMathAlphabet{\mathsfit}{bold}{\encodingdefault}{\sfdefault}{bx}{n}

\DeclareMathOperator*{\argmin}{arg\,min}

\DeclareMathOperator{\sign}{sign}

\usepackage[utf8]{inputenc} %
\usepackage[T1]{fontenc}    %
\usepackage{hyperref}       %
\usepackage{url}            %
\usepackage{booktabs}       %
\usepackage{amsfonts}       %
\usepackage{nicefrac}       %
\usepackage{microtype}      %
\usepackage{xcolor}         %

\usepackage{amsmath}
\usepackage{amssymb}
\usepackage{mathtools}
\usepackage{amsthm}

\usepackage{soul}
\usepackage{outlines}
\usepackage{outline}

\usepackage[capitalize,noabbrev]{cleveref}

\usepackage[textwidth=4.cm, textsize=tiny]{todonotes}

\usepackage{graphicx, multirow}

\theoremstyle{plain}
\newtheorem{theorem}{Theorem}[section]

\newtheorem{lemma}[theorem]{Lemma}
\newtheorem{corollary}[theorem]{Corollary}
\theoremstyle{definition}
\newtheorem{definition}[theorem]{Definition}

\usepackage[textsize=tiny]{todonotes}

\usepackage{caption}
\usepackage{subcaption}
\usepackage{wrapfig}

\usepackage{booktabs}

\usepackage[title]{appendix}

\usepackage{pifont}%

\usepackage{hyperref}
\usepackage{url}
\usepackage{dsfont}

\title{Revisiting Non-separable Binary Classification and its\\Applications in Anomaly Detection}

\author{\name Matthew Lau \email mlau40@gatech.edu  \\
      \addr School of Cybersecurity and Privacy, Georgia Institute of Technology
      \AND
      \name Ismaila Seck \email 
      ismaila@lengo.ai
\\
      \addr Lengo AI
      \AND
      \name Athanasios P  Meliopoulos \email sakis.m@gatech.edu\\
      \addr School of Electrical and Computer Engineering, Georgia Institute of Technology \\
            \AND
      \name Wenke Lee \email wenke@cc.gatech.edu\\
      \addr School of Cybersecurity and Privacy, Georgia Institute of Technology \\
            \AND
      \name Eugene Ndiaye \email e\_ndiaye@apple.com\\
      \addr Apple}

\def\month{06}  %
\def\year{2024} %
\def\openreview{\url{https://openreview.net/forum?id=zOJ846BXhl}} %

\begin{document}

\maketitle

\begin{abstract}
The inability to linearly classify \texttt{XOR} has motivated much of deep learning.
We revisit this age-old problem and show that \textit{linear} classification of \texttt{XOR} is indeed possible.
Instead of separating data between halfspaces, we propose a slightly different paradigm, \texttt{equality separation}, that adapts the SVM objective to distinguish data within or outside the margin.
Our classifier can then be integrated into neural network pipelines with a smooth approximation.
From its properties, we intuit that equality separation is suitable for anomaly detection.
To formalize this notion, we introduce \textit{closing numbers}, a quantitative measure on the capacity for classifiers to form closed decision regions for anomaly detection.
Springboarding from this theoretical connection between binary classification and anomaly detection, we test our hypothesis on supervised anomaly detection experiments, showing that equality separation can detect both seen and unseen anomalies.
\end{abstract}

\section{Introduction}

\paragraph{Linear classification of a non-linearly separable dataset} 
A common belief in machine learning is that linear models (like the perceptron \citep{PerceptronsOriginal}) cannot model some simple functions, as shown by the example \texttt{XOR} (exclusive OR) in classic textbooks (e.g. \citet{goodfellow2016deeplearning}).
This failure justifies 
using non-linear transforms such as neural networks (NNs).
We reexamine this counter-example:
\begin{center}
\vspace{-1mm}
    \texttt{XOR} can be learnt by linear models: a single neuron without hidden layers or feature representations.
\vspace{-1mm}
\end{center}

Consider learning
logical functions \texttt{AND}, \texttt{OR}, and \texttt{XOR} shown in \Cref{fig:LogicalFunction}. 
For 
\texttt{AND} and \texttt{OR}, 
there is 
a separating line where 
half of the space contains only points with label 0 and the other half only points with label 1.
However, there is no separating line for \texttt{XOR}.
The most popular is to non-linearly transform the input via kernels \citep{xor_kernel} or non-linear activation functions with hidden layers \citep{MinMLPforXOR}
(e.g. Figure \ref{fig:XOR_NN}).
Conversely, \citet{nonmonotonic_activations_mlp_thesis} 
modified
the \textit{output} representation 
non-linearly to solve \texttt{XOR}.
\looseness=-1

However, resorting to non-linearities is not necessary.
We propose using the line not to separate the input space%
, but to 
tell if a class is either on the line or not.
This \textit{linear} classifier 
can perfectly model the three logical functions above, including \texttt{XOR} (\Cref{fig:LogicalFunction}). 
We name this paradigm as \textit{equality separation}.
In fact, it can solve linearly separable and non-separable problems, demonstrating its flexibility (e.g. Figure \ref{fig:gaussian_equality_separation}, \ref{fig:inseparable_equality_separation}).
From the asymmetry of its classification rule, we posit its utility for anomaly detection (AD).
\looseness=-1
\begin{figure}[t]
  \centering
  \begin{subfigure}[t]{0.32\textwidth}
      \centering
      \includegraphics[width=0.99\textwidth]{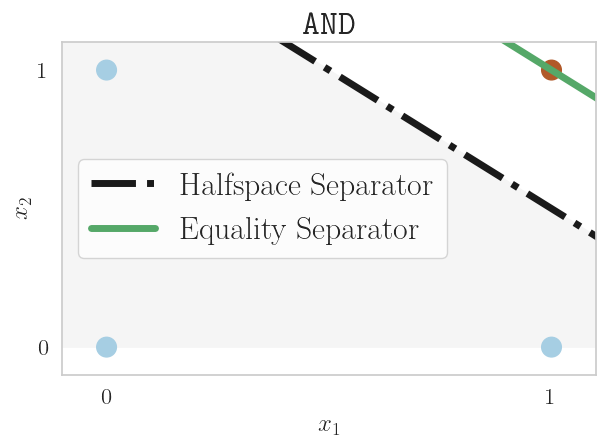}
      \caption{AND is possible.}
      \label{fig:AND}
  \end{subfigure}
  \begin{subfigure}[t]{0.32\textwidth}
      \centering
      \includegraphics[width=0.99\textwidth]{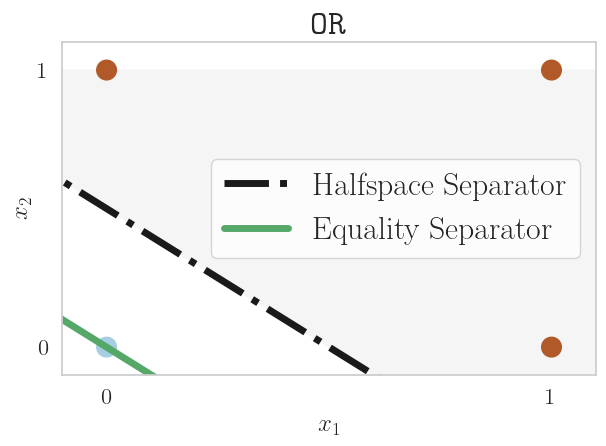}
      \caption{OR is possible.}
      \label{fig:OR}
  \end{subfigure}
  \begin{subfigure}[t]{0.32\textwidth}
      \centering
      \includegraphics[width=0.99\textwidth]{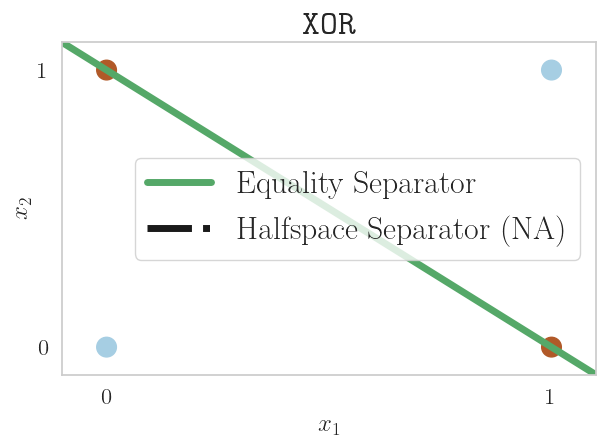}
      \caption{XOR only with equality sep.}
      \label{fig:XOR}
  \end{subfigure}
  \begin{subfigure}[t]{0.32\textwidth}
      \centering
      \includegraphics[height=0.1415\textheight, width=0.9\textwidth]
      {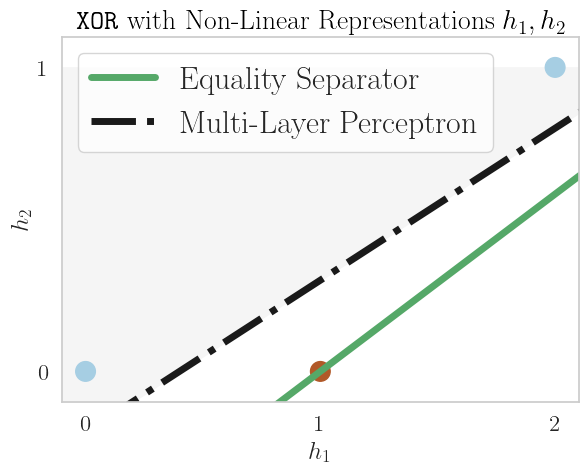}
      \caption{XOR with MLP is possible.}
      \label{fig:XOR_NN}
  \end{subfigure}
  \begin{subfigure}[t]{0.32\textwidth}
      \centering
      \includegraphics[width=0.99\textwidth]
      {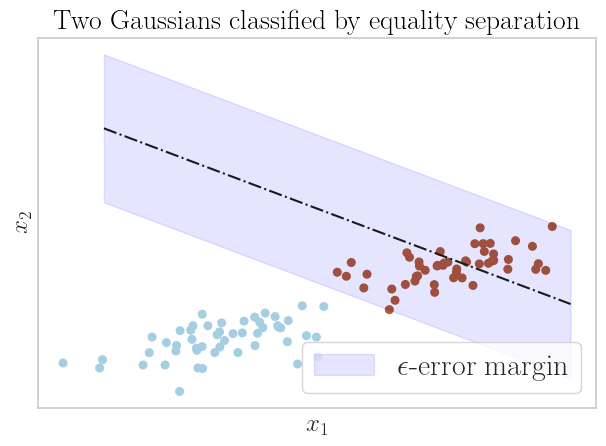}
      \caption{Halfspace separable data.}
      \label{fig:gaussian_equality_separation}
  \end{subfigure}
  \begin{subfigure}[t]{0.32\textwidth}
      \centering
      \includegraphics[width=0.99\textwidth]
      {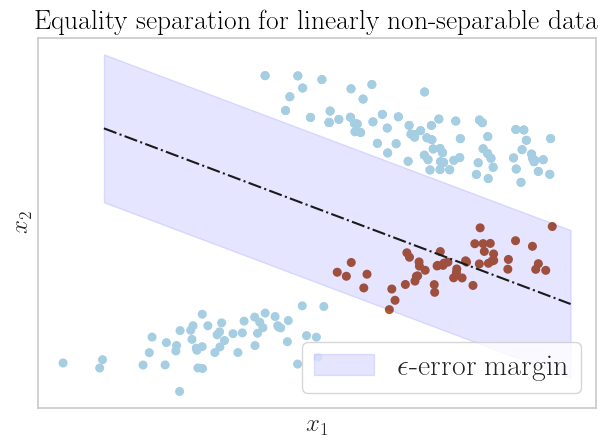}
      \caption{Halfspace non-separable data.}
      \label{fig:inseparable_equality_separation}
  \end{subfigure}
\caption{Linear classification by halfspace separators and equality separators for logical functions (Figures \ref{fig:AND}-\ref{fig:XOR}) and with 1 hidden layer with hidden units $h_1,h_2$ (Figure \ref{fig:XOR_NN}).
In general, equality separation can classify linearly separable (Figure \ref{fig:gaussian_equality_separation}) and non-separable (Figure \ref{fig:inseparable_equality_separation}) data.}
\label{fig:LogicalFunction}
\end{figure}

\paragraph{Anomaly detection}
is paramount for detecting system failure.
There may be some limited examples of failures (anomalies), but we usually do not have access to all possible types of anomalies.
For instance, cyber defenders want to detect known/seen and unknown/unseen (`zero-day') attacks.
The most direct way to train an AD model is via
binary classification with empirical risk minimization (ERM)%
, with normal data and anomalies as the two classes.
However, classic binary classification principles break when unseen anomalies do not share features with both classes of normal and seen anomalies \citep{DeepSAD}.
Rather than conventionally using the hyperplane as a separator, we posit that using it to model the manifold of normal data is more suitable:
points on/close to the line are
normal while points far from the line are anomalies.
To formalize this idea of ``suitability'', we introduce \textit{closing numbers} to understand a binary classifier's ability to detect anomalies.
Based on their suitability,
we can integrate them into NNs for AD via standard ERM.
\looseness=-1

\paragraph{Contributions}

We propose the \textit{equality separation} paradigm. Equality separators use the distance from a datum to a learnt hyperplane as the classification metric. We summarize our contributions:
\begin{itemize}
    \item 
    As a linear classifier, they have twice the VC dimension of halfspace separators.
    A consequence is its ability to \textit{linearly} classify \texttt{XOR}, which has never been done before.
    Showing the connection with the smooth \textit{bump activation}, 
    equality separation
    can be used in NNs and trained with gradient descent. \looseness=-1
    \item We introduce \textit{closing numbers}, a quantitative metric of the capacity of hypothesis classes and activation functions to form closed decision regions.
    We corroborate with this theory and empirically observe that
    equality separation can achieve a balance in detecting both known and unknown anomalies in one-class settings.
\end{itemize}

\section{Equality separator: binary classification revisited}

A standard classification problem is as follows.
Let $\mathcal X\subseteq \mathbb R^n$ be the feature space, $\mathcal Y:=\{0, 1\}$  be the label space and $\mathcal H 
$ be a hypothesis class.
Given the true function $c:\mathcal X \to \mathcal Y$ (i.e. the \textit{concept}),
we aim to find hypothesis $h\in \mathcal H$ that reduces the generalization error $L(h):= \mathbb P
(h(\mathbf x)\neq c(\mathbf x))$ for $\mathbf x \in \mathcal X$.
\looseness=-1

Linear predictors \citep[Chapter 9]{understanding_ml_book} are the simplest hypothesis classes we can consider.
In $\mathbb R^n$, the class of linear predictors comprises a set of affine functions
$$A_n = \{\mathbf x \mapsto \mathbf w^T\mathbf x+b: \mathbf w\in \mathbb R^n, b\in \mathbb R\}$$
and a labeling function $\phi:\mathbb R \to \mathcal Y$. The linear predictor class is
$$\phi \circ A_n = \{\mathbf x\mapsto \phi (h_{\mathbf w, b}(\mathbf x)): h_{\mathbf w, b} \in A_n\}.$$
Halfspace separators as in the perceptron use the sign function as $\phi$
to split the space $\mathbb R^n$ into two regions with a hyperplane: $\text{sign}(z) = 1$ if $z \geq 0$, and $\text{sign}(z) = 0$ otherwise.
Equivalently, $\text{sign}(z)=\mathds{1}(z\geq 0)=\mathds{1}_{\mathbb R^+ \cup\{0\}}(z)$ for indicator function $\mathds{1}$.
A critical issue in classification is that halfspace separators cannot achieve zero generalization error when the input space $\mathcal X$ is linearly inseparable \citep{PerceptronsOriginal}.
To resolve this, we propose to edit the set for the indicator function of $\phi$.
\looseness=-1

The core change is as follows: we change the labeling from $\phi(z)=\sign(z)=\mathds{1}(z\geq 0)$ to $\phi(z)=\mathds{1}(z=0)$ or $\phi(z)=\mathds{1}(z \text{ close to } 0)$.
We introduce two notions of equality separators: the \textit{strict equality separator} and the generalized, robust \textit{$\epsilon$-error separator}.
We define the first type as follows:
\looseness=-1

\begin{definition}
\label{def:strict_equality_separator}
A \textit{strict equality separator} is a decision rule where the hypothesis class is of the form 
\begin{equation}
\label{eqn:strict_equality_sep}
   \mathcal H=\{\mathbf x \longmapsto \mathds{1}(h_{\mathbf w, b}(\mathbf x) = 0): h_{\mathbf w, b} \in A_n\}\cup \{\mathbf x \longmapsto \mathds{1}(h_{\mathbf w, b}(\mathbf x) \neq 0): h_{\mathbf w, b} \in A_n\}. 
\end{equation}
\end{definition}

Geometrically, the strict equality separator checks if a datum is on or off some hyperplane i.e. $\phi = \mathds{1}_{\{0\}}$ or $\phi=\mathds{1}_{\mathbb R\backslash \{0\}}$.
In halfspace separation, labels are arbitrary, and we capture this with the union of 2 classes in (\ref{eqn:strict_equality_sep}).
\looseness=-1

However, a hyperplane has zero volume and sampled points may be noisy.
To increase its robustness,
we can allow some error for a datum within $\epsilon$ distance from a hyperplane.
To incorporate this $\epsilon$-error margin, we swap mapping $\phi$ from $\mathds{1}_{\{0\}}$ to $\mathds{1}_{S_{\epsilon}}$ for the set $S_{\epsilon}=\{z\in \mathbb R: -\epsilon \leq z \leq \epsilon\}$ for $\epsilon > 0$:
\begin{definition}
\label{def:epsilon_error_separator}
The \textit{$\epsilon$-error separator} is a decision rule where the hypothesis class is of the form 
$$\mathcal H_\epsilon=\{\mathbf x \longmapsto \mathds{1}_{S_\epsilon}(h_{\mathbf w, b}(\mathbf x)): h_{\mathbf w, b} \in A_n\} \cup \{\mathbf x \longmapsto \mathds{1}_{\mathbb R \backslash S_\epsilon}(h_{\mathbf w, b}(\mathbf x)): h_{\mathbf w, b} \in A_n\} \quad \text{ for } 
\epsilon > 0.$$
\end{definition}

The $\epsilon$-error separator modifies the SVM objective.
In classification, SVMs produce the optimal separating hyperplane such that no training examples are in the $\epsilon$-sized margin (up to some slack) \citep{SVM}.
Meanwhile, the margin in $\epsilon$-error separators is reserved for one class, while the region outside is for the other.
With this striking example, we proceed to analyze its expressivity.

\subsection{VC dimension}
\label{section:vc_dim}
Vapnik–Chervonenkis (VC) dimension \citep{VC_Dim_Original} measures the complexity of a hypothesis class to give bounds on its generalization error.
The VC dimension of a halfspace separator in $\mathbb R^n$ is $n+1$ \citep[Theorem 9.2]{understanding_ml_book}.
We claim that equality separators are not merely an XOR solver, but more expressive than halfspace separators in general:
\begin{theorem}
\label{theorem:vc_strict_equality}
For hypothesis class of strict equality separators $\mathcal{H}$ in Def. \ref{def:strict_equality_separator},
$\mathrm{VCdim}(\mathcal H) = 2n+1$.
\end{theorem}
We defer the proof to Appendix \ref{proof:vc} and build intuition in $\mathbb R^2$.
We draw a line through the class with the fewest points -- the union of indicator function classes allows us to.
This is always possible with 5 non-colinear points: the maximum number of points that the class with the minimum points has is 2, which perfectly determines a line.
6 points is not possible: a labeling with 3 points for each class admits an overdetermined system which is unsolvable in general.
Extrapolating to higher dimensions, the VC dimension is 1 more than $2n$. \looseness=-1

\begin{corollary}
\label{theorem:vc_eps_err_sep}
For $\epsilon$-error separators $\mathcal{H}_{\epsilon}$ in Def. \ref{def:epsilon_error_separator},
$2n+1\leq \mathrm{VCdim}(\mathcal H_\epsilon) \leq 2n+3 
$.
\end{corollary}

The lower-bound relies on Theorem \ref{theorem:vc_strict_equality} while the upper bound is constructed by the union of hypothesis classes.
Details are in Appendix \ref{proof:vc}.
We explore the significance of more expressivity in Section \ref{section:AD_data_structure}.
\looseness=-1

\subsection{Related works from classification}
\label{section:halfspace_connnection}

\citet{nonmonotonic_activations_mlp_thesis} designs hyper-hill classifiers that are NNs with bump activations, such as the Gaussian function
\begin{equation}
\label{eqn:bump_activation}
B(z, \mu, \sigma) = \exp \left [-\frac{1}{2}\left ( \frac{z-\mu}{\sigma} \right )^2 \right ]
\end{equation}
where $z$ is the input into the activation function and $\mu,\sigma$ are parameters which can be learnt.

Our contributions are two-fold here.
First, given knowledge of the class in the margin, we explicitly derive that equality separators are the linear, discrete representation of the bump activation.
Discrete classifiers are ubiquitously obtained by learning a continuous (activation) function and then thresholding the output to convert the continuous score to a classification decision (e.g. modeling halfspace separation with sigmoid activation).
This way, we can model equality separation for one class using bump activations.
Second, we give the generalization of the linear model which allows label flipping.

Label flipping is a perspective from which we can understand the doubled VC dimension.
Switching labels does not fundamentally change the optimization objective in halfspace separation due to its symmetry.
The asymmetry of equality separators yields a different optimization objective if labels are swapped -- the goal is to obtain the margin of another class instead.
This increases expressivity, making label flipping non-vacuous.
For instance, without label flipping, it is impossible to classify the function \texttt{OR} with the strict equality separator \citep{nonmonotonic_activations_mlp_thesis}.
However, knowing that \texttt{OR} induces more points from the positive class, we can choose to instead model the negative class.

Another intuition of the doubled VC dimension is that the
equality separator is the intersection between 2 halfspaces $\mathbf w^T\mathbf x+b \leq \epsilon$ and $\mathbf w^T\mathbf x+b\geq -\epsilon$.
In contrast, modeling a halfspace separator $\mathbf w^T\mathbf x+b$ with an equality separator requires non-linearity like ReLU:
$\mathds{1}(\max(0, \mathbf w^T\mathbf x+b)\neq 0)$.

From a hypothesis testing perspective, equality separators yield a more robust separating hyperplane compared to SVMs in binary classification.
We link the $\epsilon$-error separator to the most powerful test, the likelihood ratio test, via Twin SVMs \citep{twin_svm}.
In binary classification, Twin SVM learns an $\epsilon$-error separator for each class, \(\mathbf {w}_1^T\mathbf x+b_1 = 0\) and \(\mathbf {w}_2^T\mathbf x+b_2=0\).
For a test datum $\mathbf {x'}$, we can test its membership to a class based on its distance to class $i \in \{1,2\}$, $\epsilon_i 
:=\frac{|\mathbf w_i^T \mathbf {x'} + b_i|}{||\mathbf w_i||_2}$.
It is natural to use the ratio between the distances to resemble a likelihood ratio test
\[\mathds{1}\left (\frac{\epsilon_2}{\epsilon_1}\geq t \right)=\mathds{1}\left(\frac{|\mathbf w_2^T \mathbf {x'} + b_2|}{||\mathbf w_2||} - t\frac{|\mathbf w_1^T \mathbf {x'} + b_1|}{||\mathbf w_1||}\geq 0\right).\]
By constraining the Twin SVM to (1) parallel hyperplanes (\(\mathbf w_1= \mathbf w_2\)), (2) different biases ($b_1\neq b_2$) and (3) equal importance to classify both classes ($t=1$), we obtain a separating hyperplane \(2\mathbf w_1^T\mathbf  x+ b_1+b_2=0\) between the 2 classes (Appendix \ref{proof_pseudo_lrt}).
This separating hyperplane lies halfway between the hyperplanes for each class, where $\frac{|b_1-b_2|}{||\mathbf w_1||_2}$ resembles the margin in SVMs.
Without loss of generality, for $b_1>b_2$,
the region in between the 2 hyperplanes $\{x\in \mathcal X:-b_2\leq \mathbf w_1^T\mathbf x\leq -b_1\}$ acts as the soft margin.
The separating hyperplane derives from the learnt $\epsilon$-error separators for each class, which is learnt using the \textit{mass} of each class.\footnote{Linear discriminant analysis is also similar by learning the hyperplane via the class masses (see \citet{ElementsofStatisticalLearning}).}
This is unlike SVMs, which use \textit{support vectors} to learn the separating hyperplane \citep{SVM}.
Support vectors may be outliers far from high density regions for each class, so using them may not be as robust as
two $\epsilon$-error separators.
\looseness=-1

\section{Anomaly detection}

We focus on supervised AD, where different types of anomalies 
are not fully represented in the overall anomaly (negative) class during training.
Distinguishing between the normal and anomaly class is difficult due to (a) class imbalance, (b) lack of common features among anomalies and (c) no guarantees to detect unseen anomalies \citep{DeepSAD}.
We aim to find the manifold that normal data (positive class) live on where anomalies do not.
This more general non-linear formulation of equality separation shows the utility of our paradigm.
\looseness=-1

\subsection{Background}

\paragraph{Unsupervised AD}
has been the main focus of AD literature,
where training data is unlabelled and assumed to be mostly normal (i.e. non-anomalous).
Since the anomaly class is not well represented in training, most methods train models on auxiliary tasks using normal data.
This assumes that the model is unlikely to perform well at the known task during inference if the datum is unseen/anomalous \citep{AD_survey}.
Popular tasks include regression via autoencoding and forecasting \citep{lipschitz_ae_AD, kernel_PCA_AD}, density estimation \citep{DAGMM, LOF} and contrastive learning \citep{NeuTraLAD_Learnable_Transform, AD_tabular_data_internal_contrastive_learning}.

\paragraph{Classification methods} have generally not focused on supervised AD either.
Classification tasks that optimize directly on the AD objective (which is connected to Neyman-Pearson classification \citep{min_vol_SVM}) artificially generate anomalies \citep{classification_for_ad_jmlr, XAD_negative_sampling, DROCC_AD, min_vol_SVM}).
However, synthetic anomalies are hard to sample in high dimensions \citep{Unifying_Review_AD_Ruff} and may be unrepresentative of real anomalies, so covariate shift of the anomaly class during test time may affect AD performance.
The lack of common features among anomalies 
hinders models from generalizing to detecting unseen anomalies.
Nevertheless, directly optimizing over the AD objective is closer to the task and has benefits such as more effective explainable artificial intelligence results (see \citet{XAD_negative_sampling}).
Equality separation shifts the perspective of classification from a probabilistic to a geometric lens, aiming to reduce the dependency on sampling.
\looseness=-1

\paragraph{Limited labeled anomalies}
is a less explored problem, mostly explored from a semi-supervised approach.
Methodologically, this is the same as supervised AD: a small portion of the dataset is labeled while the unlabeled portion is assumed to be (mostly) normal.
\citet{DevNet_AD} proposes deviation networks to aid direct binary classification.
They stay within halfspace separation and focus on anomalies limited in \textit{number} rather than \textit{type}.
\citet{DeepSAD} proposes Deep Semi-Supervised Anomaly Detection (SAD), which performs an auxiliary task of learning data representations such that normal data have representations within a hypersphere while anomalies fall outside (this is ideologically the same as having an RBF activation in the output layer).
Their baseline comparison on directly solving AD with ERM on seen anomalies does not yield competitive results because ``supervised ... classification learning principles ... [are] ill-defined for AD'' \citep{DeepSAD}.

\paragraph{Out-of-distribution (OOD) detection} is the field that motivates our application of equality separation in AD.
To reject OOD inputs and reduce approximation error, classifiers can be designed to induce less open decision regions in input space \citep{gaussian_activation_ood_detection}, also known as open space risk minimization \citep{towards_open_set_recognition}.
Decision regions where the decision boundary encloses points belonging to a single class are considered closed.
Forming closed decision regions has a strong connection to density level set estimation under the concentration assumption:
for some class with density function $p$, the density level set for significance level $\alpha$ (and corresponding threshold $\tau_\alpha$) can be formulated as a small, non-empty set
\begin{align*}
C_\alpha :
= \{\mathbf x \in \mathcal X \mid p(\mathbf{x}) > \tau_\alpha\}
= \argmin_{C \text{ measurable}
} \{\mathrm{Vol}(C)  : \mathbb P(C)\geq 1-\alpha\}
\end{align*}
while anomalies of the class belong to $\mathcal X \backslash C_\alpha$ \citep{Unifying_Review_AD_Ruff}.
Especially when feature space $\mathcal X$ is unbounded, we see a connection between classifiers that induce closed decision regions and classifiers that induce smaller, finite-volumed density level sets $C_{\alpha}$.
Smaller sets are `more conservative', reducing the open space risk (i.e. region of wrongly classifying an anomaly as normal/false negatives).
To model the normal class in (one-class) AD, we can estimate the density level set.
From a simple classification framework, we estimate the level sets from data by forming a decision boundary that encloses points within the class, similar to SAD.
Figure \ref{fig:closed_decision_region_input} is an example where the true decision boundary encloses a region in the input space such that all points in those regions correspond to high density regions of a particular class.

One-vs-set machines \citep{towards_open_set_recognition} and its successors (e.g. \citet{best_fitting_hyperplane_classification}) use the idea of 2 parallel hyperplanes for open set recognition, but the optimization procedure is not flexible enough to integrate directly into NNs as we do through bump activations.
We extend this idea to the one-class scenario, where the negative data are not other meaningful classes, but (seen) anomalies.
We view our contributions as the union of these various ideas applied to AD and the insights to why they work in unison. 
Based on the concentration assumption in AD, we seek to generalize one-vs-set machines to neural networks, which results in hyper-hill networks.
However, it is not straightforward to characterize if we can still obtain closed decision regions after passing data through a NN.
We proceed to analyze locality properties of activation functions, which leads to our proposed \textit{closing numbers} which helps characterize closed decision regions across the NN.

\subsection{Locality: precursor to closed decision regions}
\label{section:density_level_set_estimation}
Using NNs as classifiers can help to find these complex decision boundaries, but it is not immediately clear if the decision regions can be closed and, consequently, if the density level set can be well estimated.
In NNs, classifiers are modeled by their smooth versions (activation functions) and have largely been split into 3 categories: global, semi-local and local \citep{nonmonotonic_activations_mlp_thesis, advantage_semilocal}.
Local mappings, e.g. radial basis function (RBF) networks and activation functions, ``activate'' around a center point i.e. have outputs beyond a threshold.
In contrast, global mappings (e.g. perceptron, sigmoid, ReLU) activate on a whole halfspace.
Semi-local mappings straddle in between. 
In particular, \citet{nonmonotonic_activations_mlp_thesis} proposes bump activations, hyper-ridge and hyper-hill classification which are locally activated along one dimension and globally activated along others.
Since equality separators are the discrete version of bump activations (Section \ref{section:halfspace_connnection}),
we view equality separators as semi-local too.
We summarize these observations in Table \ref{tab:activation_function_comparison}.
\looseness=-1

\begin{table}[]
    \centering
    \caption{Difference between perceptron, equality separator and RBF network (inputs from $\mathbb R^n$).}
    \begin{tabular}{cccc}
\toprule
        Decision Rule & Perceptron & Equality separator & RBF network\\
    \midrule
        Bias in Input Space & Halfspace & Hyperplane & Point/Mean\\
        Activation Function & Sigmoid, ReLU & Bump & RBF\\
        Properties (Closing Number) & Global ($n+1$) & Semi-local ($n$) & Local ($1$) \\
\bottomrule
    \end{tabular}
    \label{tab:activation_function_comparison}
\end{table}

The semi-locality of equality separators and bump activations achieve the benefits of global mappings like sigmoids in perceptrons and local mappings like RBF activations in RBF networks \citep{nonmonotonic_activations_mlp_thesis}.
Local activations can converge faster \citep{advantage_semilocal} and provide reliability in adversarial settings \citep{goodfellow_fast_gradient_sign, gaussian_activation_ood_detection}, while global mappings work well in high dimensions with exponentially fewer neurons \citep{advantage_semilocal} and are less sensitive to initializations \citep{nonmonotonic_activations_mlp_thesis}, making them easier to train \citep{goodfellow_fast_gradient_sign}.
Consider our \texttt{XOR} example with a single neuron: sigmoids in perceptrons are too global, while one RBF unit cannot cover 2 points from either class.
Equality separation and bump activations limit the globality of the perceptron but retain a higher globality than RBF activations, which allows them to linearly classify \texttt{XOR}.
But what does ``globality'' or ``locality'' precisely mean, and can we quantitatively measure this?
Perhaps more importantly, can we enclose decision regions?

\subsection{Locality and closing numbers for closed decision regions}
\label{section:connection_to_AD}

The locality of a classifier not only has implications on classification, but also on AD.
Specifically, it is intuitive to see a connection between the locality of classifiers and their ability to form closed decision regions, which in turn is useful for density level set estimation in AD.
To quantify the learnability of closed decision regions, we propose one metric to formalize this notion of locality:
\begin{definition}
The \textit{closing number} of a hypothesis class is the minimum number of hypotheses from the class such that their intersection produces a closed (positive-volumed\footnote{Meaningful analysis requires constraints that decision regions (in raw or latent space) has non-zero volume.\looseness=-1}) decision region.
\end{definition}

\begin{figure}
    \centering
    \begin{subfigure}[]{0.32\textwidth}
    \centering
    \includegraphics[height=0.16\textheight,  keepaspectratio]{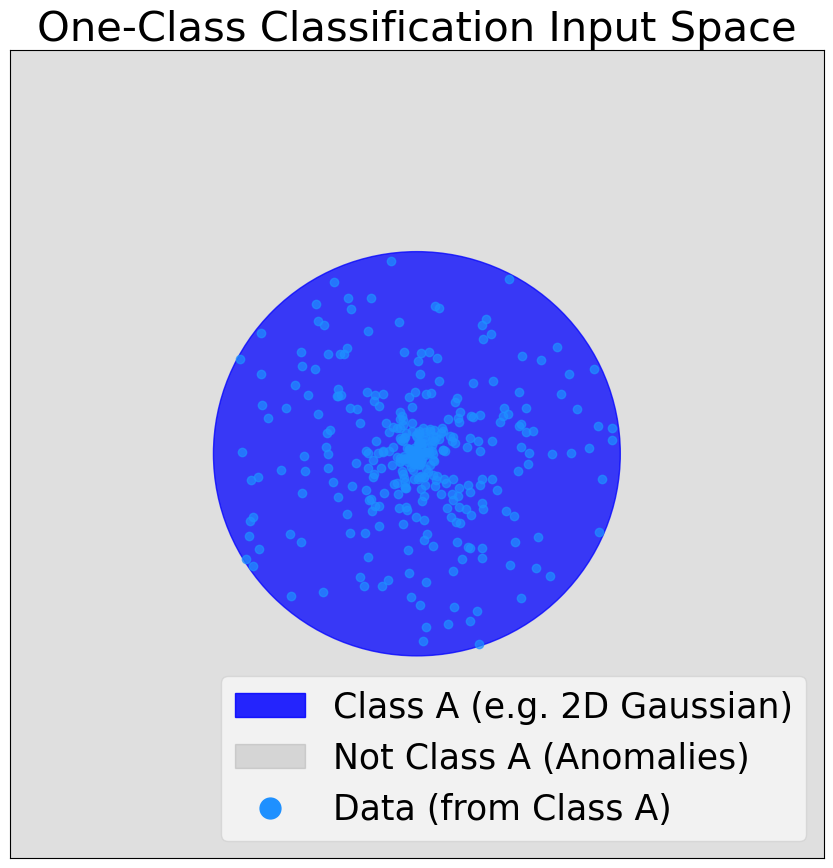}
    \caption{Normal data is in blue circle.}
    \label{fig:closed_decision_region_input}
    \end{subfigure}
    \begin{subfigure}[]{0.32\textwidth}
    \centering
    \includegraphics[height=0.16\textheight,  keepaspectratio]{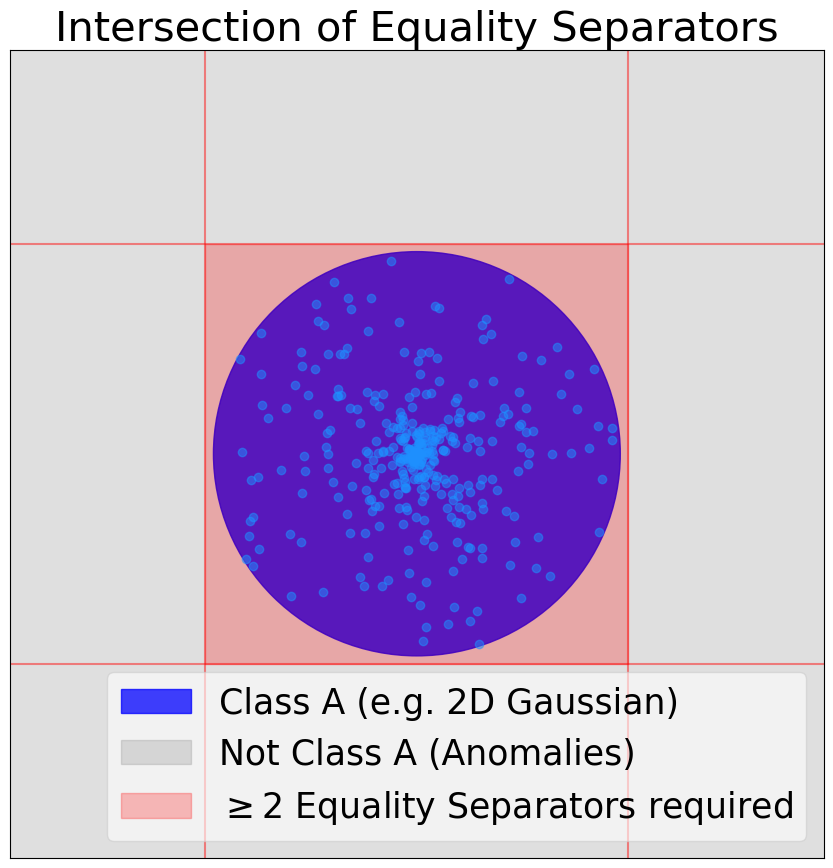}
    \caption{Equality separators: rectangle.}
    \label{fig:closed_decision_region_ES}
    \end{subfigure}
    \begin{subfigure}[]{0.32\textwidth}
    \centering
    \includegraphics[height=0.16\textheight,  keepaspectratio]{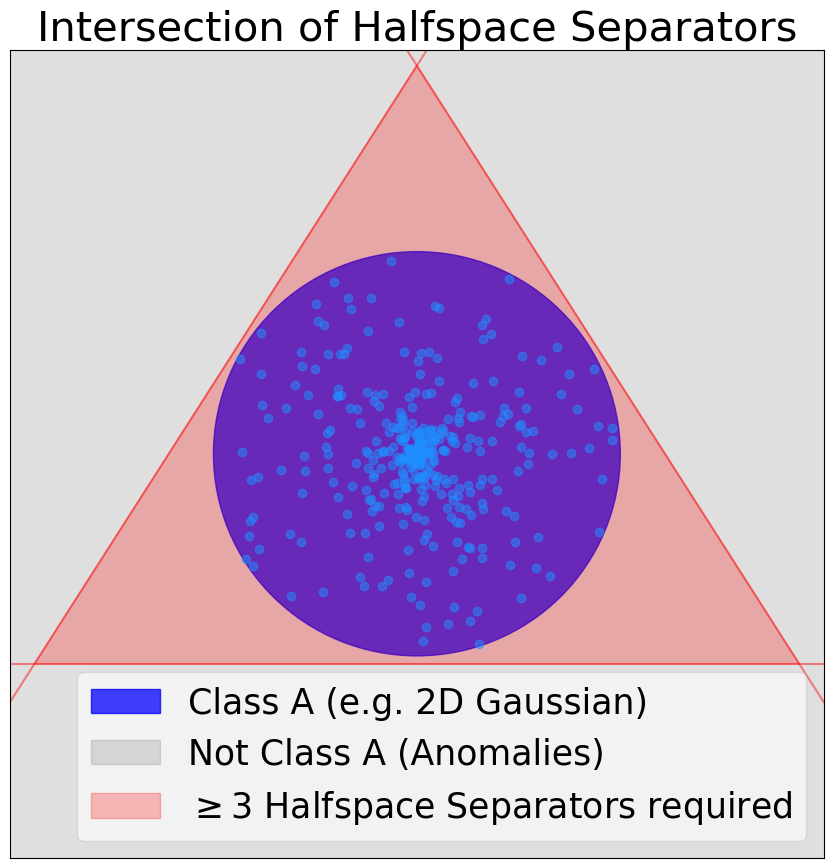}
    \caption{Halfspace separators: triangle.}
    \label{fig:closed_decision_region_HS}
    \end{subfigure}
\label{fig:closed_decision_regions}
\caption{Normal data occupies a space with non-zero, finite volume.
Classifiers can form closed decision regions to capture this.}
\end{figure}

\paragraph{Significance of closing numbers} Closing numbers is not only an important formalization of a big picture idea of locality, but also has two other implications. 
First, the output of a neuron is a function (e.g. scaled weighted average for sigmoid\footnote{Parameters can be normalized while negations represent the complement of a region.}) of neuron activations in the previous layer.
Thus, the closing number represents the minimum number of neurons in a layer to induce a closed decision region in the previous layer when using an activation of corresponding locality.
This information is crucial in hyperparameter selection for the number of neurons we design our neural network to have in each layer if we want to enclose decision regions.
By extension, it also suggests the (minimum) dimension of the space that contains false negatives (undetected anomalies), since any projection of unseen (anomalous) data to the normal manifold will cause false negatives.
Second, the number (or, more specifically, the measure) of configurations of learnable parameters that do indeed enclose the decision region suggests how much guarantee we have to form closed decision regions.
In summary, closing numbers tell users the amount of risk for false negatives they take when choosing fewer neurons than the closing number and the analysis of the configuration suggests the residue risk even when users put the corresponding number of neurons. 
We proceed to state the closing numbers for respective classifiers.
\looseness=-1

\begin{theorem}
\label{thm:closing numbers}
    Closing number of RBF, equality and halfspace separators is 1, $n$ and $n+1$ respectively in $\mathbb R^n$.
    Closed decision regions are formed with the intersection of the corresponding number of classifiers everywhere for RBF separators and almost everywhere for equality separators, but normal vectors need to form an $n$-simplex volume for halfspace separators.
\end{theorem}
The proof is in Appendix \ref{appendix:closing_numbers}.
In 2D, these closed decision regions would correspond to a circle (e.g. the circle in Figure \ref{fig:closed_decision_region_input}), parallelogram (Figure \ref{fig:closed_decision_region_ES}) and triangle (Figure \ref{fig:closed_decision_region_HS}) respectively.

The implication of closing numbers and the conditions needed for its realization in the continuous version (i.e. activation functions) across adjacent layers in feed-forward NNs is as follows.
Although equality separation has a relatively large closing number compared to RBF, we are still guaranteed to have closed decision regions almost everywhere in the space of $n$ learnable $n$-sized normal vectors (which is an $n\times n$ weight matrix in a NN).
Meanwhile, we have no reason to expect closed decision regions from halfspace separation even with $n+1$ (or more) neurons in the following layer.
Closing numbers provide a theoretical framework to understand the corroborate with 
semi-locality of bump functions \citep{nonmonotonic_activations_mlp_thesis}, where it is easier to learn than RBFs but more robust to adversarial settings than halfspace separators \citep{gaussian_activation_ood_detection}.
Additionally, \citet{goodfellow_fast_gradient_sign} argued for the decrease robustness of sigmoid and ReLU to adversarial attacks due to the linearity in individual activations.
Closing numbers complement this view, providing a global perspective of activations working together to induce closed decision regions.
This gives insight on the capacity and learnability of closed decision regions.
\looseness=-1

\subsection{Controlling generalization error with inductive bias in AD}
\label{section:AD_data_structure}

For binary classification, we allow label flipping because it is a priori unclear which class we should model to be in the margin.
This changes for AD.
To enclose the decision region%
, we impose the constraint that the normal (and not the anomaly) class should lie within the margin of equality separators.
Using domain knowledge to assign normal data to the positive class, we aim to find the normal data manifold rather than discriminate between normal and (seen) anomalies.
A better AD objective model lowers the approximation error (i.e. the error between the Bayes classifier and the best classifier in our hypothesis class) compared to halfspace separators.
Consequently, prohibiting label flipping reduces to hyper-hill classification, which has VC dimension of $n+1$ like halfspace separation \citep{nonmonotonic_activations_mlp_thesis}, resulting in the same generalization bound as halfspace separators for the estimation error.
Compared to halfspace separators, we obtain the same bound on the estimation error while improving the approximation error in AD.
Therefore, we focus on AD applications.
\footnote{We note that AD is hard because any validation set has no information about unseen anomalies, so its utility as a proxy of test-time generalization is limited.
Thus, generalization heavily relies on the design of a model.}
\looseness=-1

\section{Applications}
\paragraph{Optimization scheme}
To obtain a classifier, we perform ERM on data $\{(\mathbf x_i, y_i)\}_{i=1}^m$.
In this paper, we consider mean-squared error (MSE) and logistic loss (LL) as loss functions.
Optimizing with BFGS with LL and MSE on the affine class $A_2$ produces a solution for \texttt{XOR}, corroborating with the results in \cite{nonmonotonic_activations_mlp_thesis}.
For the bump function in (\ref{eqn:bump_activation}), we fix $\mu=0$.
A smaller $\sigma$ initializes a stricter decision boundary but limits gradient flow.
We defer further analysis of the role of $\sigma$ to Appendix \ref{section:sigma_role}.
\looseness=-1

\paragraph{Experiments}
For binary classification, we observe that equality separators can classify some linearly non-separable data (e.g. Figure \ref{fig:inseparable_equality_separation}). They are also competitive with SVMs in classifying linearly separable data (e.g. Figure \ref{fig:gaussian_equality_separation}, where equality separation is more conservative in classifying the positive brown class).
Since our focus is to evaluate in supervised AD, we defer results and further discussion of each experiment to Appendix \ref{appendix:linear_classification_exp}.
For supervised AD, we explore 2 settings.
The first is with synthetic data, where we can understand the differences between classifiers more clearly with visualizations and ablation studies.
The second is with real-world data from cyber-security, medical and manufacturing.
We evaluate the ability of models to separate normal and anomaly classes with area under the precision-recall curve (AUPR), also known as average precision.
For evaluation, we seek to answer two research questions (RQs):
\textbf{RQ1.} Is equality separation better than conventional approaches of halfspace separation or RBF separation?
\textbf{RQ2.} Can pure ERM work for supervised AD?
For RQ1, we compare across different locality approaches (equality separation, halfspace separation or RBFs) for a fixed method.
For RQ2, we compare standard ERM with other methods on seen and unseen anomalies with other supervised methods for a fixed locality approach.
\looseness=-1

\subsection{Supervised anomaly detection: toy model}
\label{section:nonlinear_ad}
We explore NNs as embedding functions into a linear feature space for non-linear supervised AD.
With a simple binary classification objective, we test how well different NNs can form one-class decision boundaries.
Other deep methods have been used for unsupervised \citep{DAGMM}, semi-supervised \citep{DeepSAD} and supervised \citep{abc_supervised_ad} AD.
We only consider directly optimizing over the traditional binary classification objective.
To show how local activations are more suitable for AD as suggested by closing numbers, we mitigate their limitations of vanishing gradients by experimenting with low dimensionality.

\paragraph{Setup}
To demonstrate AD with seen and unseen anomalies, we generate synthetic 2D data by adapting the problem in \citet{coin_flipping_nn}.
The normal (positive\footnote{This is according to neuron activation and differs from AD literature. More comments in Appendix \ref{appendix:non_linear_ad}.}) class is uniformly sampled from a circle of radius 1, and the anomaly (negative) class during training and testing is uniformly sampled from circles of radius 2 and 0.5 respectively, with all centred at the origin.
This task is not trivial because it is easy to overfit the normal class to the ball rather than just its boundary.
To challenge each method, we limit this experiment to 100 data samples.
More details can be found in Appendix \ref{appendix:non_linear_ad}.
\looseness=-1

\paragraph{Models}
We make apples-to-apples comparisons with classical binary classifiers: tree-based methods decision trees (DT), random forest (RF), extreme gradient boosted trees (XGB), and halfspace methods logistic regression (LR) with RBF kernel and SVM with RBF kernel.
To illustrate the difficulty of this AD task, we also use standard AD baselines: One-Class SVM (OCSVM) with RBF kernel, Local Outlier Factor (LOF) and Isolation Forest (IsoF).\footnote{Note that these are unsupervised methods. We use them only to benchmark and not as our main comparisons.}
We compare these shallow models with equality separators with an RBF kernel.
We also compare NNs directly trained for binary classification. 
We modify NNs with 2 hidden layers to have (a) halfspace separation, RBF separation or equality separation at the output layer, 
and (b) leaky ReLU, RBF or bump activations in hidden layers.
We provide quantitative (Table \ref{tab:supervised_ad}) and qualitative (Figure \ref{fig:AD_heatmap}) analyses
that suggest either accurate one-class classification or overfitting to seen anomalies.
\looseness=-1

\begin{table}
    \caption{Test AUPR for synthetic data
    comparing shallow models and neural networks.
        }
\begin{subtable}{0.34\textwidth}
    \centering
    \caption{Shallow models.}
    \begin{tabular}{cc}
        \toprule
Shallow Models & AUPR\\
\midrule
DT & 0.87$\pm$0.03\\
RF  & 0.88$\pm$0.01\\
 XGB  &  0.55$\pm$0.01\\
 LR  &  0.54$\pm$0.00\\
 SVM   &  0.54$\pm$0.00\\
 OCSVM   &  0.83$\pm$0.00\\
 IsoF & 0.86$\pm$0.05\\
 LOF & \textbf{1.00$\pm$0.00}\\
 ES  (ours) &  \textbf{1.00$\pm$0.00}\\
\bottomrule
    \end{tabular}
\end{subtable}
\begin{subtable}{0.63\textwidth}    \centering
\caption{For NNs, we modify the output layer (halfspace separation (HS), equality separation (ES) and RBF separation (RS)) in the rows and activation functions in the columns.
        Random classifier AUPR is 0.75.}
    \begin{tabular}{lccc}
        \toprule
NN$\backslash$Activation & Leaky ReLU & Bump & RBF\\
\midrule
HS  &  0.62$\pm$0.12 &  0.92$\pm$0.13 & 0.83$\pm$0.18\\
ES (ours)  &  0.83$\pm$0.15 &  {0.99$\pm$0.04} & \textbf{1.00$\pm$0.00}\\
RS  &  0.97$\pm$0.05 &  0.91$\pm$0.14 & 0.97$\pm$0.06\\
\bottomrule
    \end{tabular}
\end{subtable}
    \label{tab:supervised_ad}
\end{table}

\begin{figure}[t]
    \centering
    
\begin{subfigure}[b]{0.175\textwidth}
    \centering
    \includegraphics[height=0.10\textheight, keepaspectratio]{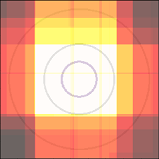}
        \caption{RF}
        \label{fig:AD_heatmap_RF}
\end{subfigure}
\begin{subfigure}[b]{0.175\textwidth}
    \centering
    \includegraphics[height=0.10\textheight, keepaspectratio]{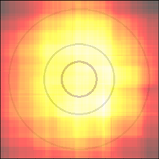}
        \caption{IsoF}
        \label{fig:AD_heatmap_IF}
\end{subfigure}
\begin{subfigure}[b]{0.175\textwidth}
    \centering
    \includegraphics[height=0.10\textheight, keepaspectratio]{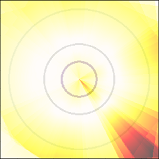}
        \caption{LOF}
        \label{fig:AD_heatmap_LOF}
\end{subfigure}
\begin{subfigure}[b]{0.175\textwidth}
    \centering
    \includegraphics[height=0.10\textheight, keepaspectratio]{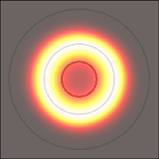}
        \caption{ES (RBF)}
        \label{fig:AD_heatmap_ES_RBF}
\end{subfigure}
\begin{subfigure}[b]{0.175\textwidth}
    \centering
    \includegraphics[height=0.10\textheight, keepaspectratio]{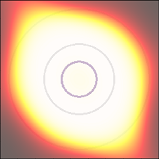}
        \caption{HS-b}
        \label{fig:AD_heatmap_HS2b}
\end{subfigure}
\begin{subfigure}[b]{0.175\textwidth}
    \centering
    \includegraphics[height=0.10\textheight, keepaspectratio]{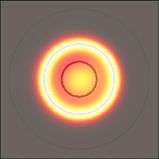}
        \caption{ES-r}
        \label{fig:AD_heatmap_ES2r}
\end{subfigure}
\begin{subfigure}[b]{0.175\textwidth}
    \centering
    \includegraphics[height=0.10\textheight, keepaspectratio]{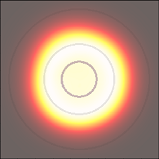}
        \caption{RS-r}
        \label{fig:AD_heatmap_RBF2r}
\end{subfigure}
\begin{subfigure}[b]{0.5\textwidth}
    \centering
    \includegraphics[height=0.10\textheight, keepaspectratio]{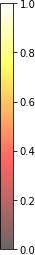}
        \caption{Color map: brighter colors $\mapsto$ normal class.}
        \label{fig:AD_heatmap_colourbar}
\end{subfigure}
    \caption{
    Sample heatmap predictions by different models.
    The middle circle is the positive class, while the outer and inner circle is the negative class during training and testing respectively.
    Figures \ref{fig:AD_heatmap_HS2b}, \ref{fig:AD_heatmap_ES2r} and \ref{fig:AD_heatmap_RBF2r} are deep models, with hidden layer activations representated by the suffix `b' for bump and `r' for RBF.
    Shallow equality separator (Figure \ref{fig:AD_heatmap_ES_RBF}) has a one-class decision boundary closest to ground truth, followed by equality separator neural networks with RBF activation (Figure \ref{fig:AD_heatmap_ES2r}).
    }
    \label{fig:AD_heatmap}
\end{figure}

\paragraph{Shallow models}
From Table \ref{tab:supervised_ad}, we see classical binary classifiers XGB, LR and SVM perform worse than random and are inappropriate for supervised AD.
DT, RF and LOF perform better than random but overfit the training data;
even though LOF has a perfect test AUPR, it has a train AUPR of 0.84, suggesting a bad one-class decision boundary (Figure \ref{fig:AD_heatmap_LOF}), while DT and RF overfit with an $\ell_\infty$ norm-like decision boundary (e.g. Figure \ref{fig:AD_heatmap_RF}).
IsoF has a high AUPR (Figure \ref{fig:AD_heatmap_IF}), but our ES with RBF kernel has the highest and perfect AUPR with a decision boundary close to ground truth (Figure \ref{fig:AD_heatmap_ES_RBF}).
Overall, we observe that equality separation is better than halfspace separation (RQ1) and ERM works only with equality separation (RQ2).
\looseness=-1

\paragraph{Neural networks}
Shallow methods cannot always be used (e.g. large amounts of non-tabular or high-dimensional data), so we also experiment on NNs with hidden layers.
Though straightforward, our equality separator and bump activation achieves better NN performance%
. Answering RQ1 and RQ2, we observe that
equality separators and RBF separators with hidden RBF activations trained with vanilla ERM consistently obtain perfect AUPR (Table \ref{tab:supervised_ad}), and
we validate that closing numbers are useful in quantifying the ability to induce closed decision boundaries.
Looking at individual models, about half of the equality separators achieve high separation between normal and anomaly classes, with anomalies 8 times further from the hyperplane in the penultimate layer than normal data.
In contrast, only 1 out of 20 RBF separators with RBF activations attain such high separation -- most models produced outputs similar to Figure \ref{fig:AD_heatmap_RBF2r}.
The semi-locality of equality separators retain locality for AD while enhancing separation between normal and anomaly classes of NNs with RBF activations.
\looseness=-1

\paragraph{Geometry of decision boundary}
More local activation functions that have a reduced region being activated correspond to regions in space that the model has seen before.
Through our inductive bias of the geometry formalized through closing numbers, we can use more local activation functions to sidestep the assumptions and difficulty of probabilistically reducing the open space risk via ERM by sampling anomalies, such as specifying a distribution to sample in the high-dimensional input space \citep{XAD_negative_sampling} or latent space \citep{OOD_non_parametric_outlier_synthesis}.
This intuition of the geometry is also present in how the margin width changes across layers.
Like how deep learning has been observed to extract lower-level features in earlier layers and higher-level features in later layers, we expect margin width to be larger in earlier layers to allow more information through and smaller in later layers to be more discriminatory.
In Appendix \ref{appendix:non_linear_ad}, we corroborate with this intuition.
Through ablations, we also show that the equality separation's good results remain regardless of number of hidden layers and loss function.

\subsection{Supervised anomaly detection: real-world data}
\label{section:experiments_real_world}

\paragraph{Cyber-attacks}
We use the NSL-KDD network intrusion dataset \citep{NSL_KDD_Dataset} for real-world cyber-attack experiments.
To simulate unseen attacks (anomalies), training data only has normal data and denial of service attacks (DoS), and test data has DoS and unseen attacks, probe, privilege escalation and remote access attacks.
For simplicity, we consider NNs with only 1 type of activation function in terms of locality as vanilla binary classifiers rather than switching between activations between hidden and output layers like in the previous experiments.
Baselines we compare against are SVMs and one-vs-set machines (OVS) with RBF kernels for binary classifiers, and OCSVM, IsoF and LOF for unsupervised AD.
Since the original OVS returns a binary decision, so we also include an edited version from the perspective of equality separation which we denote as OVS-ES, where we center a bump activation in the center of the two hyperplanes used for the original OVS.
The state-of-the-art (SOTA) deep AD methods we compare against make use of labels: (1) negative sampling (NS), where binary classifiers are trained with generated anomalies \citep{XAD_negative_sampling} and (2) SAD, where AD is based on hypersphere membership \citep{DeepSAD}.
As a note, SAD is effectively an RBF separator on a pre-trained autoencoder.
We report results in Table \ref{tab:supervised_ad_kdd}, comparing methods within their categories.

\begin{table}
    \caption{Test AUPR (mean $\pm$ standard deviation) on NSL-KDD dataset, grouped according to different methods.
    The different methods are random classifiers (calculated in expectation), shallow supervised (Sup.) and unsupervised (Unsup.) methods, binary classifier NNs, Negative Sampling (-NS) and Deep Semi-Supervised Anomaly Detection (-SAD).
    We compare using different kinds of activations in NNs: global (HS), semi-local (ES) and local (RS).
    AUPR is reported vis-a-vis each attack and all attacks (overall).
    DoS attacks are seen during training, while the rest are not. 
        }
        \centering
    \begin{tabular}{cl|ccccc}
\toprule
&Model$\backslash$Attack  & DoS     & Probe  & Privilege  & Access  & Overall     \\
\midrule
 &Random  &  0.435  &  0.200  &  0.007  &  0.220  &  0.569 \\
 \midrule
\parbox[t]{0.5mm}{\multirow{3}{*}{\rotatebox[origin=c]{90}{Sup.}}} &{SVM}   &  \textbf{0.959$\pm$0.000}  &  \textbf{0.787$\pm$0.000}  &  0.037$\pm$0.000  &  0.524$\pm$0.000  &  0.948$\pm$0.000 \\
  &{OVS}   &  {0.717$\pm$0.000}  &  0.600$\pm$0.000  &  \textbf{0.503$\pm$0.000}  &  \textbf{0.610$\pm$0.000}  &  0.785$\pm$0.000 \\
  &{OVS-ES (ours)}   &  {0.956$\pm$0.000}  &  0.785$\pm$0.000  &  0.038$\pm$0.000  &  0.535$\pm$0.000  &  \textbf{0.949$\pm$0.000} \\
 \midrule
\parbox[t]{0.5mm}{\multirow{3}{*}{\rotatebox[origin=c]{90}{Unsup.}}} &OCSVM &  \textbf{0.779$\pm$0.000} &  0.827$\pm$0.000 &  \textbf{0.405$\pm$0.000} &  \textbf{0.760$\pm$0.000} &  \textbf{0.897$\pm$0.000} \\
 &IsoF  &  0.765$\pm$0.073  & \textbf{0.850$\pm$0.066} &  0.089$\pm$0.044  &  0.392$\pm$0.029  &  0.865$\pm$0.031 \\
 &LOF  &   0.495$\pm$0.000  &  0.567$\pm$0.000  &  0.039$\pm$0.000  &  0.455$\pm$0.000  &  0.718$\pm$0.000 \\
\midrule
\parbox[t]{0.5mm}{\multirow{3}{*}{\rotatebox[origin=c]{90}{ERM}}}&HS     &  0.944$\pm$0.016  &  \textbf{0.739$\pm$0.018}  &  0.016$\pm$0.006  &  0.180$\pm$0.033  &  0.877$\pm$0.010 \\
 &ES (ours)  &  \textbf{0.974$\pm$0.001}  &  0.717$\pm$0.107  &  \textbf{0.045$\pm$0.014}  &  \textbf{0.510$\pm$0.113}  &  \textbf{0.941$\pm$0.014} \\
 &RS  &  0.356$\pm$0.002  &  0.213$\pm$0.002  &  0.010$\pm$0.000  &  0.228$\pm$0.002  &  0.406$\pm$0.001 \\
\midrule
\parbox[t]{0.5mm}{\multirow{3}{*}{\rotatebox[origin=c]{90}{NS}}} &HS-NS         &  0.936$\pm$0.001  &  0.642$\pm$0.030  &  0.006$\pm$0.000  &  0.139$\pm$0.001  &  0.845$\pm$0.006 \\
 &ES-NS (ours)  &  \textbf{0.945$\pm$0.009}  &  \textbf{0.659$\pm$0.013}  &  \textbf{0.023$\pm$0.011}  &  0.206$\pm$0.013  &  \textbf{0.881$\pm$0.007} \\
 &RS-NS         &  0.350$\pm$0.002  &  0.207$\pm$0.002  &  0.009$\pm$0.000  &  \textbf{0.223$\pm$0.002}  &  0.401$\pm$0.002
\\
\midrule
\parbox[t]{0.5mm}{\multirow{3}{*}{\rotatebox[origin=c]{90}{SAD}}} &HS-SAD         & 0.955$\pm$0.003  & 0.766$\pm$0.011  &  \textbf{0.100$\pm$0.000}	  & 0.447$\pm$0.000  &  0.935$\pm$0.007 \\
 &ES-SAD (ours)  & \textbf{0.960$\pm$0.002}   & \textbf{0.795$\pm$0.004}   & 0.047$\pm$0.002   & \textbf{0.509$\pm$0.009}   & \textbf{0.952$\pm$0.002}  \\
 &RS-SAD         & 0.935$\pm$0.000   &  0.678$\pm$0.000  &  0.022$\pm$0.000  &  0.142$\pm$0.000  &  0.846$\pm$0.000 \\
\bottomrule
    \end{tabular}
    \label{tab:supervised_ad_kdd}
\end{table}

\paragraph{\textit{Shallow methods}}
Due to their similarity, we expect some generalization of DoS to probe attacks, while little is expected for privilege escalation or remote access attacks.
Our SVM baseline validates this, achieving a much higher increase in performance from 
random for probe attacks than
privilege and access attacks.

The RBF kernel seems to help encode a bias of similarity through locality -- we observe a drop in AUPR across all attacks when using a linear kernel (Table \ref{tab:supervised_ad_kdd_FULL_TABLE} in Appendix \ref{appendix:kdd}).
For OVS, we see a significant drop in performance on the seen DoS attack and probe attack, while unseen performance on privilege and access attacks are high.
Such a trade-off in performance likely suggests that these unseen attacks and some normal data are around the region where the SVM prediction is very high.
Our OVS-ES mitigates large trade-offs, maintaining seen anomaly AUPR while slightly increasing the unseen anomaly AUPR.
Overall, there is some generalization from seen to unseen anomalies, but it is not obvious how to improve these shallow methods.
Among unsupervised benchmarks, we also see a trade-off between detecting known and unknown anomalies: OCSVM is best at detecting privilege and access attacks, while IsoF is best at detecting DoS and probe attacks.
Notably, unsupervised methods are not competitive with the supervised SVM on seen (DoS) attacks.
\looseness=-1

\paragraph{\textit{NN}}Across the NN methods, we notice that equality separators/bump activations are the most competitive at detecting both seen and unseen anomalies, often obtaining the largest AUPR in their category (binary classifiers, NS and SAD).
In particular, equality separation with SAD and standard binary classification seem to be among the best, being competitive with
the SVM baseline.
we observe that equality separation in NNs is better than halfspace separation and RBFs (RQ1) and vanilla ERM with equality separation is competitive with SAD without any pre-training (or sampling, as in NS), as well as a good alternative to other supervised methods which achieve good AUPR on the seen DoS attacks (RQ2).
\looseness=-1

\paragraph{Medical dataset}
We use the thyroid dataset \citep{thyroid_dataset} for medical anomalies.
Using hyperfunction as the seen anomaly and subnormal as unseen, we compare ERM methods in Table \ref{tab:thyroid}.
Even though the thyroid dataset has the fewest dimensions of all datasets (21D), we see RBF separation already performing very poorly despite its stellar performance in our 2D synthetic dataset.
Nevertheless, equality separation still detects seen and unseen anomalies better than halfspace separation (RQ1).
\looseness=-1

\paragraph{Low data, high-dimensional image dataset}
MVTec dataset \citep{MVTec_Dataset} has images of normal and defective objects.
We select one anomaly class to be seen during training for supervised AD, while the other anomaly classes are only seen during test time.
There are on the order of 300 training samples, while feature embeddings (from DINOv2 \citep{dinov2}) are 1024-dimensional.

We compare AUPR of binary classifiers across 6 objects in Table \ref{tab:mvtec_overall}, zooming in on individual defect detection performance for pills (which has the most number of defect groups) in Table \ref{tab:mvtec_pill}.
Since DINOv2 acts as a kernel to embed features into a high dimensional space, we consider halfspace separation as our baseline,
while RBF separators act like SAD in this case for comparison.
During training, we observe that RBF separators consistently have high training error and low train AUPR (below 0.55), suggesting underfitting.
Here, too much locality impedes learning, especially in high dimensions.
Meanwhile, we see equality separators detecting unseen anomalies the best for 5/6 of the objects.
Furthermore, equality separators are consistently the best across different anomaly types (Table \ref{tab:mvtec_pill}).
Hence, equality separators can be integrated with foundation models to ensure a higher fidelity to the AD objective (RQ1, RQ2).

\begin{table}
    \caption{Test AUPR (mean $\pm$ standard deviation) on thyroid and MVTec dataset across various anomaly types.}
    \begin{subtable}[t]{0.33\textwidth}
    \centering
    \caption{Thyroid AUPR for hyperfunction (seen) and subnormal (unseen) \\anomalies.}
    \label{tab:thyroid}
    \vspace{-0.5mm}
        \resizebox{\columnwidth}{!}{%
            \begin{tabular}{lccc}
        \toprule
        Diagnosis & HS & RS  & ES (ours)\\
        \midrule
Hyper.      & 0.907$\pm$0.013& 0.030$\pm$0.002 & \textbf{0.934$\pm$0.004}  \\
Subnorm.  & 0.459$\pm$0.040 & 0.082$\pm$0.002 & \textbf{0.596$\pm$0.007} \\
Overall  & 0.615$\pm$0.032 & 0.103$\pm$0.001 & \textbf{0.726$\pm$0.004} \\
        \bottomrule
    \end{tabular}
    }
    \hspace{1mm}
    \end{subtable}
    \begin{subtable}[t]{0.37\textwidth}
    \centering
    \caption{MVTec overall AUPR for 6 objects. SAD is implemented with RBF.}
    \vspace{-0.5mm}
    \label{tab:mvtec_overall}
\resizebox{\columnwidth}{!}{%
    \begin{tabular}{lccc}
        \toprule
        Object & HS & RS (SAD) & ES (ours) \\
        \midrule
Capsule     & 0.918$\pm$0.002 & 0.894$\pm$0.000 & \textbf{0.947$\pm$0.003}\\
Hazelnut    & \textbf{0.932$\pm$0.006} & 0.783$\pm$0.000 & \textbf{0.932$\pm$0.034}\\
Leather     & \textbf{1.000$\pm$0.000} & 0.848$\pm$0.000 & 0.996$\pm$0.001\\
Pill        & 0.917$\pm$0.005 & 0.908$\pm$0.000 & \textbf{0.933$\pm$0.014}\\
Wood        & 0.984$\pm$0.004 & 0.866$\pm$0.000 & \textbf{0.986$\pm$0.007}\\
Zipper      & 0.998$\pm$0.000 & 0.879$\pm$0.000 & \textbf{1.000$\pm$0.000}\\
        \bottomrule
    \end{tabular}
    }
    \end{subtable}
    \begin{subtable}[t]{0.29\textwidth}
    \centering
    \caption{MVTec AUPR for pills object by defect type.}
    \label{tab:mvtec_pill}
\resizebox{0.9425\columnwidth}{!}{%
            \begin{tabular}{lccc}
        \toprule
        Defect & HS & ES (ours) \\
        \midrule
Combine        & 0.913$\pm$0.010 & \textbf{0.943$\pm$0.016}\\
Contam.   & 0.768$\pm$0.029 & \textbf{0.802$\pm$0.114}\\
Crack           & 0.652$\pm$0.009 & \textbf{0.730$\pm$0.046}\\
Imprint         & 0.470$\pm$0.035 & \textbf{0.573$\pm$0.094}\\
Type            & 0.885$\pm$0.071 & \textbf{0.910$\pm$0.156}\\
Scratch         & 0.655$\pm$0.021 & \textbf{0.690$\pm$0.058}\\
        \bottomrule
    \end{tabular}
    }
    \end{subtable}
    \label{tab:mvtec_thyroid}
\end{table}

\paragraph{Limitations/Extensions}
In general, we see that more local activations perform better in the low-dimension synthetic dataset, while more global activations perform better in the higher dimensional datasets.
Along with observations on training performance, we notice a trade-off between the ability to learn from the empirical risk (i.e. converging to a low loss) and the robustness to reject unseen samples, which we quantify with closing numbers and observe through our empirical results.
This observation complements the understanding from \citet{goodfellow_fast_gradient_sign} on the robustness of more local activations from the lens of enclosing decision regions across consecutive layers in neural networks (i.e. closing numbers).
Comparing the results of NNs to the SVM shallow baseline for the NSL-KDD dataset, we suspect that more work needs to be explored on how to learn well while being able to reject anomalies in NNs.

In addition, we note that equality separation generally has good results regardless of dimensionality.
However, we corroborate with 
\citet{nonmonotonic_activations_mlp_thesis} that bump activations in equality separation can be sensitive to initializations, evidenced by some variance in test AUPR.
To mitigate this, we can set $\sigma$ to be large to mitigate this. 
Ensembling equality separators can also help suppress inaccurate outputs from bad initializations.
We discuss more in Appendix \ref{appendix:non_linear_ad}.
We also note that NNs can struggle at generalizing to unseen data, as OCSVM achieves better AUPR than NNs for the unseen attacks in NSL-KDD.
We posit that understanding how the activation function, loss function and optimization scheme work together in equality separation to form closed decision regions is the next step towards designing better NNs for AD.
From our experiments in varying dimensions, we suspect that the key to balancing learning and enclosing decision regions is mitigating vanishing gradients for the regions classified as anomalous.
For unseen anomalies, although equality separation is not competitive with OCSVM, equality separation still outperforms other supervised methods as a whole, highlighting the suitability of equality separation for AD.

\section{Conclusion}

We revisit the belief that \texttt{XOR} is not linearly separable, which prompts us to propose
a new linear classification paradigm, \textit{equality separation}.
Equality separators use the distance of a datum from a hyperplane as their classification metric.
They have twice the VC dimension of halfspace separators (e.g. it can linearly solve some linearly non-separable problems like \texttt{XOR}) and enjoy the benefits of globality and locality, which we formalize and quantify with \textit{closing numbers}.
Closing numbers guide hyperparameter choice of the number of neurons in each hidden layer of a neural network and suggest how often we can expect closed decision regions under different activation functions.
Using the \textit{bump activation} in hyper-hill classifiers to model equality separation, we interpret hyper-hill classifiers as equality separators across each neural network layer which encloses decision regions almost everywhere.
This property is especially useful in supervised anomaly detection when limited supervision is present for anomalies and for representation learning of the normal class when we require the normal class to be within the margin.
We support these claims empirically, quantitatively and qualitatively showing that equality separation can produce better one-class decision boundaries.
Equality separation can be used with kernels, neural networks and foundation models, striking a good balance detecting seen and unseen anomalies in both low and high dimensions.
\looseness=-1

\subsubsection*{Broader Impact Statement}
Not applicable.
\subsubsection*{Acknowledgments}

The authors thank the anonymous reviewers for their constructive feedback during the review process.
The authors also thank Tri Kien Tran and Tanmay Kenjale for sharing code for MVTec dataset preprocessing.
This work was supported in part by the Department of Energy (DoE) project GR00001113.

\newpage
\bibliography{bib}
\bibliographystyle{tmlr}

\appendix
\begin{appendices}
\newpage
\section*{Appendix contents}
This appendix is ordered as follows.
\begin{enumerate}
    \item Appendix \ref{appendix:connections} mentions connections of the equality separator with other concepts in the literature.
    \item Appendix \ref{proof:vc} proves the VC dimension of equality separators, including a proof in 2D for intuition.
    \item Appendix \ref{proof_pseudo_lrt} shows the connection between SVMs and equality separation through a formulation of a likelihood ratio test.
    \item Appendix \ref{appendix:closing_numbers} proves the closing numbers of equality separators and halfspace separators, showing why equality separators can induce closed decision regions better than halfspace separators.
    \item Appendix \ref{appendix:bump_and_loss_choices} discusses the different choices of bump activation functions and loss functions, which bias the hypothesis to model different things. In summary, the choosing different bump activation functions may model a hyperplane compared to a margin, while mean-squared error loss may be more suitable than logistic loss when there is label noise, such as through anomaly sampling.
    \item Appendix \ref{appendix:experiment_details} lays out all experimental details for reproducibility, ablations and visualizations.
    For instance, we show that the margin width increases in earlier layers while decreasing in later layers, which corresponds to the intuition of deeper layers learning more specialized information and being more selective to activate on fewer inputs (Appendix \ref{appendix:non_linear_ad}).
    Code used can be accessed at \url{https://github.com/mattlaued/XOR-is-Linearly-Classifiable}.
\end{enumerate}

\newpage
\section{More connections of equality separators}
\label{appendix:connections}

We explore more connections of equality separators with other concepts in machine learning.

\subsection{Connection to hyperrectangle classifiers}

Since the halfspace separator is tightly connected to our decision rule, our decision rule is also connected to other classifiers that are related to halfspace separators. To illustrate this, we use one example of the hypothesis class being the set of all hyperrectangles in $\mathbb R^n$.
Just as how using hyperrectangles as classifiers in $\mathbb R^n$ can be seen as using the intersection of $2n$ halfspaces, we can also see its relation to our $\epsilon$-error separator.

For each dimension $i\in [n]$, define a weight vector $\mathbf w_i\in \mathbb R^n$, a bias term $b_i\in \mathbb R$ and a width parameter $\epsilon_i\in \mathbb R^n$ that will control the width of the hyperrectangle in the $i^{th}$ dimension.
Hyperrectangle classifiers can be produced by the intersection of these $n$ $\epsilon$-error separators.
In fact, we can also generalize this to hyperparallelograms too, by removing the constraint of the orthogonality of the weight vector.

\subsection{Connection with linear and support vector regression}
Let us view equality separation from a regression perspective and observe the similarities and differences between equality separation and regression.
Due to the equality form of $\mathbf w^T\mathbf x+b=0$, the strict equality separator is similar to a linear regression model for one class.
Meanwhile, our $\epsilon$-error separator can be seen as a support vector regression problem for one class, with $\epsilon$ defining the error/noise tolerated.
In support vector regression (SVR) \citep{SVR}, points within the margin are not penalized for being a small distance away from the hyperplane.

However, our approach differs from regression in two ways.
First, equality separation handles anomalies differently in supervised anomaly detection (AD).
Our optimization scheme accounts for information provided by the other class, which can act as a form of regularization.
In a vanilla linear regression problem, the negative class may not matter so much during training, and we can discard the negative class during training and perform linear regression with only the positive class.
However, in a learning setting with a non-linear embedding function (like a neural network), the negative class can help with updating the embedding function such that the negative class will not lie on the hyperplane (and can help prevent degenerate solutions, such as projecting all points onto a single point). 
SVR can account for anomalies during training, but anomalies are identified in an unsupervised fashion and cannot utilize labeled information.
For a classification task, the distance from the hyperplane determines how similar a data point is to the class defined by the hyperplane, and we can also leverage label information to learn this.
This inductive bias is especially useful for one-class classification.

Second, equality separation is agnostic to the exact position on the regression manifold.
The aim of regression is to transform a datum into its corresponding response variable.
In equality separation, we wish to classify the datum, so we are only concerned if the transformed datum lies somewhere on the regression manifold (of some class), indicating that it belongs to the corresponding class.

For sake of concreteness, consider the homogenous linear case.
The output of the equality separator would tell us if the input was a valid input by checking membership to the margin.
On the other hand, the SVR model would tell us what the response of the input is
given that the input is valid ($\mathbf b$ can be broken down into the sum of a learnt translation and the response variable for the corresponding input during training). 
More generally, given that there is some (unknown) regression task with a non-linear relationship, neural networks with equality separators can be viewed as first embedding the input manifold onto a hyperplane (with margin) and checking the validity of the input by verifying if the embedded input lies on the hyperplane. 
If a regression task were specified, then we would instead want the output of the hyperplane. For equality separators implemented in neural networks, there are many possible embeddings that will regress inputs to different parts of the hyperplane, but as long as the input is embedded onto the hyperplane, our task of classification is done.

Such a perspective also suggests that equality separation is similar but ideologically different from classification:
traditional classification seeks to make a statement about a datum being more likely than not, while equality separation is more concerned about manifold membership.\footnote{We are not claiming that classification tasks are not about manifold membership. Rather, we are suggesting that it is more natural to view manifold membership from equality separation perspective.}
This is probabilistic in nature, and its geometric implication is that data are tipped over decision boundaries.
Ideologically, there are no `optimal' data, only more likely ones.
In this way, equality separation is closer to regression in the sense that there is an optimal point to reach zero loss, which is exactly on the hyperplane.

\subsection{More details on hyper-ridge, hyper-hill classification and one-vs-set machine}

We provide more details on the models most related to equality separation in the broader literature, which are hyper-ridge classification, hyper-hill classification and one-vs-set machines.
In summary, hyper-ridge classifiers, hyper-hill classifiers and one-vs-set machines all do not allow label flipping and specify that the positive class should be within the margin, which is the constrained version of equality separation in AD applications.

The hyper-ridge network \citep{nonmonotonic_activations_mlp_thesis} uses the ridge activation function, with the classification rule being
\begin{align*}
    \mathrm{sign}(1-(\mathbf w^T \mathbf x + b)^2).
\end{align*}
The hyper-hill network \citep{nonmonotonic_activations_mlp_thesis} is a smooth approximation of the hyper-ridge network, inducing a ``nonmonotonic, `bump-like', activation function'' such as the Gaussian
\begin{align*}
    \exp \left (-\frac{(\mathbf w^T \mathbf x + b)^2}{\sigma^2} \right )
\end{align*}
as in (\ref{eqn:bump_activation}). 
These two networks are proposed for regular binary classification.

On the other hand, one-vs-set machine is proposed for multi-class classification while reducing open space risk for open set recognition.
It was first proposed in \citet{towards_open_set_recognition}, but has also extended in \citet{best_fitting_hyperplane_classification}.
It trains a linear SVM to first find a linear decision boundary between 2 classes, then produces a second parallel hyperplane to classify the positive class within the margin.

These 3 classification approaches are fundamentally the same, finding the margin where the positive class lies in.
They merely differ in how this margin is found: one-vs-set machine does this linearly while hyper-ridge and hyper-hill classifiers do this with many non-linear layers.

However, these 3 approaches do not explicitly state the linear condition for classification while we do, which is the generalization of the $\phi$ labeling function to the indicator function.
This linear observation leads to a more generalized version, which emphasizes
the asymmetry between the two classes where one class falls in the margin while the other falls outside, so we instead refrain from specifying which class falls in the margin without prior knowledge.
This classifier (equality separator) is linear and admits twice the VC dimension of halfspace separation, while these 3 other classifiers admit the same VC dimension of halfspace separation \citep{nonmonotonic_activations_mlp_thesis}.
As a concrete example, even though these 3 other classifiers can classify \texttt{XOR}, they cannot classify both \texttt{OR} and \texttt{AND}.
However, equality separation allows label flipping from the observation that labels in regular binary classification are arbitrary.
This insight is intuitive but (1) immediately makes our classifier have twice the expressivity and (2) suggests that the previous models are helpful when domain knowledge informs us of the class belonging to the margin.

In particular, equality separation reduces to these 3 classification approaches in AD applications, because we assign the normal class to be the positive class (i.e. the class within the margin).
In AD, we assert prior knowledge of the specific labels of the classes, where the labels are not arbitrary and have specific meaning (normal versus anomalies i.e. the complement of normal data).
\looseness=-1

\newpage
\section{VC dimension proofs and comments}
\label{proof:vc}

To layer on the intuition, we will first prove the VC dimension in $\mathbb R^2$ of the strict equality separator for binary variables, then for the whole space of $\mathbb R^2$ and finally, the general version of $\mathbb R^n$.
The former two will provide good intuitions for the generalization to $\mathbb R^n$.
We note that we will assume that $n\geq2$ in this section since the case for $n=1$ is trivial.
We denote $\mathcal H$ as the hypothesis class of strict equality separators, as defined in Definition \ref{def:strict_equality_separator}.
Let $\mathrm{VCdim}(\mathcal F)$ denote the VC-dimension of a function class $\mathcal F$.

\subsection{Binary (boolean) variables in $\mathbb R^2$}
\label{binary variables in R2}
 We first go through a concrete example to illustrate how we can find hypotheses to classify all points using label flipping.

In the special case where $\mathbf x$ is binary (ie. has entries which are binary variables) in $\mathbb R^2$, we can observe that $\mathrm{VCdim}(\mathcal H) = 4$, which is the maximum number of points in the binary feature space in $\mathbb R^2$.

To see how one can always shatter the 4 points of $(0,0), (0,1), (1,0)$ and $(1,1)$, we can use the following rule:
\begin{enumerate}
\item If there is a class that does not have any point that belongs to it, then a line passing through none of the points suffices as a classifier.
\item If there is a class that has only 1 point that belongs to it, then a line passing through that point that does not pass through any other point suffices as a classifier. Examples are given in the $\texttt{AND}$ and $\texttt{OR}$ cases.
\item If there is a class that has only 2 points that belong to it, then a line passing through those 2 points as a classifier. Example is given in the $\texttt{XOR}$ cases.
\end{enumerate}
Other cases are equivalent to label flipping, which can be captured by the hypothesis class $\mathcal H$. Hence, the set of 4 points can be shattered by $\mathcal H$.

\subsection{VC dimension in $\mathbb R^2$}
\label{proof:vc_R2}

In fact, we can analyse a more general setting where the elements of $\mathbf x$ live in $\mathbb R^2$ rather than being contrained to binary variables.
In this more general setting, we obtain the following theorem.
\begin{theorem}
\label{VC_Dim_2D}
Let $\mathcal H=\{\mathds{1}(\mathbf w^T\mathbf x+b = 0)\} \cup \{\mathds{1}(\mathbf w^T\mathbf x+b\neq 0)\}$ and $\mathbf x\in \mathbb R^2$. Then, $\mathrm{VCdim}(\mathcal H) = 5$.
\end{theorem}
The VC-dimension of $\mathcal H$ in $\mathbb R^2$ is an increase of 2 from the VC-dimension of the half-space separator.
The increase in 2 is due to the assignment rule as proposed in Appendix \ref{binary variables in R2}.

To prove the VC-dimension, we will construct a 2-sided inequality with lemmas \ref{VC geq 5 R2 lemma} and \ref{VC leq 5 R2 lemma}.

\begin{lemma}
\label{VC geq 5 R2 lemma}
Let $\mathcal H=\{\mathds{1}(\mathbf w^T\mathbf x+b = 0)\} \cup \{\mathds{1}(\mathbf w^T\mathbf x+b\neq 0)\}$ and $\mathbf x\in \mathbb R^2$. Then, $\mathrm{VCdim}(\mathcal H) \geq 5$.
\end{lemma}
\begin{proof}
We find a set of 5 distinct points that are shattered by $\mathcal H$. Let these 5 points be arranged on a unit circle centred at the origin, equally spaced (in the complex plane, the points will be the $5^{th}$ roots of unity).
Following the assignment rule in Appendix \ref{binary variables in R2}, we can always draw a line that classifies the points.
Therefore, these 5 points are shattered by $\mathcal H$.
\end{proof}

\begin{lemma}
\label{VC leq 5 R2 lemma}
Let $\mathcal H=\{\mathds{1}(\mathbf w^T\mathbf x+b = 0)\} \cup \{\mathds{1}(\mathbf w^T\mathbf x+b\neq 0)\}$ and $\mathbf x\in \mathbb R^2$. Then, $\mathrm{VCdim}(\mathcal H) \leq 5$.
\end{lemma}

\begin{proof}
We will prove that any 6 distinct points cannot be  shattered by $\mathcal H$.\footnote{The case where points are not distinct is trivial to show the inability to shatter, so we will not explicitly mention this fact in future analyses.}
Consider 2 exhaustive cases: (i) when there does not exist 3 colinear points and (ii) when there exists 3 points that are colinear (ie. lie on the same line) .

In case (i), any labelling that has 3 points from class $0$ and the remaining 3 points from class $1$ will not produce a perfect classifier.
Since 3 points belong to each class, then if there exists a perfect classifier from $\mathcal H$, then the line must pass through all 3 points from either class.
Since no 3 points are colinear, then there is no line that passes through 3 points from either class.
Therefore, any 6 points that do not have 3 colinear points are not shattered by $\mathcal H$.

In case (ii), it is not possible to shatter the 6 points.
Denote the concept (ie. the true function) as $c:\mathcal X \to \mathcal Y$.
Consider 3 colinear points $\mathbf a_1, \mathbf a_2, \mathbf a_3\in \mathbb R^2$. Denote the other points as $\mathbf d_1, \mathbf d_2, \mathbf e \in \mathbb R^2$.
Let $c(\mathbf a_1)=c(\mathbf d_1)=c(\mathbf d_2)=0$ and $c(\mathbf a_2)=c(\mathbf a_3)=c(\mathbf e)=1$.
The linear classifier from $\mathcal H$ has to be a line that lies on all the classes of 1 class.
However, there is no line that lies on all the points from class $1$, because that line will either 
\begin{enumerate}
    \item contain a point from class $0$ (e.g. the line passing through $\mathbf a_2$ and $\mathbf a_2$ also passes through $\mathbf a_1$), or
    \item not contain all the points from class $1$ (because a line that does not pass through $\mathbf a_2$ and $\mathbf a_3$ will not have $\mathbf a_2$ or $\mathbf a_3$).
\end{enumerate}
Hence, the line needs to pass through all 3 points from class $0$. If such a line does not exist, then the proof is done.
If such a line exists, then there exists some points $\mathbf d_1, \mathbf d_2\in \mathbb R^2$ that are colinear to $\mathbf a_1$ but not to $\mathbf a_2$, $\mathbf a_3$ or $\mathbf e$.
Then, define a new labelling function $c'$ that switches the labels of $\mathbf d_2$ and $\mathbf e$. In other words, $c'(\mathbf a_1)=c'(\mathbf d_1)=c'(\mathbf e)=0$ and $c'(\mathbf a_2)=c'(\mathbf a_3)=c'(\mathbf d_2)=1$.
Then, we face a similar difficulty in classifying the colinear points $\mathbf a_1, \mathbf d_1$ and $\mathbf e$, because no line contains all class $0$ points.
Since there exists a labelling where no hypothesis from $\mathcal H$ can classify these 6 points, then any 6 points that have 3 colinear points are not shattered by $\mathcal H$.

Thus, there does not exist any set of 6 points that can be shattered by $\mathcal H$.
\end{proof}

\subsection{ Proof of Theorem \ref{theorem:vc_strict_equality}: VC dimension in $\mathbb R^n$}
\label{proof:vc_Rn}
 Before proving Theorem \ref{theorem:vc_strict_equality}, we first restate it:
\begin{theorem}
Let $\mathbf x\in \mathbb R^n$ for $n\in\{2,3,...\}$, and the hypothesis class be the strict equality separator $\mathcal H = \{\mathds{1}(\mathbf w^T\mathbf x+b = 0)\} \cup \{\mathds{1}(\mathbf w^T\mathbf x+b\neq 0)\}$.
Then, $\mathrm{VCdim}(\mathcal H) = 2n+1$.
\end{theorem}

The classical halfspace separator gives a VC dimension of $n+1$. An intuitive way to interpret the factor of $2$ in the $n$-term is that we allow 2 kinds of hypotheses -- one that asserts equality and one that does not. In essence, our hypothesis class allows for label flipping, which gives it the flexibility of expression. Hence, the equality separator hypothesis class is on the order of twice as expressive as the halfspace separator. However, allowing label flipping for the halfspace separator does not increase its expressivity. The asymmetry of the equality separator gives rise to this property that label flipping is indeed useful, doubling its expressivity.

We break  Theorem \ref{theorem:vc_strict_equality} into a 2-sided inequality with lemmas \ref{lemma:VC_general_greater} and \ref{lemma:VC_general_lesser}.

\begin{lemma}
\label{lemma:VC_general_greater}
Let $\mathcal H=\{\mathds{1}(\mathbf w^T\mathbf x+b = 0)\} \cup \{\mathds{1}(\mathbf w^T\mathbf x+b\neq 0)\}$ and $\mathbf x\in \mathbb R^n$. Then, $$\mathrm{VCdim}(\mathcal H) \geq 2n+1.$$
\end{lemma}
\begin{proof}
Consider a set of $2n+1$ distinct points such that for any hyperplane in $\mathbb R^n$, only a maximum of $n$ points lie on the hyperplane.
Then, for any labelling of the $2n+1$ points, let the positive class be the class that has the fewest data points that belong to it.
(We can do this because our hypothesis class allows for label flipping.)
Let the number of data points belonging to this positive class be $k$.
Notice that $k\leq\lfloor \frac{2n+1}{2} \rfloor=n$.
Denote the non-empty set of all hyperplanes that contain all of these at most $k$ points as $H$.

If $k=n$, then $|H|=1$ (ie. the hyperplane is unique).
Note that this unique hyperplane does not contain any other points, by construction of the data points, so the hyperplane is a perfect classifier on this labelling.

Otherwise, if $k<n$, then $|H|=\infty$ (ie. there are infinitely many hyperplanes).
With induction, we prove that there is a hyperplane that passes through $k$ points that does not pass through the other $(2n+1)-k$ points.

\paragraph{Base case: $k=n-1$.} 
For any set of $k$ points,
let $H$ be the set of all hyperplanes that pass through these $k$ points, and let $H'\subseteq H$ be the set of all hyperplanes that pass through the $k$ points but also pass through at least 1 of the other $(2n+1)-k$ points.
Then,
$|H|=\infty$ while $|H'|=(2n+1)-k=n+2$ because the other $n+2$ points (ie. the points from the negative class) restrict us from using those hyperplanes for classification.
Since $|H|>|H'|$ and $H'\subseteq H$, then $H'$ is a proper subset of $H$, and there exists a hyperplane in $H$ that is not in $H'$.
By definition, this hyperplane contains all $k$ points and none of the other $(2n+1)-k$ points.

\paragraph{Inductive step:
Assume that there exists a hyperplane that passes through any $k+1$ points and does not pass through the other $(2n+1)-(k+1)$ points, for $0\leq k<n-1$. Then, we prove that there exists a hyperplane that passes through any $k$ points.}

Consider any set of $k+1$ points (and denote this as $P$) and consider a hyperplane $h$ that passes through the $k+1$ points but not the other $2n-k$ points.
Let the set of the other $2n-k$ points be $Q$ (ie. $P$ and $Q$ form a partition of the set of all $2n+1$ points).
Let $\mathbf p \in P$ be any point in $P$ (since $P$ is non-empty).
Then, 
there is an added free variable (ie. added degree of freedom) to define a hyperplane that passes through $P\backslash \{\mathbf p\}$.
To avoid passing through points in $Q\cup \{\mathbf p\}$, this free variable is unable to take on a finite number of values to avoid the $(2n+1)-k$ points.
Since the variable can take on an infinite number of values except a finite number of values, there is a value that this variable can take on.
Using this value, we can find a hyperplane that passes through all $k$ points in $P$ while not passing through the $(2n+1)-k$ points in $Q\cup  \{\mathbf p\}$.

By induction, we have shown that for all $k\in \{0, 1, ..., n-1\}$, that there exists a hyperplane that passes through $k$ points but not the other $(2n+1)-k$ points.
Hence, there always exists a hyperplane to classify the $k$ points from the positive class such that none of the $(2n+1)-k$ points from the negative class are on this hyperplane.

Since a perfect classifier from $\mathcal H$ always exists for the aforementioned dataset for any labelling, then there exists a set of $2n+1$ points in $\mathbb R^n$ that $\mathcal H$ can shatter.
Thus, $\mathrm{VCdim}(\mathcal H)\geq 2n+1$.
\end{proof}

\vspace{1mm}

\begin{lemma}
\label{lemma:VC_general_lesser}
Let $\mathcal H=\{\mathds{1}(\mathbf w^T\mathbf x+b = 0)\} \cup \{\mathds{1}(\mathbf w^T\mathbf x+b\neq 0)\}$ and $\mathbf x\in \mathbb R^n$. Then, $$\mathrm{VCdim}(\mathcal H) \leq 2n+1.$$
\end{lemma}

The proof only requires basic linear algebra and the idea is as follows.
We desire to find a labeling on $2n+2$ points such that there is no strict equality separator for that labeling.
For some labeling, we consider (translated) subspaces of the 2 classes defined by the points from that class and relabel points until we find a labeling that satisfies the above condition.
In the arguments, we often consider 2 classes with $n$ points each, and leave the remaining 2 points to help us produce the desired labeling by labeling the 2 points with the corresponding labels and sometimes swapping the labeling of points already with labels.
 We include the high level overview of the proof in Figure \ref{fig:vcdim_proof}.

\begin{figure}
    \centering
    \includegraphics[width=\linewidth]{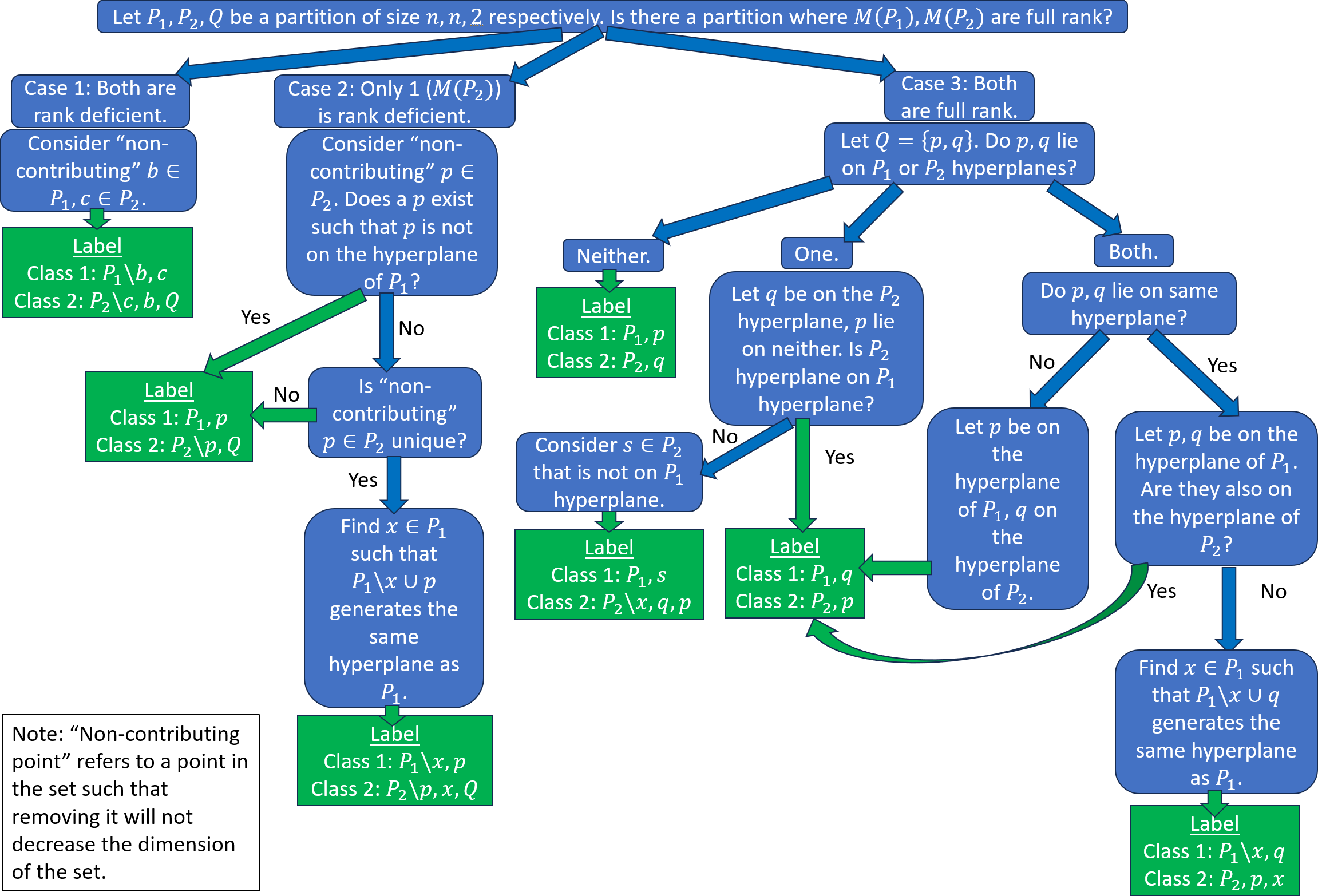}
    \caption{Proof overview flow chart of Lemma \ref{lemma:VC_general_lesser}.}
    \label{fig:vcdim_proof}
\end{figure}

\begin{proof}

Consider any set of distinct $2n+2$ points in $\mathbb R^n$, and denote this set as $X$.
We aim to show that there exists a labelling on these points such that no $h\in \mathcal H$ can classify the points perfectly (i.e. $\mathcal H$ cannot shatter any set of $2n+2$ points).
We will exhaustively consider cases and show that in all cases, there exists a labelling such that no classifier in $\mathcal H$ can perfectly classify all the points on such a labelling.

Note that if $X$ lies on a hyperplane, then there exists a labelling half of $X$ as positive and the other half as negative that would fail to produce a perfect classifier. Hence, we consider the fact that $X$ does not lie on one hyperplane.

Let $\{P_1, P_2, Q\}$ be a partition of $X$, where $|P_1|=|P_2|=n$ and $|Q|=2$.
Let $M:A\to \mathbb R^{n\times |A|}$
be the function that outputs a matrix where each column is an element of $A$.
Formally, for non-empty, finite set $A:=\{\mathbf a_1,...,\mathbf a_{|A|}\}\subseteq X$, define $M$ as a function that generates the matrix $M(A) = [\mathbf a_1,...,\mathbf a_{|A|}]$ from $A$.

\paragraph{Case 1}
If there exists a partition of $X$ such that $M(P_1)$ and $M(P_2)$ are not full rank, then consider points $\mathbf b\in P_1$ and $\mathbf c\in P_2$ such that $\mathrm{rank}(M(P_1\backslash \{\mathbf b\})) = \mathrm{rank}(M(P_1))$ and $\mathrm{rank}(M(P_2\backslash \{\mathbf c\})) = \mathrm{rank}(M(P_2))$.
Such points exist based on a counting argument, given that $M(P_1)$ and $M(P_2)$ are not full rank.
Then, consider the labelling such that $P_1\backslash \{\mathbf b\}, \mathbf c$ belong to class one and $P_2\backslash \{\mathbf c\}, \mathbf b, Q$ belong to class two.
A perfect classifier from $\mathcal H$ does not exist on this labelling: using a hyperplane generated by $P_1\backslash \{\mathbf b\}$%
\footnote{A hyperplane such that all points in $P_1\backslash\{\mathbf b\}$ lie on it.}
to label class one as the positive class will classify $\mathbf b$ wrongly because $\mathbf b\in \mathrm{col}(M(P_1\backslash\{\mathbf b\}))$ (by assumption) will be classified as negative although it is from the positive class.
Likewise, using a hyperplane generated by $P_2\backslash \{\mathbf c\}$ to label class two as the positive class will classify $\mathbf c$ wrongly because $\mathbf c\in \mathrm{col}(M(P_2\backslash\{\mathbf c\}))$.

Note that using hyperplanes other than the ones described will also fail to produce perfect classifiers, because points in $P_1$ (when class one is the positive class) or $P_2$ (when class two is the positive class) will not be classified correctly. We omit this comment in further analyses.

\paragraph{Case 2}
Otherwise, consider the case where there exists a partition of $X$ such that either $M(P_1)$ or $M(P_2)$ is full rank.
Without loss of generality, let
$M(P_2)$ be rank deficient, while $M(P_1)$ is full rank.
By the same counting argument, there exists non-empty set $K\subseteq P_2$ such that $\mathrm{rank}(M(P_2\backslash K)) = \mathrm{rank}(M(P_2))$.

If there is a point $\mathbf p\in K$ for any $K$ such that $\mathbf p\not \in \mathrm{col}(M(P_1))$, then consider the labelling such that $P_1, p$ belong to class one and $P_2\backslash\{\mathbf p\}, Q$ belong to class two.
A perfect classifier from $\mathcal H$ does not exist on this labelling: using the unique hyperplane generated by $P_1$ to label class one as the positive class will classify $\mathbf p$ wrongly, because $\mathbf p\not\in \mathrm{col}(M(P_1))$.
On the other hand, using a hyperplane generated by $P_2\backslash \{\mathbf p\}$ to label class two as the positive class will classify $\mathbf p$ wrongly because $\mathbf p\in \mathrm{col}(M(P_2\backslash\{\mathbf p\}))$.

Thus, consider that there is no such $\mathbf p$ that exists (ie. for any $K$ that fulfills the constraints, $\mathbf p\in K$ will belong to the column space of $M(P_1)$).
Then, pick $K$ such that $K$ is the largest possible set that fulfills the constraints (ie. $\mathrm{rank}(M(P_2\backslash K)) = \mathrm{rank}(M(P_2))$).
Note that $|K|\geq 1$.

If $|K| > 1$, then consider 2 distinct points $\mathbf p, \mathbf q\in K$ and the labelling such that $P_1, \mathbf p, Q$ belong to class one and $P_2 \backslash\{\mathbf p\}$  belong to class two.
A perfect classifier from $\mathcal H$ does not exist on this labelling: using the unique hyperplane generated by $P_1$ to label class one as the positive class will classify $\mathbf q$ wrongly, because $\mathbf q\in \mathrm{col}(M(P_1))$.
On the other hand, using a hyperplane generated by $P_2\backslash \{\mathbf p\}$ to label class two as the positive class will classify $\mathbf p$ wrongly because $\mathbf p\in \mathrm{col}(M(P_2\backslash\{\mathbf p\}))$.

Otherwise, if $|K|=1$, then consider $\mathbf x\in P_1$ such that $P_1\backslash\{\mathbf x\}\cup K$ generates a unique hyperplane which is the same hyperplane as the unique hyperplane generated by $P_1$.
Such a point exists because
$\mathbf p \in K$ is a linear combination of points in $P_1$.
Hence, such a point $\mathbf x$ exists.%
\footnote{We repeat a similar argument in future analyses, and omit this comment in the future for brevity.}
Then, consider the labelling where $P_1\backslash\{\mathbf x\}, K$ belong to class one and $P_2\backslash K, \mathbf x, Q$ belong to class two.
A perfect classifier from $\mathcal H$ does not exist on this labelling: using the unique hyperplane generated by $P_1\backslash\{\mathbf x\}\cup K$ to label class one as the positive class will classify $\mathbf x$ wrongly, because $\mathbf x\in \mathrm{col}(M(P_1))$.
On the other hand, using a hyperplane generated by $P_2\backslash K$ to label class two as the positive class will classify $\mathbf p\in K$ wrongly because $\mathbf p\in \mathrm{col}(M(P_2\backslash\{\mathbf p\}))$.

\paragraph{Case 3}
Lastly, we consider the case such that there is no such partition such that both $M(P_1)$ and $M(P_2)$ are not rank deficient.
In other words, the 2 hyperplanes generated by $P_1$ and $P_2$ respectively are both unique.
Consider $Q:=\{\mathbf p, \mathbf q\}$.

If both $\mathbf p$ and $\mathbf q$ do not lie in the column space of either $M(P_1)$ or $M(P_2)$, then consider the labelling where $P_1,\mathbf  p$ belong to class one and $P_2, \mathbf q$ belong to class two.
A perfect classifier from $\mathcal H$ does not exist on this labelling: using the unique hyperplane generated by $P_1$ to label class one as the positive class will classify $\mathbf q$ wrongly, because $\mathbf p\not\in \mathrm{col}(M(P_1))$.
On the other hand, using the unique hyperplane generated by $P_2$ to label class two as the positive class will classify $\mathbf q$ wrongly because $\mathbf q\not\in \mathrm{col}(M(P_2))$.

Otherwise, there exists a point in $Q$ that lies in the column space of either $M(P_1)$ or $M(P_2)$.
Consider the case where one point does not lie in the column space of either $M(P_1)$ or $M(P_2)$.
Without loss of generality, let $\mathbf p$ be the point that does not lie in the column space of either $M(P_1)$ or $M(P_2)$, and let $\mathbf q\in \mathrm{col}(M(P_2))$.
If $\mathrm{col}(M(P_2)) \subseteq \mathrm{col}(M(P_1))$, then consider the labelling where $P_1, \mathbf q$ belong to class one and $P_2, \mathbf p$ belong to class two.
A perfect classifier from $\mathcal H$ does not exist on this labelling: using the unique hyperplane generated by $P_1$ to label class one as the positive class will classify some point $\mathbf x\in P_2$ wrongly, because $\mathbf x\in \mathrm{col}(M(P_2))\subseteq \mathrm{col}(M(P_1))$ is classified as the positive class although it belongs to the negative class.
Meanwhile, using the unique hyperplane generated by $P_2$ to label class two as the positive class will classify $\mathbf p$ wrongly because $\mathbf p\not\in \mathrm{col}(M(P_2))$.\\
Else, $\mathrm{col}(M(P_2)) \not\subseteq \mathrm{col}(M(P_1))$, so there exists a point $\mathbf s\in P_2$ such that $\mathbf s\not\in \mathrm{col}(M(P_1))$.
Then, $M(P_2\backslash\{\mathbf s\}\cup\{\mathbf q\})$ is full rank because $\mathbf q$ is a linear combination of all points in $P_2$.
 Hence, the unique hyperplane generated by $P_2\backslash\{\mathbf s\}\cup\{\mathbf q\}$ is the same hyperplane as the unique hyperplane generated by $P_2$.
Then consider the labelling where $P_1, \mathbf s$ belong to class one and $P_2\backslash\{s\}, \mathbf q, \mathbf p$ belong to class two.
A perfect classifier from $\mathcal H$ does not exist on this labelling: using the unique hyperplane generated by $P_1$ to label class one as the positive class will classify $\mathbf s$ wrongly, because $\mathbf s\not\in \mathrm{col}(M(P_1))$.
On the other hand, using the unique hyperplane generated by $P_2\backslash\{\mathbf s\}\cup\{\mathbf q\}$ to label class two as the positive class will classify $\mathbf p$ wrongly because $\mathbf p\not\in \mathrm{col}(M(P_2))$.

Hence, consider the case where both $\mathbf p$ and $\mathbf q$ lie on the column space of $M(P_1)$ or $M(P_2)$.
If they lie on different column spaces, without loss of generality, let $\mathbf p\in \mathrm{col}(M(P_1))$ and $\mathbf q\in \mathrm{col}(M(P_2))$.
Note that $\mathbf p\not\in \mathrm{col}(M(P_2))$ and $\mathbf q\not\in \mathrm{col}(M(P_1))$ by our assumption.
Then consider the labelling where $P_1, \mathbf q$ belong to class one and $P_2, \mathbf p$ belong to class two.
A perfect classifier from $\mathcal H$ does not exist on this labelling: using the unique hyperplane generated by $P_1$ to label class one as the positive class will classify $\mathbf q$ wrongly, because $\mathbf q\not\in \mathrm{col}(M(P_1))$.
On the other hand, using the unique hyperplane generated by $P_2$ to label class two as the positive class will classify $\mathbf p$ wrongly because $\mathbf p\not\in \mathrm{col}(M(P_2))$.

Thus, we consider the case where $\mathbf p$ and $\mathbf q$ lie on the same column space.
Without loss of generality, let $\mathbf p, \mathbf q \in \mathrm{col}(M(P_1))$.
If there is a point that does not lie in the column space of $M(P_2)$, then without loss of generality, let this point be $\mathbf p$.
Consider a point $\mathbf x\in P_1$  the unique hyperplane generated by $P_1\backslash\{\mathbf x\}\cup\{\mathbf q\}$ is the same hyperplane as the unique hyperplane generated by $P_1$, because $\mathbf q$ is a linear combination of points in $P_1$.
Then consider the labelling where $P_1\backslash\{\mathbf x\}, \mathbf q$ belong to class one and $P_2, \mathbf p, \mathbf x$ belong to class two.
A perfect classifier from $\mathcal H$ does not exist on this labelling: using the unique hyperplane generated by $P_1$ to label class one as the positive class will classify $\mathbf x$ wrongly, because $\mathbf x\in \mathrm{col}(M(P_1))$ will be labelled as positive although it is from the negative class.
On the other hand, using the unique hyperplane generated by $P_2$ to label class two as the positive class will classify $\mathbf p$ wrongly because $p\not\in \mathrm{col}(M(P_2))$.

Otherwise, consider the case where both $\mathbf p,\mathbf q\in \mathrm{col}(M(P_1))$ and $\mathbf p,\mathbf q\in \mathrm{col}(M(P_2))$.
Then consider the labelling where $P_1, \mathbf q$ belong to class one and $P_2, \mathbf p$ belong to class two.
A perfect classifier from $\mathcal H$ does not exist on this labelling: using the unique hyperplane generated by $P_1$ to label class one as the positive class will classify $\mathbf p$ wrongly, because $\mathbf p\in \mathrm{col}(M(P_1))$ will be labelled as positive although it is from the negative class.
On the other hand, using the unique hyperplane generated by $P_2$ to label class two as the positive class will classify $\mathbf q$ wrongly because $\mathbf q\in \mathrm{col}(M(P_2))$ will be labelled as positive although it is from the negative class.

We have exhaustively considered cases and showed that, for any set of $2n+2$ distinct points, $\mathcal H$ is unable to shatter it.
Thus, $\mathrm{VCdim}(\mathcal H)\leq 2n+1$.
\end{proof}

\subsection{ Proof of Corollary \ref{theorem:vc_eps_err_sep}: VC dimension of $\epsilon$-error separator}

Before proving Corollary \ref{theorem:vc_eps_err_sep}, we first restate it:
\begin{corollary}
For $\epsilon$-error separators $\mathcal{H}_{\epsilon}$ in Def. \ref{def:epsilon_error_separator},
$2n+1\leq \mathrm{VCdim}(\mathcal H_\epsilon) \leq 2n+3 
$.
\end{corollary}

The lower bound of $2n+1$ follows from Theorem \ref{theorem:vc_strict_equality}, since the margin of the $\epsilon$-error separator can be made arbitrarily small.

The upper bound is from the union property of hypothesis classes found in Exercise 6.11 in Chapter 6.8 of \citet{understanding_ml_book}.
The VC dimension of hyper-hill (equivalent to hyper-ridge) classifiers are $n+1$ \citep{nonmonotonic_activations_mlp_thesis}, so we have an upper bound of $(n+1) + (n+1) + 1 = 2n+3$.

Work can be done to show that the VC dimension in $\mathbb R^2$ is $2n+2$.
We conjecture that this is true for all $n>2$ as well, but leave this to future work.
Note that for $n=1$, the VC dimension is $2n+1=3$ -- the reduction in expressivity seems to arise from the distance of a given datum to be a given point rather than a line, plane or hyperplane (which extends to infinity, unlike a point).

\newpage
\section{Proof for separating hyperplane in pseudo likelihood ratio test}
\label{proof_pseudo_lrt}
We provide the proof in Section \ref{section:halfspace_connnection} that, when we use 2 parallel equality separators ($\mathbf w_1 =\mathbf  w_2$) and equal class weighting ($t=1$, ie. a false positive is as bad as a false negative) for binary classification, our pseudo likelihood ratio test
\[\mathds{1}\left (\frac{\epsilon_2}{\epsilon_1}\geq t \right)
=\mathds{1}\left(\frac{1}{||\mathbf w_2||}|\mathbf w_2^T \mathbf {x'} + b_2|-\frac{t}{||\mathbf w_1||}|\mathbf w_1^T \mathbf {x'} + b_1|\geq 0\right)
\]
decomposes to
\[\mathds{1}((b_2-b_1)(2\mathbf w_1^T \mathbf x+ b_1+b_2)\geq 0)\]
which corresponds to the separating hyperplane
\[2\mathbf w_1^T \mathbf x+ b_1+b_2.\]
\begin{proof}
Instead of dealing with absolute values, we take the square on both sides (which does not change the inequality because both sides are positive).
With parallel hyperplane and equal class weighting, we are able to cancel out the square terms to obtain a separating hyperplane.
The exact steps are detailed as follows:
\begin{align*}
\mathds{1}\left (\frac{\epsilon_2}{\epsilon_1}\geq t \right)
&= \mathds{1}\left (\frac{\epsilon_2^2}{\epsilon_1^2}\geq t^2 \right) \\
&= \mathds{1}\left ({\epsilon_2^2}\geq t^2\epsilon_1^2 \right) \\
&= \mathds{1}\left ((\mathbf{w}_2^T \mathbf x +b_2)^2\geq t^2(\mathbf{w}_1^T \mathbf x +b_1)^2\right) \\
&= \mathds{1}\left ((\mathbf{w}_2^T \mathbf x +b_2)^2 -t^2(\mathbf{w}_1^T \mathbf x +b_1)^2  \geq 0\right) \\
&= \mathds{1}\left ((\mathbf{w}_2^T \mathbf x +b_2)^2 -(\mathbf{w}_1^T \mathbf x +b_1)^2\geq 0 \right) & t=1  \\
&= \mathds{1}\left ((\mathbf{w}_1^T \mathbf x +b_2)^2 -(\mathbf{w}_1^T \mathbf x +b_1)^2\geq 0 \right) & \mathbf w_1=\mathbf w_2 \\
&= \mathds{1}\left (2\mathbf{w}_1^T \mathbf x (b_2-b_1) +(b_2^2 - b_1^2)\geq 0 \right) &\text{expand the square} \\
&= \mathds{1}\left (
(b_2 - b_1)(
2\mathbf{w}_1^T \mathbf x  +(b_2 + b_1))\geq 0 \right) \\
&= \mathds{1}\left (
(b_2 - b_1)(
2\mathbf{w}_1^T \mathbf x  +b_1 + b_2)\geq 0 \right)
\end{align*}

If $b_1\neq b_2$, then we have\begin{align*}
\mathds{1}\left (\frac{\epsilon_2}{\epsilon_1}\geq t \right)
&= \mathds{1}\left (
(b_2 - b_1)(
2\mathbf{w}_1^T \mathbf x  +b_1 + b_2)\geq 0 \right)\\
&= \mathds{1}\left (
2\mathbf{w}_1^T \mathbf x  +b_1 + b_2\geq 0 \right).
\end{align*}
\end{proof}

We note that in the case of \texttt{XOR}, $b_1=b_2$ and the hyperplane to classify each class is the same.
Hence, there is no separating hyperplane between the 2 classes.
This is another way we can view the linear non-separability of \texttt{XOR}.
\newpage
\section{Closing Numbers}
\label{appendix:closing_numbers}

Formally, the closing number of a hypothesis class is defined as
\begin{equation*}
    CN(\mathcal H):= \min_{k\in \mathbb N} \{k : \; \exists h_1,...,h_k\in \mathcal H \;, \;  0 < \mathrm{Vol}\left (\bigcap_{i=1}^k h_i \right )<\infty\}
\end{equation*}
where $\mathrm{Vol}$ is the standard Lebesgue measure and $\bigcap_{i=1}^k h_i := \{x\in \mathcal X: \;\; \forall i\in [k], \; h_i(x) = 1\}$ is the set of all points that is labeled positive by all hypotheses.
Note that although possible in reality, we disallow zero-volumed decision regions due to degenerate cases in analysis (e.g. the intersection between two non-intersecting parallel halfspaces is not particularly interesting for classifying data).

Closing numbers are a dimension-like number, similar to VC dimension.
Instead of quantifying the expressivity of a hypothesis class, it quantifies the capacity to reduce open space risk induced by a hypothesis class.
While classical measures refer to the open space risk of a particular hypothesis (similar to the empirical risk of a particular hypothesis), closing numbers refer to the hypothesis class (like VC dimension).
Such a measure of a hypothesis class is useful in choosing which hypothesis classes to use in practice -- the generalization on seen examples (i.e. normal data and seen anomalies) is captured in the VC dimension, while ability to reject unseen examples (i.e. unseen anomalies) is quantified through closing numbers.\footnote{Closing numbers describe the finite-ness of the enclosed decision region, there is no explicit relation on minimizing the volume of this enclosed space of the chosen hypothesis.}

We first note that whenever we refer to a closed decision region in this section, we refer to a region with positive (i.e. non-zero) volume unless otherwise stated.

To prove that the closing number is $k$, it is helpful to show the existence of a closed decision region using the intersection of $k$ hypotheses and the impossibility using fewer (i.e. intersection of $k-1$ hypotheses).

We proceed to prove Theorem \ref{thm:closing numbers}.
To recall, we restate Theorem \ref{thm:closing numbers}:
\begin{theorem}
    Closing number of RBF, equality and halfspace separators is 1, $n$ and $n+1$ respectively in $\mathbb R^n$.
    Closed decision regions are formed with the intersection of the corresponding number of classifiers everywhere for RBF separators and almost everywhere for equality separators, but normal vectors need to form an $n$-simplex volume for halfspace separators.
\end{theorem}

We break the theorem into 3 lemmas:
\begin{lemma}
\label{lemma:closing_number_RS}
    Closing number of RBF separators is 1 in $\mathbb R^n$.
    Closed decision regions are formed everywhere.
\end{lemma}
\begin{lemma}
\label{lemma:closing_number_ES}
    Closing number of equality separators is $n$ in $\mathbb R^n$.
    Closed decision regions are formed with the intersection of $n$ equality separators almost everywhere.
\end{lemma}
\begin{lemma}
\label{lemma:closing_number_HS}
    Closing number of halfspace separators is $n+1$ in $\mathbb R^n$.
    Closed decision regions are formed with the intersection of $n+1$ halfspace separators where normal vectors form an $n$-simplex volume.
\end{lemma}

Note that RBF separators have a closing number of 1 by construction, so Lemma \ref{lemma:closing_number_RS} is trivial.
We proceed to prove the closing numbers of equality separators (Lemma \ref{lemma:closing_number_ES}) and halfspace separators (Lemma \ref{lemma:closing_number_HS}).

\subsection{Proof of Lemma \ref{lemma:closing_number_ES}: closing number of equality separators in $\mathbb R^n$}

To show that the closing number is at most $n$ in $\mathbb R^n$, we observe that a hyperrectangle can be formed through the intersection of $n$ equality separators with orthogonal normal (ie. weight) vectors where the volume in the margin is considered the positive class.
Since the volume of a hyperrectangle is finite, so is this intersection.

\begin{proof}
If we have the intersection of $k<n$ equality separators, we observe a decision region with infinite volume or no volume.
The first case is when there exists an equality separator that assigns the positive class to the region outside of the margin.
This induces an infinite-volume decision region for non-parallel normal vectored equality separators.
Otherwise, the decision region either has no volume (when the intersection of parallel equality separators is a pair of hyperplanes, a hyperplane or the empty set) or infinite volume (in the more general case).
The other case is where all equality separators assign a positive classification to the region within the margin.
Without loss of generality, let the bias terms of all $k$ equality separators be 0 and consider the subspace $\mathcal S$ spanned by the $k$ normal vectors.
Since $k<n$ (or, similarly, in the case where the set of weight vectors of equality separators do not form a basis for $\mathbb R^n$), then there exists a vector $\mathbf{v}$ that is perpendicular to subspace $\mathcal S$.
The volume of the intersection $I$ between these $k$ equality separators can then be formulated as
\begin{align*}
&\mathrm{Vol}(\text{decision region of $k$ equality separators})\\
= &\mathrm{Vol}(\{\mathbf x \in I\} )\\
= &\mathrm{Vol}(\cup_{i\in \mathbb Z} \{\mathbf x \in I : i \leq \mathbf{v}^T \mathbf x \leq i+1\} ) & \exists \mathbf v \perp \mathcal S\\
= &\sum_{i\in \mathbb Z} \mathrm{Vol}( \{\mathbf x \in I : i \leq \mathbf{v}^T \mathbf x \leq i+1\} )
\end{align*}
because the intersections (hyperplanes) have zero measure.
When none of the equality separators are strict equality separators, we obtain an infinite volume since $\mathrm{Vol}( \{\mathbf x \in I : i \leq \mathbf{v}^T \mathbf x \leq i+1\} ) > 0$.
By definition of the Lebesgue measure, the intersection has (at least) infinite volume by moving along the $\mathbf v$ direction.
Otherwise, $\mathrm{Vol}( \{\mathbf x \in I : i \leq \mathbf{v}^T \mathbf x \leq i+1\} ) = 0$, so we have zero volume.
Hence, having fewer than $n$ equality separators cannot induce a closed decision region.
\end{proof}

\paragraph{Conditions for closed decision regions}
Consider $n$ equality separators, which can be represented as a collection of their normal vectors in an $n\times n$ matrix $A\in \mathbb R^{n\times n}$.
The first condition for the finite-volume argument to hold (by observing the formation of a hyperparallelogram) is that these normal vectors form a set of linearly independent vectors (more precisely, that they form a basis for $\mathbb R^n$).
Furthermore, we notice that the set of all rank deficient matrices has zero measure.
An intuition for this is that, if we have any rank deficient matrix $A$, we can always perturb its eigenvalues to obtain a full rank matrix (i.e. non-zero eigenvalues) $A+\lambda I$ for some scalar $\lambda$, where $\lambda$ can be arbitrarily small.
Hence, this set has no volume.

The second condition for equality separators to enclose decision regions is when they are hyper-hill classifiers (i.e. normal data belongs within the margin).
This is how our insight is birthed that we can simply use hyper-hill classification in AD -- we are using a more expressive hypothesis class with equality separation but technically have the VC dimension of hyper-hills (or hyper-ridge classifiers), which is half of equality separators.

\subsection{Proof of Lemma \ref{lemma:closing_number_HS}: closing number of halfspace separators in $\mathbb R^n$}

To show that the closing number is at most $n+1$, we observe that a hypertetrahedron (convex hull of $n+1$ points or a bounded convex polytope with non-empty interior) can be formed with the intersection of $n+1$ hyperplanes, where each of the facets of the convex hull correspond to each of the $n+1$ hyperplanes.

Now, we want to show that $n+1$ is minimal.
Since we are considering the intersection of hyperplanes to form a bounded convex polytope with non-empty interior, we assign a one-to-one correspondence between a hyperplane and a facet of the convex polytope.
In fact, by considering a dual polytope of the convex polytope, there is a further correspondence to each hyperplane and a vertex of the dual polytope.
With non-empty interior, the dual polytope is $n$-dimensional, so has at least $n+1$ vertices.
Hence, we need at least $n+1$ facets.
Thus, to form a bounded convex polytope with non-empty interior with the intersection of hyperplanes, we need at least $n+1$ hyperplanes.

\newpage
\section{Design choices: bump activations and loss functions}
\label{appendix:bump_and_loss_choices}

Using different activations and loss functions to model equality separation introduces different implicit biases in the solutions obtained.
We introduce possible activations and loss functions, exploring an intuition of the implicit bias induced by these differences (and hence, what applications fit better with these implicit biases).

\subsection{Other possible bump activations}
Other activation functions can be used as a bump in addition to the Gaussian bump activation that we used in this paper.
For instance, we can use the inverse-squared of the hyperbolic tangent, \(\mathrm{tanh}(v/z^2)\), where $v$ is a hyperparameter controlling the flatness of the bump function near $z=0$ (Figure \ref{fig:bump_tanh_gaussian}).
Using the plateau can be used to drastically reduce the penalization when the input is within some distance from the hyperplane, which is similar to the idea of hinge loss in SVM.
In other words, the Gaussian bump encodes hyperplane membership, while the inverse-squared of the hyperbolic tangent encodes margin membership (which has a more forgiving penalty for data off the hyperplane).

\begin{figure}[h]
    \centering
    \begin{subfigure}{0.45\linewidth}
    \centering
    \includegraphics[width=1\textwidth, keepaspectratio]{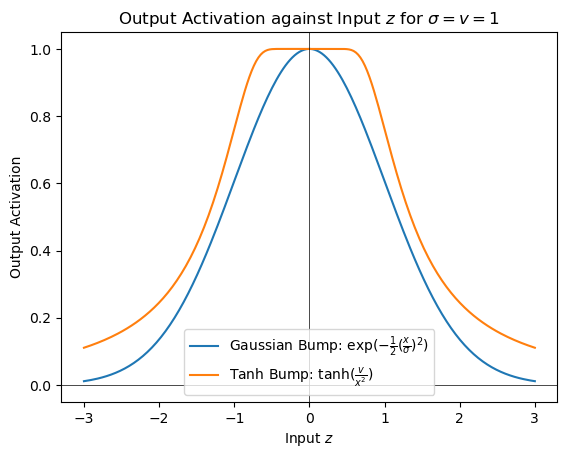}
    \caption{Output of the activation functions.}
    \label{fig:bump_tanh_gaussian_output}
    \end{subfigure}
    \hspace{0.05\linewidth}
    \begin{subfigure}{0.45\linewidth}
    \centering
    \includegraphics[width=1\textwidth, keepaspectratio]{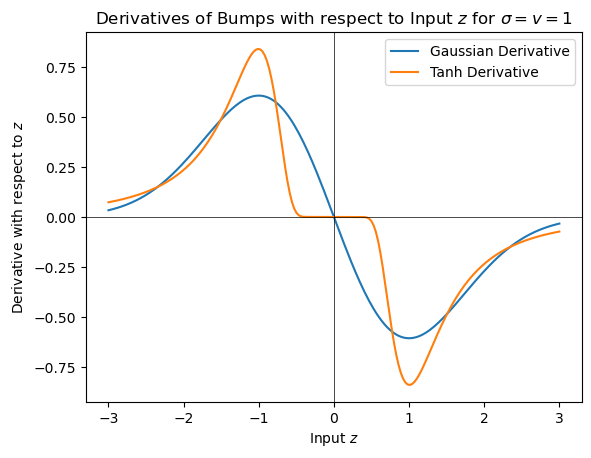}
    \caption{Derivatives of the activation functions.}
    \label{fig:bump_tanh_gaussian_derivatives}
    \end{subfigure}
    \caption{Two different bump activations: the Gaussian bump which we used in blue, and the hyperbolic tangent bump in orange.}
    \label{fig:bump_tanh_gaussian}
\end{figure}

As a note, the Gaussian bump is of the form $[\exp \left ( \mathbf w_i \cdot \mathbf z_i \right )]_{i=1}^d\in \mathbb R^d$ for input and weights $\mathbf z, \mathbf w\in \mathbb R^d$ respectively.
This is different from the Gaussian RBF, which is of the form $\exp \left ( ||\mathbf z- \mathbf c ||_2 \right )\in \mathbb R$ for input and center $\mathbf z, \mathbf c\in \mathbb R^d$ respectively.

\subsection{Loss functions}
\label{appendix:loss_fn}
For input $z$ into the activation function in the output layer, we provide visualizations of how the mean-squared error (MSE) and logistic loss (LL) compare for the negative class (Figure \ref{fig:neg_class_MSEvsLL}) and positive class (Figure \ref{fig:pos_class_MSEvsLL}).
LL penalizes more on wrong classifications for both classes.
Such a scheme is good for learning a stricter decision boundary, but potentially harder to learn when the distributions of the positive and negative class are close and harder when they overlap.
On the other hand, MSE has smaller penalization for wrong classifications, which leads to more tolerance for noise but may lead to learning sub-optimal solutions.
Such a difference hints at different applications, where LL is generally preferred unless the distributions of the positive and negative class are very close (or overlapping, like in negative sampling), in which case the more conservative MSE would probably be more suitable.

\begin{figure}[h]
    \centering
    \begin{subfigure}[b]{0.45\linewidth}
    \centering
    \includegraphics[keepaspectratio, width=0.99\linewidth]{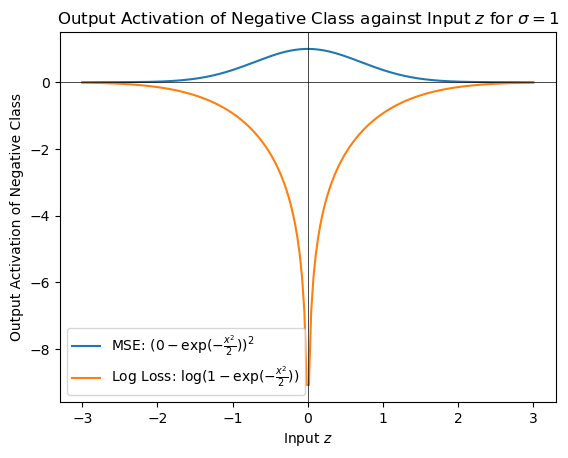}
    \caption{Output of mean-squared error and logistic loss against the input to the bump activation function when the datum is from the negative class.}
    \label{fig:neg_class_MSEvsLL}
    \end{subfigure}
    \hspace{0.07\linewidth}
    \begin{subfigure}[b]{0.45\linewidth}
    \centering
    \includegraphics[keepaspectratio, width=0.99\linewidth]{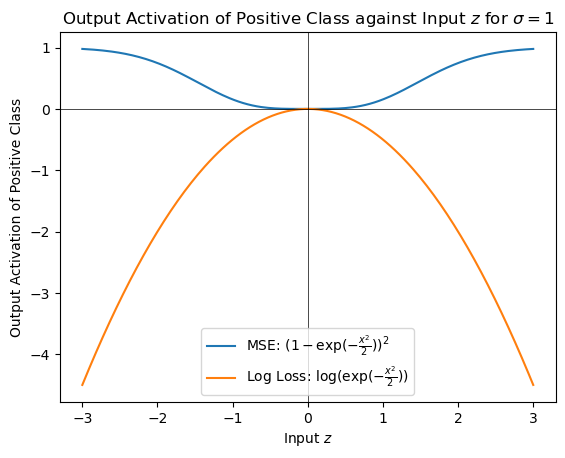}
    \caption{Output of mean-squared error and logistic loss against the input to the bump activation function when the datum is from the positive class.}
    \label{fig:pos_class_MSEvsLL}
    \end{subfigure}
    \caption{Comparison of mean-squared error and logistic loss ($\sigma=1$ in the bump activation).}
    \label{fig:MSEvsLL}
\end{figure}
\newpage

\section{Experimental details}
\label{appendix:experiment_details}
In this section, we provide details of our experiments.
Unless otherwise stated, training and testing both have $m=100$ data samples.

\subsection{Role of $\sigma$ in the bump activation}
\label{section:sigma_role}
The gradient of the
log-concave 
bump activation is
\begin{equation}
\label{gradient_bump}
\frac{\partial}{\partial z}B(z, \mu, \sigma) = -\frac{z-\mu}{\sigma^2} \exp \left [-\frac{1}{2}\left ( \frac{z-\mu}{\sigma} \right )^2 \right ],
\end{equation}
which vanishes to 0 as the input tends towards negative or positive infinity.
To prevent vanishing gradients, we initialize $\sigma$ to be decently large to decrease the sensitivity to weight initialization.
Experiments in Appendix \ref{section:vary_sigma} will shed further light on how the initialization of $\sigma$ affects performance (including the effect of loss functions).
For simplicity, we set $\mu=0$ but give the option for $\sigma$ to be learnable across neurons in the same hidden layer in our experiments.
Setting $\sigma$ to be learnable allows the model to adapt to its ability to represent the data, such as decreasing $\sigma^2$ to increase the separation between the two classes.
Otherwise, the separation is controlled implicitly by $\mathbf{w}$ through $||\mathbf w||_2$, where the margin is proportional to $\frac{\sigma}{||\mathbf w||_2}$.
Unless otherwise stated, $\sigma$ is assumed to be a fixed hyperparameter at 10.
In equality separator NNs, we fix $\sigma=0.5$ in the output layer for stability during training.

\subsection{Linear classification for anomaly detection}
\label{appendix:linear_classification_exp}

To understand the usability of equality separation in NNs, we first
look at linear classification for AD with (1) linearly inseparable data and (2) linearly separable data.
We select key results to show here and include more details in the following subsections.

\paragraph{Linear problem 1: living on the plane, linearly inseparable}
\label{section:linear_inseparable_experiment}
We generate data on hyperplanes as the positive class while the negative class follows a uniform distribution in the ambient space.
The positive to negative class imbalance (P:N) of 0.9:0.1.
Weights
and bias of the hyperplane are drawn from a standard multivariate normal, and data are perturbed with Gaussian noise centered at zero with varying standard deviations.
We repeat experiments 20 times.
Results are reported in Table \ref{tab:linear_results_pr}.

Especially in higher dimensions, our equality separator is robust even when noise is sampled with higher variance than the weights.
The inductive bias of hyperplane learning helps our equality separator to be less sensitive to the signal-to-noise ratio.
Although simple linear regression with only positive examples may be more suitable in this case, we demonstrate that using negative examples as regularization can still achieve good performance, suggesting its usability in deep learning.

\begin{table}[]
\setlength{\tabcolsep}{4pt}
    \centering
    \caption{AUPR for linear AD of linearly inseparable data with varying standard deviation (Std) of noise and data dimension (Dim).
    No noise is denoted by 0.0 std.
    Random classifier has 0.90 AUPR.}
    \vspace{\baselineskip}
    \begin{tabular}{cccccccc}
\toprule
Std$\backslash$Dim & 5 & 10 & 15 & 20 & 25 & 30 & 35\\
\midrule
 0.0  &  0.99$\pm$0.06  &  1.00$\pm$0.00 &  0.99$\pm$0.06  &  1.00$\pm$0.00  &  1.00$\pm$0.00  &  0.99$\pm$0.06  &  1.00$\pm$0.00 \\
 0.5  &  0.98$\pm$0.06  &  1.00$\pm$0.00  &  0.99$\pm$0.06  &  1.00$\pm$0.00  &  1.00$\pm$0.00  &  0.99$\pm$0.06  &  1.00$\pm$0.00 \\
 1.0  &  0.98$\pm$0.05  &  1.00$\pm$0.01  &  0.98$\pm$0.05  &  1.00$\pm$0.00  &  0.99$\pm$0.04  &  0.99$\pm$0.06  &  1.00$\pm$0.00 \\
 1.5  &  0.97$\pm$0.05  &  1.00$\pm$0.01  &  0.98$\pm$0.05  &  0.99$\pm$0.01  &  0.99$\pm$0.01  &  0.99$\pm$0.06  &  1.00$\pm$0.01 \\
 2.0 &  0.97$\pm$0.05 &  0.98$\pm$0.02 &  0.99$\pm$0.01 &  0.99$\pm$0.01 &  1.00$\pm$0.00 &  0.98$\pm$0.05 &  1.00$\pm$0.00\\
\bottomrule
    \end{tabular}
    \label{tab:linear_results_pr}
\end{table}

\paragraph{Linear problem 2: competing with halfspace separators on linear separability}
\label{section:halfspace_problem}

\begin{table}[]
\setlength{\tabcolsep}{4pt}
    \centering
    \caption{AUPR for 2D Gaussian Separation by SVM and our equality separator (ES), with varying noise multiplier (NM). ES is competitive with SVM.}
    \begin{tabular}{cccccccc}
\toprule
NM & 1.0  &  4.0  &  7.0  &  10.0  &  13.0  &  16.0  &  19.0 \\
\midrule
SVM & 1.00$\pm$0.00  &  1.00$\pm$0.00  &  0.99$\pm$0.01  &  0.98$\pm$0.01  &  0.97$\pm$0.03  &  0.97$\pm$0.01  &  0.97$\pm$0.03\\
ES (ours) & 1.00$\pm$0.00  &  1.00$\pm$0.00  &  0.99$\pm$0.01  &  0.98$\pm$0.01  &  0.97$\pm$0.02  &  0.97$\pm$0.01  &  0.97$\pm$0.01\\
\bottomrule
    \end{tabular}
    \label{tab:aupr_gaussian_linear}
\end{table}

We compare our equality separator with robust halfspace separation, linear-kernel SVM with weighted class loss of the P:N of 0.9:0.1.
To vary the overlap of distributions,
we multiply the covariance matrix of each Gaussian by varying scalars, denoted as a \textit{noise multiplier}.
We repeat experiments 20 times
and report results in Table \ref{tab:aupr_gaussian_linear}.

Unlike the SVM trained, our equality separator does not explicitly account for class imbalance and does not incoporate the robustness of hinge loss. 
Yet, our equality separator performs as well as the class-weighted SVM and also reduces the open space risk with respect to the positive class (i.e. the class with more data available).
Once again, the bias of hyperplane learning encoded in equality separators makes it robust to class imbalance.

\subsubsection{Linear problem 1: linearly inseparable data}
\label{appendix:linearly_inseparable_exp}
We report the details of the experiments done on our experiments on linearly non-separable data.
As mentioned, we model the positive class to live on some hyperplane, while the negative class lives uniformly in the ambient space, where we vary the dimension of the ambient space $d$ across experiments to be $5, 10, 15, 20, 25, 30$ and $35$.
The normal vector $\mathbf w\in \mathbb R^d$ of the hyperplane is drawn from a standard multivariate normal 
\begin{equation*}
\mathbf w \sim \mathcal{MVN}(0, I_d)
\end{equation*}
while the bias $b\in\mathbb R$ is drawn from a standard normal
\begin{equation*}
b \sim \mathcal N(0, 1).
\end{equation*}
A datum from the positive class $\mathbf x_p$ is generated where the first $d-1$ dimensions are drawn from a uniform distribution
\begin{equation*}
x_p^i \sim \mathrm{Uniform}(-10, 10) \; , i\in \{1,...,d-1\}
\end{equation*}
and the last $d^{th}$ dimension is calculated by
\begin{equation*}
x_p^d = \frac{-b-\sum_{i=1}^{d-1}w_{i}x_p^i}{w_{d}}
\end{equation*}
where $w_i\in \mathbb R$ is the $i^{th}$ dimension of normal vector $\mathbf w\in \mathbb R^d$.
On the other hand, data from the negative class are obtained by sampling each dimension from a uniform distribution between $-10$ and $10$.
There is a 0.9:0.1 positive to negative class imbalance, and a total of $100$ data samples.

After all data are generated, random noise is added to the data.
The random noise follows a Gaussian distribution with zero mean and standard deviation of $0.0, 0.5, 1.0$ and $1.5$ (we abuse notation and refer to no noise being added as $0.0$ standard deviation) across the different experiments.

We optimize over mean-squared error (MSE)
where the hypothesis class is the set of affine predictors $A_d$ and the hyperparameters $\mu=0$ and $\sigma=10$ are fixed.
The gradient-based optimizer used is Broyden-Fletcher-Goldfarb-Shannon (BFGS), the loss function used in Section \ref{section:linear_inseparable_experiment} is MSE, and the parameters are initialized to a standard multivariate normal distribution.
We use the $\mathrm{minimize}$ function in the $\mathrm{scipy.optimize}$ package\footnote{https://docs.scipy.org/doc/scipy/reference/optimize.minimize-bfgs.html\#optimize-minimize-bfgs} where the stopping criteria is when the step size is $0$.

Each trial is repeated a total of $20$ times by generating a new dataset by sampling a different weight and bias vector (that defines a different hyperplane) and running optimization again, and we report the mean and standard deviation of the results.
As previously reported, results are in Table \ref{tab:linear_results_pr}.

\paragraph{Results}
The results of our equality separator reported 
are somewhat interesting.
On one hand, as expected, an increased level of noise leads to a decreased mean AUPR.
With more noise, the overlap between the positive and negative class increases, so the separation among the classes (and hence the mean AUPR) decreases.
On the other hand, the performance in higher dimensions generally increases, with mean AUPR increasing and AUPR variance decreasing.
This phenomenon is probably because there is a greater density of noise in lower dimensions for constant amount of data, leading to a greater overlap of the 2 classes.

\paragraph{More data for increased dimensionality}
Furthermore, a linear increase in the data does not seem to affect the performance significantly too.
We performed the same experiment, but instead of a constant $100$ data samples for all experiments, the number of data samples was adjusted be $20d$ for input dimension $d$.
Increasing the number of data samples will increase the amount of data from the negative class, and hence increase the overlap between the positive and negative class.
However, results reported in Table \ref{tab:linear_results_pr_more_data} suggest that this increase in overlap does not affect the mean and variance of AUPR of the equality separator.
The inductive bias of linear regression helps in hyperplane learning, even when the overlap between the positive and negative class increases.

\begin{table}[]
\setlength{\tabcolsep}{4pt}
    \centering
    \caption{AUPR for linear AD of linearly inseparable data with varying standard deviation (Std) of noise and data dimension (Dim), optimized with mean-squared error.
    No noise is denoted by 0.0 std.
    Random classifier has 0.90 AUPR.
    Number of data samples is $20$ times the number of dimensions.\\}
    \vspace{\baselineskip}
    \begin{tabular}{cccccccc}
\toprule
Std$\backslash$Dim & 5 & 10 & 15 & 20 & 25 & 30 & 35\\
\midrule
 0.0  &  0.99$\pm$0.06  &  1.00$\pm$0.00  &  1.00$\pm$0.00  &  1.00$\pm$0.00  &  1.00$\pm$0.00  &  0.99$\pm$0.05  &  1.00$\pm$0.00 \\
 0.5  &  0.98$\pm$0.06  &  1.00$\pm$0.00  &  0.98$\pm$0.05  &  1.00$\pm$0.00  &  1.00$\pm$0.00  &  0.98$\pm$0.05  &  1.00$\pm$0.00 \\
 1.0  &  0.98$\pm$0.05  &  0.99$\pm$0.01  &  0.99$\pm$0.00  &  0.99$\pm$0.01  &  0.99$\pm$0.00  &  0.98$\pm$0.05  &  0.99$\pm$0.01 \\
 1.5  &  0.97$\pm$0.05  &  0.99$\pm$0.00  &  0.98$\pm$0.05  &  0.99$\pm$0.01  &  0.99$\pm$0.01  &  0.97$\pm$0.05  &  0.99$\pm$0.01 \\
 2.0 &  0.97$\pm$0.05 &  0.99$\pm$0.01 &  0.98$\pm$0.01 &  0.98$\pm$0.01 &  0.98$\pm$0.01 &  0.97$\pm$0.05 &  0.98$\pm$0.01\\
\bottomrule
    \end{tabular}
    \label{tab:linear_results_pr_more_data}
\vspace{-1mm}
\end{table}

As seen in both Table \ref{tab:linear_results_pr}
and Table \ref{tab:linear_results_pr_more_data}, there are instances where results are reported with high variance (particularly in 5D and 30D data).
These instances of high variability of results suggest that the equality separator could be sensitive to initialization.
Such variability would be decreased in a regular least-squares linear regression because gradient descent can be run on the convex mean-squared error loss and without the need for including regularization of the negative examples.
However, the regularization in the equality separator to have negative examples further away from the hyperplane means that the equality separator is able to use negative examples and is useful in deep methods for supervised anomaly detection.

\paragraph{Logistic loss}
To compare our results with MSE, we also use LL and report results in Table \ref{tab:linear_results_pr_bce} for constant 100 data samples across experiments and Table \ref{tab:linear_results_pr_more_data_bce} for the same linear increase in amount of data as explored in the MSE case.
For constant data, we also observe that LL produces performance that improves (higher mean and smaller variance in AUPR) as the number of dimensions increases.
On the other hand, the same linear increase in data as the MSE case does not increase the mean AUPR for equality separators optimized with LL, but only decreases the variance of the AUPR.
This simple comparison hints at the different usages of different loss functions.

As suggested in Appendix \ref{appendix:loss_fn}, LL is more sensitive to wrong classifications, so it performs worse as the overlap between the positive and negative class increases (ie. as the noise increases).
Since noise is sparser in higher dimension for constant data settings, there is less overlap and LL performs better as the dimensions increase
(a sample visualization is shown in Figure \ref{fig:dim35log_and_mseloss}).
Since the hyperplane is overdetermined with the data, finding the hyperplane is not much of a problem.
However, when the amount of data increases, the amount of noise also increases.
Since LL is more sensitive than MSE to the overlap between the positive and negative classes, the higher density of noise causes LL to degrade in performance.

\begin{table}[]
\setlength{\tabcolsep}{4pt}
    \centering
    \caption{Logistic loss optimization with constant amount of data for linear AD of linearly inseparable data with varying standard deviation (Std) of noise and data dimension (Dim).
    No noise is denoted by 0.0 std.
    Constant $100$ data samples.
    AUPR is reported.\\}
    \vspace{\baselineskip}
    \begin{tabular}{cccccccc}
\toprule
Dim$\backslash$Std &5 &10 &15 &20 &25 &30 &35 \\
\midrule
 0.0 & 0.95$\pm$0.06 & 0.98$\pm$0.01 & 0.98$\pm$0.01 & 0.99$\pm$0.0 & 0.99$\pm$0.01 & 0.98$\pm$0.05 & 1.00$\pm$0.00 \\
 0.5 & 0.96$\pm$0.05 & 0.98$\pm$0.01 & 0.98$\pm$0.01 & 0.99$\pm$0.00 & 0.99$\pm$0.01 & 0.98$\pm$0.05 & 1.00$\pm$0.00 \\
 1.0 & 0.95$\pm$0.05 & 0.98$\pm$0.02 & 0.98$\pm$0.01 & 0.99$\pm$0.0 & 0.99$\pm$0.01 & 0.98$\pm$0.05 & 1.00$\pm$0.00 \\
 1.5 & 0.94$\pm$0.05 & 0.97$\pm$0.02 & 0.97$\pm$0.01 & 0.98$\pm$0.01 & 0.99$\pm$0.01 & 0.98$\pm$0.05 & 1.00$\pm$0.00 \\
 2.0& 0.93$\pm$0.05& 0.96$\pm$0.03& 0.97$\pm$0.01& 0.98$\pm$0.01& 0.98$\pm$0.01& 0.97$\pm$0.05& 0.99$\pm$0.01\\
\bottomrule
    \end{tabular}
    \label{tab:linear_results_pr_bce}
\end{table}

\begin{table}[]
\setlength{\tabcolsep}{4pt}
    \centering
    \caption{Logistic loss optimization with number of samples $20$ times the number of dimensions, done for linear AD of linearly inseparable data with varying standard deviation (Std) of noise and data dimension (Dim).
    No noise is denoted by 0.0 std.
    Number of data samples is $20$ times the number of dimensions.
    AUPR is reported.\\}
    \vspace{\baselineskip}
    \begin{tabular}{cccccccc}
\toprule
Std$\backslash$Dim & 5 & 10 & 15 & 20 & 25 & 30 & 35\\
\midrule
 0.0  &  0.95$\pm$0.06  &  0.96$\pm$0.01  &  0.96$\pm$0.01  &  0.96$\pm$0.01  &  0.96$\pm$0.01  &  0.96$\pm$0.01  &  0.96$\pm$0.01 \\
 0.5  &  0.96$\pm$0.05  &  0.96$\pm$0.02  &  0.95$\pm$0.01  &  0.96$\pm$0.01  &  0.96$\pm$0.01  &  0.95$\pm$0.01  &  0.96$\pm$0.01 \\
 1.0  &  0.95$\pm$0.05  &  0.96$\pm$0.02  &  0.95$\pm$0.01  &  0.96$\pm$0.02  &  0.96$\pm$0.01  &  0.96$\pm$0.01  &  0.95$\pm$0.01 \\
 1.5  &  0.94$\pm$0.05  &  0.95$\pm$0.02  &  0.95$\pm$0.02  &  0.95$\pm$0.01  &  0.95$\pm$0.01  &  0.95$\pm$0.01  &  0.95$\pm$0.02 \\
 2.0 &  0.93$\pm$0.05 &  0.95$\pm$0.02 &  0.94$\pm$0.02 &  0.95$\pm$0.01 &  0.95$\pm$0.02 &  0.94$\pm$0.02 &  0.94$\pm$0.01\\
\bottomrule
    \end{tabular}
    \label{tab:linear_results_pr_more_data_bce}
\end{table}

\begin{figure}
    \centering
    \begin{subfigure}[t]{0.49\linewidth}
    \centering
    \includegraphics[width=\textwidth]{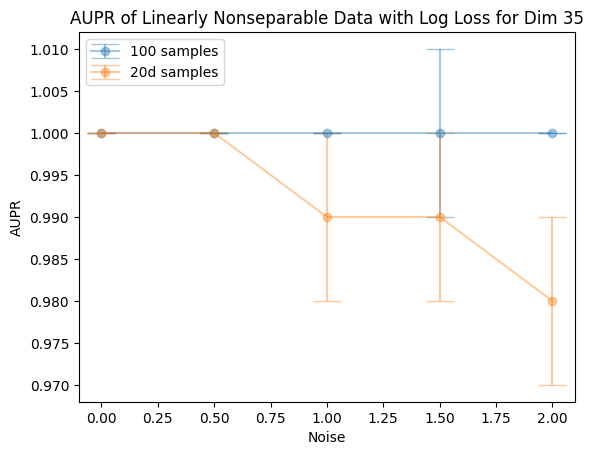}
    \caption{Logistic Loss}
    \label{fig:dim35logloss}        
    \end{subfigure}
    \begin{subfigure}[t]{0.49\linewidth}
    \centering
    \includegraphics[width=\textwidth]{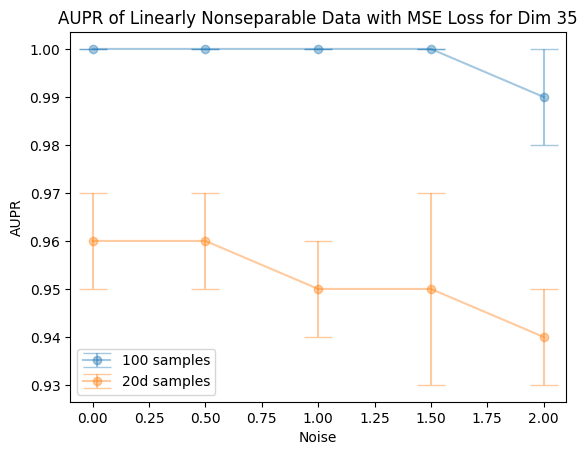}
    \caption{MSE Loss}
    \label{fig:dim35mseloss}        
    \end{subfigure}
    \caption{AUPR of equality separators on 35D linearly nonseparable data across varying noise levels, optimized either with logistic or MSE loss. Each plot shows the AUPR for the 2 cases of keeping the sample size constant or linearly growing with the dimension.
    Logistic loss is the most affected by overlap between the positive and negative class, which increases with noise and decreases with dimension (note the different $y$-axis scales).}
    \label{fig:dim35log_and_mseloss}     
\end{figure}

\subsubsection{Linearly separable data}
\label{appendix:linearly_separable_data}
We report the details of the experiments done for linearly separable data.
As mentioned, the each of the two classes are sampled from a 2D multivariate Gaussians.
The positive class follows
\begin{equation*}
\mathbf x_{positive} \sim \mathcal{MVN}\left(
\begin{bmatrix}
1\\
1
\end{bmatrix}
,
k
\begin{bmatrix}
0.2 & 0.1\\
0.1 & 0.2
\end{bmatrix}\right
)
\end{equation*}
while the negative class follows
\begin{equation*}
\mathbf x_{negative} \sim \mathcal{MVN}\left(
\begin{bmatrix}
-1\\
-1
\end{bmatrix}
,
k
\begin{bmatrix}
0.2 & 0.1\\
0.1 & 0.2
\end{bmatrix}\right
)
\end{equation*}
where $k$ is the \textit{noise multiplier}, a scalar that determines the spread of each class (and hence, the overlap between both classes).
There is a 0.9:0.1 P:N imbalance and a total of $100$ data samples for both training and testing datasets.

\paragraph{SVM}
To implement the SVM, we use the $\mathrm{sklearn.SVM.SVC}$ support vector classifier class in the scikit-learn package.
We specify a linear kernel and balanced class weight in its argument, and the prediction probability is based on an internal 5-fold cross validation.\footnote{https://scikit-learn.org/stable/modules/generated/sklearn.svm.SVC.html}

\paragraph{Equality Separator}
We optimize over mean-squared error as in Equation \ref{eqn:bump_activation} where the hypothesis class is the set of affine predictors $A_2$ and the hyperparameters $\mu=0$ and $\sigma=10$ are fixed.
The gradient-based optimizer used is BFGS, and the parameters are initialized to a standard multivariate normal distribution.
We use the $\mathrm{minimize}$ function in the $\mathrm{scipy.optimize}$ package where the stopping criteria is when the step size is $0$.

Each trial is repeated a total of $20$ times by generating a new dataset by re-sampling from the same multivariate Gaussians and running optimization again, and we report the mean and standard deviation of the results.

\subsubsection{Study on sensitivity of $\sigma$}
\label{section:vary_sigma}

We repeat experiments done in Section \ref{section:halfspace_problem} on linear separable data, but vary the value of $\sigma$ in the bump activation (Equation \ref{eqn:bump_activation}) to $0.1, 0.5, 5$ and $10$.
We report results in Table \ref{tab:aupr_gaussian_linear_mse}.

Equality separators with $\sigma\geq 0.5$ have competitive performance with SVMs, and results are mostly within 1 standard deviation away from each other across the equality separators with $\sigma$ values of $0.5, 5$ and $10$ (Figure \ref{fig:aupr_gaussian_linear_mse}).
Equality separators with a small $\sigma$ at $0.1$ have slightly worse performance in the low noise settings, with a lower mean and higher variance for its AUPR, but quickly become competitive with SVMs a noise multiplier of 2.0 onwards.
These good results for large $\sigma$ suggest that large $\sigma$ values decrease the sensitivity to randomness, such as the initialization of the weights in the equality separator.
These empirical results corroborate with our comment in Appendix \ref{section:sigma_role} about vanishing gradients with small $\sigma$ values impeding effective gradient-based optimization.
However, larger $\sigma$ values widen the bump of the Gaussian, which increases the margin for error and may cause predictions to be less informative.

Additionally, for all values of $\sigma$, equality separation is still quite robust at high noise levels, obtaining a slightly lower AUPR variance while maintaining the same mean AUPR as SVM.
Lower variance in AUPR suggests that equality separators can potentially be less sensitive to noise than SVMs.
The decreased sensitivity could be due to the equality separator modelling the hyperplane of the normal class (positive class) rather than the separating hyperplane, capitalizing on the class with more data.
Hence, equality separation in class-imbalanced settings can produce competitive results with SVMs, especially when the overlap between the positive and negative class increases.

\begin{table}[]
    \centering
    \caption{AUPR for 2D Gaussian separation by SVM and our equality separator (ES), with varying noise multiplier (NM) optimized with mean-squared error. 
    The value of $\sigma$ in the bump activation is varied in the ES to illustrate its effect.\\}
    \begin{tabular}{cccccc}
\toprule
NM & SVM & ES ($\sigma=0.1$) & ES ($\sigma=0.5$) & ES ($\sigma=5.0$) & ES ($\sigma=10.0$)\\
\midrule
1.0 & 1.00$\pm$0.0
 & 0.96$\pm$0.08 & 1.00$\pm$0.00 & 1.00$\pm$0.00 & 1.00$\pm$0.00\\
1.5 & 1.00$\pm$0.0
 & 0.99$\pm$0.04 & 1.00$\pm$0.00 & 1.00$\pm$0.00 & 1.00$\pm$0.00\\
2.0 & 1.00$\pm$0.0
 & 1.00$\pm$0.00 & 1.00$\pm$0.00 & 1.00$\pm$0.00 & 1.00$\pm$0.00\\
2.5 & 1.00$\pm$0.0
 & 1.00$\pm$0.00 & 1.00$\pm$0.00 & 1.00$\pm$0.00 & 1.00$\pm$0.00\\
3.0 & 1.00$\pm$0.0
 & 1.00$\pm$0.00 & 1.00$\pm$0.00 & 1.00$\pm$0.00 & 1.00$\pm$0.00\\
3.5 & 1.00$\pm$0.0
 & 1.00$\pm$0.01 & 1.00$\pm$0.01 & 1.00$\pm$0.01 & 1.00$\pm$0.01\\
4.0 & 1.00$\pm$0.0
 & 1.00$\pm$0.00 & 1.00$\pm$0.00 & 1.00$\pm$0.00 & 1.00$\pm$0.00\\
4.5 & 0.99$\pm$0.0
 & 0.99$\pm$0.00 & 0.99$\pm$0.00 & 0.99$\pm$0.00 & 0.99$\pm$0.00\\
5.0 & 0.99$\pm$0.0
 & 0.99$\pm$0.00 & 0.99$\pm$0.00 & 0.99$\pm$0.00 & 0.99$\pm$0.00\\
5.5 & 0.99$\pm$0.0
 & 0.99$\pm$0.01 & 0.99$\pm$0.01 & 0.99$\pm$0.01 & 0.99$\pm$0.01\\
6.0 & 0.99$\pm$0.01
 & 0.99$\pm$0.01 & 0.99$\pm$0.01 & 0.99$\pm$0.01 & 0.99$\pm$0.01\\
6.5 & 0.99$\pm$0.01
 & 0.99$\pm$0.01 & 0.99$\pm$0.01 & 0.99$\pm$0.01 & 0.99$\pm$0.01\\
7.0 & 0.99$\pm$0.01
 & 0.99$\pm$0.01 & 0.99$\pm$0.01 & 0.99$\pm$0.01 & 0.99$\pm$0.01\\
7.5 & 0.99$\pm$0.01
 & 0.98$\pm$0.01 & 0.98$\pm$0.01 & 0.98$\pm$0.01 & 0.98$\pm$0.01\\
8.0 & 0.99$\pm$0.01
 & 0.98$\pm$0.01 & 0.98$\pm$0.01 & 0.98$\pm$0.01 & 0.98$\pm$0.01\\
8.5 & 0.99$\pm$0.01
 & 0.98$\pm$0.02 & 0.98$\pm$0.02 & 0.98$\pm$0.02 & 0.98$\pm$0.02\\
9.0 & 0.99$\pm$0.01
 & 0.98$\pm$0.02 & 0.98$\pm$0.02 & 0.98$\pm$0.02 & 0.98$\pm$0.02\\
9.5 & 0.98$\pm$0.01
 & 0.98$\pm$0.01 & 0.98$\pm$0.01 & 0.98$\pm$0.01 & 0.98$\pm$0.01\\
10.0 & 0.98$\pm$0.01
 & 0.98$\pm$0.01 & 0.98$\pm$0.01 & 0.98$\pm$0.01 & 0.98$\pm$0.01\\
10.5 & 0.98$\pm$0.01
 & 0.98$\pm$0.01 & 0.98$\pm$0.01 & 0.98$\pm$0.01 & 0.98$\pm$0.01\\
11.0 & 0.98$\pm$0.01
 & 0.98$\pm$0.01 & 0.98$\pm$0.01 & 0.98$\pm$0.01 & 0.98$\pm$0.01\\
11.5 & 0.98$\pm$0.01
 & 0.98$\pm$0.01 & 0.98$\pm$0.01 & 0.98$\pm$0.01 & 0.98$\pm$0.01\\
12.0 & 0.98$\pm$0.01
 & 0.98$\pm$0.02 & 0.98$\pm$0.02 & 0.98$\pm$0.02 & 0.98$\pm$0.02\\
12.5 & 0.98$\pm$0.01
 & 0.97$\pm$0.02 & 0.97$\pm$0.02 & 0.97$\pm$0.02 & 0.97$\pm$0.02\\
13.0 & 0.97$\pm$0.03
 & 0.97$\pm$0.02 & 0.97$\pm$0.02 & 0.97$\pm$0.02 & 0.97$\pm$0.02\\
13.5 & 0.97$\pm$0.03
 & 0.97$\pm$0.02 & 0.97$\pm$0.02 & 0.97$\pm$0.02 & 0.97$\pm$0.02\\
14.0 & 0.97$\pm$0.03
 & 0.97$\pm$0.02 & 0.97$\pm$0.02 & 0.97$\pm$0.02 & 0.97$\pm$0.02\\
14.5 & 0.98$\pm$0.01
 & 0.97$\pm$0.01 & 0.97$\pm$0.01 & 0.97$\pm$0.01 & 0.97$\pm$0.01\\
15.0 & 0.98$\pm$0.01
 & 0.97$\pm$0.01 & 0.97$\pm$0.01 & 0.97$\pm$0.01 & 0.97$\pm$0.01\\
15.5 & 0.97$\pm$0.03
 & 0.97$\pm$0.01 & 0.97$\pm$0.01 & 0.97$\pm$0.01 & 0.97$\pm$0.01\\
16.0 & 0.97$\pm$0.01
 & 0.97$\pm$0.01 & 0.97$\pm$0.01 & 0.97$\pm$0.01 & 0.97$\pm$0.01\\
16.5 & 0.97$\pm$0.01
 & 0.97$\pm$0.02 & 0.97$\pm$0.02 & 0.97$\pm$0.02 & 0.97$\pm$0.02\\
17.0 & 0.97$\pm$0.03
 & 0.97$\pm$0.01 & 0.97$\pm$0.01 & 0.97$\pm$0.01 & 0.97$\pm$0.01\\
17.5 & 0.97$\pm$0.03
 & 0.97$\pm$0.01 & 0.97$\pm$0.01 & 0.97$\pm$0.01 & 0.97$\pm$0.01\\
18.0 & 0.97$\pm$0.03
 & 0.97$\pm$0.01 & 0.97$\pm$0.01 & 0.97$\pm$0.01 & 0.97$\pm$0.01\\
18.5 & 0.97$\pm$0.03
 & 0.97$\pm$0.01 & 0.97$\pm$0.01 & 0.97$\pm$0.01 & 0.97$\pm$0.01\\
19.0 & 0.97$\pm$0.03
 & 0.97$\pm$0.01 & 0.97$\pm$0.01 & 0.97$\pm$0.01 & 0.97$\pm$0.01\\
19.5 & 0.97$\pm$0.03
 & 0.97$\pm$0.01 & 0.97$\pm$0.01 & 0.97$\pm$0.01 & 0.97$\pm$0.01\\
20.0 & 0.97$\pm$0.03
 & 0.97$\pm$0.01 & 0.97$\pm$0.01 & 0.97$\pm$0.01 & 0.97$\pm$0.01\\
\bottomrule
    \end{tabular}
    \label{tab:aupr_gaussian_linear_mse}
\end{table}

\paragraph{Logistic loss}
We perform the same experiments on linearly separable data, but change the loss function to LL.
We report results in Table \ref{tab:aupr_gaussian_linear_bce} and provide a visualization in Figure \ref{fig:aupr_gaussian_linear_bce}.

Training on LL is generally slightly less competitive than training with MSE, although not significantly less.
With an increase in overlap between the positive and negative classes, there is generally a slightly faster decrease in mean AUPR, with a higher AUPR variance.
As explained in Appendix \ref{appendix:loss_fn}, LL does not handle overlap between classes well, leading to a significant decrease in performance.
Interestingly enough, smaller $\sigma$ values seem to achieve better performance with LL.
Smaller $\sigma$ values produce stricter decision boundaries, similar to LL compared to MSE, which suggest how loss functions and $\sigma$ values may pair well together.

\begin{figure}
    \centering
    \begin{subfigure}[b]{0.45\linewidth}
    \centering
    \includegraphics[width=0.99\linewidth, keepaspectratio]{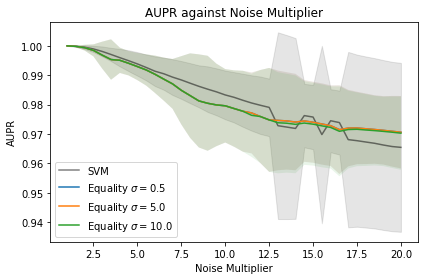}
    \caption{Mean-squared error optimization. $\sigma=0.1$ omitted for clearer visualizations.}
    \label{fig:aupr_gaussian_linear_mse}
    \end{subfigure}
    \hspace{0.05\linewidth}
    \begin{subfigure}[b]{0.45\linewidth}
    \centering
    \includegraphics[width=0.99\linewidth, keepaspectratio]{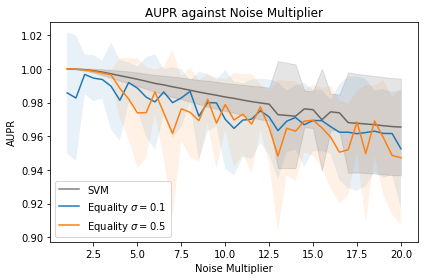}
    \caption{Logistic loss optimization. $\sigma= 5.0, 10.0$ omitted for clearer visualizations.}
    \label{fig:aupr_gaussian_linear_bce}
    \end{subfigure}
    \caption{AUPR for 2D Gaussian separation by SVM and equality separator under increasing noise and $\sigma$ in its bump activation. Mean-squared error and logistic loss optimization are shown.}
    \label{fig:my_label}
\end{figure}

\begin{table}[]
    \centering
    \caption{AUPR for 2D Gaussian separation by SVM and our equality separator (ES), with varying noise multiplier (NM) optimized with logistic loss. 
    The value of $\sigma$ in the bump activation is varied in the ES to illustrate its effect.\\}
    \begin{tabular}{cccccc}
\toprule
NM & SVM & ES ($\sigma=0.1$) & ES ($\sigma=0.5$) & ES ($\sigma=5.0$) & ES ($\sigma=10.0$)\\
\midrule
1.0 & 1.00$\pm$0.00
 & 0.99$\pm$0.04 & 1.00$\pm$0.00 & 1.00$\pm$0.00 & 1.00$\pm$0.00\\
1.5 & 1.00$\pm$0.00
 & 0.98$\pm$0.04 & 1.00$\pm$0.00& 1.00$\pm$0.00& 0.99$\pm$0.03\\
2.0 & 1.00$\pm$0.00
 & 1.00$\pm$0.01 & 1.00$\pm$0.00& 1.00$\pm$0.00& 0.99$\pm$0.02\\
2.5 & 1.00$\pm$0.00
 & 0.99$\pm$0.01 & 1.00$\pm$0.00& 0.99$\pm$0.02 & 1.00$\pm$0.01\\
3.0 & 1.00$\pm$0.00
 & 0.99$\pm$0.01 & 1.00$\pm$0.00& 0.99$\pm$0.02 & 0.99$\pm$0.02\\
3.5 & 1.00$\pm$0.00
 & 0.99$\pm$0.03 & 1.00$\pm$0.00& 0.98$\pm$0.02 & 0.98$\pm$0.02\\
4.0 & 1.00$\pm$0.00
 & 0.98$\pm$0.02 & 0.99$\pm$0.02 & 0.98$\pm$0.02 & 0.97$\pm$0.03\\
4.5 & 0.99$\pm$0.00
 & 0.99$\pm$0.01 & 0.98$\pm$0.02 & 0.98$\pm$0.02 & 0.97$\pm$0.04\\
5.0 & 0.99$\pm$0.00
 & 0.99$\pm$0.01 & 0.97$\pm$0.03 & 0.97$\pm$0.03 & 0.97$\pm$0.03\\
5.5 & 0.99$\pm$0.00
 & 0.98$\pm$0.02 & 0.97$\pm$0.03 & 0.97$\pm$0.03 & 0.97$\pm$0.03\\
6.0 & 0.99$\pm$0.01
 & 0.98$\pm$0.03 & 0.99$\pm$0.01 & 0.96$\pm$0.03 & 0.96$\pm$0.03\\
6.5 & 0.99$\pm$0.01
 & 0.99$\pm$0.01 & 0.97$\pm$0.03 & 0.96$\pm$0.03 & 0.95$\pm$0.04\\
7.0 & 0.99$\pm$0.01
 & 0.98$\pm$0.02 & 0.96$\pm$0.05 & 0.97$\pm$0.03 & 0.95$\pm$0.03\\
7.5 & 0.99$\pm$0.01
 & 0.98$\pm$0.02 & 0.98$\pm$0.02 & 0.97$\pm$0.02 & 0.95$\pm$0.03\\
8.0 & 0.99$\pm$0.01
 & 0.99$\pm$0.01 & 0.97$\pm$0.03 & 0.97$\pm$0.02 & 0.95$\pm$0.03\\
8.5 & 0.99$\pm$0.01
 & 0.97$\pm$0.03 & 0.97$\pm$0.02 & 0.96$\pm$0.03 & 0.94$\pm$0.04\\
9.0 & 0.99$\pm$0.01
 & 0.98$\pm$0.02 & 0.98$\pm$0.01 & 0.97$\pm$0.03 & 0.95$\pm$0.04\\
9.5 & 0.98$\pm$0.01
 & 0.98$\pm$0.01 & 0.97$\pm$0.03 & 0.96$\pm$0.04 & 0.95$\pm$0.03\\
10.0 & 0.98$\pm$0.01
 & 0.97$\pm$0.03 & 0.98$\pm$0.01 & 0.96$\pm$0.03 & 0.95$\pm$0.03\\
10.5 & 0.98$\pm$0.01
 & 0.96$\pm$0.03 & 0.97$\pm$0.03 & 0.95$\pm$0.04 & 0.95$\pm$0.03\\
11.0 & 0.98$\pm$0.01
 & 0.97$\pm$0.02 & 0.97$\pm$0.02 & 0.95$\pm$0.03 & 0.95$\pm$0.03\\
11.5 & 0.98$\pm$0.01
 & 0.97$\pm$0.02 & 0.97$\pm$0.03 & 0.95$\pm$0.04 & 0.95$\pm$0.03\\
12.0 & 0.98$\pm$0.01
 & 0.98$\pm$0.02 & 0.98$\pm$0.01 & 0.95$\pm$0.03 & 0.94$\pm$0.03\\
12.5 & 0.98$\pm$0.01
 & 0.97$\pm$0.02 & 0.96$\pm$0.02 & 0.95$\pm$0.03 & 0.94$\pm$0.03\\
13.0 & 0.97$\pm$0.03
 & 0.96$\pm$0.03 & 0.95$\pm$0.05 & 0.95$\pm$0.03 & 0.95$\pm$0.03\\
13.5 & 0.97$\pm$0.03
 & 0.97$\pm$0.02 & 0.96$\pm$0.03 & 0.95$\pm$0.03 & 0.95$\pm$0.03\\
14.0 & 0.97$\pm$0.03
 & 0.97$\pm$0.02 & 0.96$\pm$0.03 & 0.95$\pm$0.03 & 0.95$\pm$0.03\\
14.5 & 0.98$\pm$0.01
 & 0.97$\pm$0.02 & 0.97$\pm$0.02 & 0.95$\pm$0.03 & 0.95$\pm$0.03\\
15.0 & 0.98$\pm$0.01
 & 0.97$\pm$0.02 & 0.97$\pm$0.02 & 0.95$\pm$0.03 & 0.95$\pm$0.03\\
15.5 & 0.97$\pm$0.03
 & 0.97$\pm$0.02 & 0.97$\pm$0.02 & 0.95$\pm$0.03 & 0.94$\pm$0.03\\
16.0 & 0.97$\pm$0.01
 & 0.97$\pm$0.02 & 0.96$\pm$0.03 & 0.94$\pm$0.04 & 0.94$\pm$0.03\\
16.5 & 0.97$\pm$0.01
 & 0.96$\pm$0.03 & 0.95$\pm$0.03 & 0.94$\pm$0.03 & 0.95$\pm$0.03\\
17.0 & 0.97$\pm$0.03
 & 0.96$\pm$0.03 & 0.95$\pm$0.03 & 0.95$\pm$0.03 & 0.95$\pm$0.03\\
17.5 & 0.97$\pm$0.03
 & 0.96$\pm$0.03 & 0.97$\pm$0.02 & 0.95$\pm$0.03 & 0.95$\pm$0.03\\
18.0 & 0.97$\pm$0.03
 & 0.96$\pm$0.03 & 0.95$\pm$0.04 & 0.95$\pm$0.03 & 0.94$\pm$0.03\\
18.5 & 0.97$\pm$0.03
 & 0.96$\pm$0.02 & 0.97$\pm$0.02 & 0.95$\pm$0.03 & 0.94$\pm$0.03\\
19.0 & 0.97$\pm$0.03
 & 0.96$\pm$0.03 & 0.96$\pm$0.03 & 0.95$\pm$0.03 & 0.95$\pm$0.03\\
19.5 & 0.97$\pm$0.03
 & 0.96$\pm$0.02 & 0.95$\pm$0.04 & 0.95$\pm$0.03 & 0.94$\pm$0.03\\
20.0 & 0.97$\pm$0.03
 & 0.95$\pm$0.04 & 0.95$\pm$0.04 & 0.95$\pm$0.03 & 0.94$\pm$0.03\\
\bottomrule
    \end{tabular}
    \label{tab:aupr_gaussian_linear_bce}
\end{table}

\paragraph{Varying dimensionality of linearly separable data}
To understand the effect of increasing dimensionality, we also vary the dimensionality $d$ of the multivariate Gaussian.
For simplicity, we generate data from zero mean isotropic Gaussians with covariance $0.1k \cdot I_d$ for noise multiplier $k$.
We test $d=5, 10, 15, 20, 25, 30$.
All other details remain the same.
As informed by our experiments on varying $\sigma$ for different loss functions, we test on MSE optimization with $\sigma=10$ and LL optimization with $\sigma=0.5$.

In general, equality separators are not significantly worse than SVMs, maintaining close to perfect AUPR above 0.99 mean AUPR.
As we have observed previously, MSE optimization produces better results (higher mean AUPR with lower variance) in lower dimensions while LL optimization produces better results in higher dimensions.
Figure \ref{fig:30dGaussians} is a sample visualization of the comparison between the performance of equality separators and SVMs at $d=30$.

\begin{figure}[]
    \centering
    \begin{subfigure}{0.45\linewidth}
    \includegraphics[width=1\linewidth]{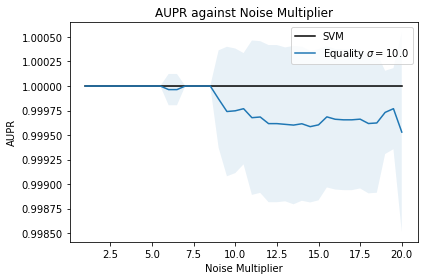}
    \caption{Mean-squared error optimization with $\sigma=10$.}
    \label{fig:30dG_mse}
    \end{subfigure}
    \hspace{0.05\linewidth}
    \begin{subfigure}{0.45\linewidth}
    \includegraphics[width=1\linewidth]{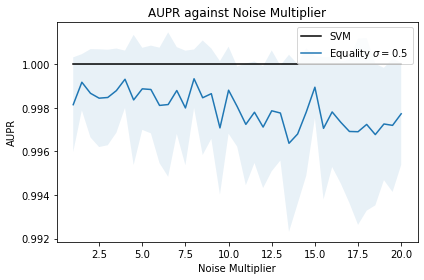}
    \caption{Logistic loss optimization with $\sigma=0.5$.}
    \label{fig:30dG_ll}
    \end{subfigure}
    \caption{AUPR of equality separators and SVMs on two linearly separable 30D multivariate Gaussians.}
    \label{fig:30dGaussians}
\end{figure}

\subsection{Non-linear anomaly detection: synthetic data}
\label{appendix:non_linear_ad}
We report additional details of the experiments done in Section \ref{section:nonlinear_ad}.
Experiments are repeated 20 times by re-initializing models and training them on the dataset.
By using the same dataset but re-initializing the models, we can observe each model's sensitivity to initializations and optimization methods.
We perform experiments with a relatively low positive to negative class imbalance of 0.75:0.25 during train and test time to more prominently compare the difference between models.

As a note, we refer to the normal class as the positive class in the text because it corresponds to a neuronal activation.
However, when calculating results, we follow the literature and refer to anomalies as the positive class.
Since we focus on comparing between the \textit{separation} between normal data and anomalies across methods, these details do not affect the analysis.

\paragraph{Shallow models}
Shallow models are trained using the $\mathrm{sklearn}$ library and the boosted trees from the $\mathrm{XGBoost}$ library\footnote{https://xgboost.readthedocs.io/en/stable/python/python\_api.html}.
Hyperparameters are picked using 10-fold stratified cross validation under logistic loss on the validation dataset.
We list the hyperparameters that were chosen.
If a hyperpameter is not listed, we used the default ones from the respective libraries.
\begin{itemize}
\item Decision trees: Split with random split and half of the features, maximum depth of the tree at 8, no maximum leaf nodes, minimum number of samples required to split an internal node and to be at a leaf node at 2 and 1 respectively, and no class weight.
\item Random forests: Split with half of the features, 50 trees, no maximum depth of a tree, no maximum leaf nodes, minimum number of samples required to split an internal node and to be at a leaf node at 2 and 1 respectively, and no class weight.
\item XGBoost (extreme gradient boosted trees): 50 trees, 0 logistic loss reduction required for further splits, $\alpha=0.1$ for $\ell_1$ regularization, $\lambda=0.01$ for $\ell_2$ regularization, balanced class weight and subsampling ratio of 0.6.
\item Logistic regression with RBF kernel: $C=1/\lambda=0.01$ regularization parameter with $\ell_2$ penalty and no class weight.
\item SVM with RBF kernel: $C=1/\lambda=0.1$ regularization parameter with $\ell_2$ penalty and no class weight.
\item OCSVM with RBF kernel: The upper bound set on the fraction of training errors is 0.0001, the kernel coefficient is the inverse of the product of the number of features (2) and the variance of the data.
\item Isolation forest: 50 trees that use all samples and all features and an automatic detection of the contamination ratio based in the original paper \citep{IsolationForest}.
\item Local outlier factor: 5 neighbours used with a leaf size of 30 and an automatic detection of the contamination ratio based in the original paper \citep{LOF}.
\end{itemize}

For unsupervised AD, there are two possible ways to use the labels:
we can either feed all the data into the algorithm (and include information of the ratio of anomalies to normal data when possible, such as for OCSVM), or we can just feed in normal data into the algorithm and remove the anomalies (because unsupervised methods assume that most/all data fed in is normal).
We noticed that the first method gives us better results, and we report these.

For equality separators, we report the results of varying the loss function (LL or MSE) and the value of $\sigma$ ($0.1, 0.5, 5, 10$) in Table \ref{tab:nonlinear_ad_ES}.
Equality separators optimized with LL performed better with lower values of $\sigma$, while those optimized with MSE performed better with greater values of $\sigma$.
When $\sigma=0.1$, the equality separators had a tendency of underfitting, scoring an lower average AUPR on train data than that of test data.
The underfitting phenomenon at this low $\sigma$ value corroborates with our observation of the sensitivity of initialization with smaller $\sigma$ values to prevent vanishing gradients. 

\begin{table}[]
    \centering
    \caption{Test AUPR for equality separator with RBF kernel for non-linear anomaly detection. We vary the loss function (mean-squared error or logistic loss) and the value of $\sigma$ ($0.1, 0.5, 5, 10$)\\}
    \begin{tabular}{ccccc}
    \toprule
    Loss$\backslash\sigma$ & 0.1 & 0.5 & 5 & 10  \\
    \midrule
    MSE     & 0.98$\pm$0.04 &  1.00$\pm$0.01 &  \textbf{1.00$\pm$0.00}  & \textbf{1.00$\pm$0.00} \\
    LL      & 0.99$\pm$0.03 & \textbf{1.00$\pm$0.00} &  \textbf{1.00$\pm$0.00}  &  1.00$\pm$0.01 \\
    \bottomrule
    \end{tabular}
    \label{tab:nonlinear_ad_ES}
\end{table}
 
\paragraph{Neural networks}
NNs used are fully connected feed-forward NNs.
Below are the details of the NNs:
\begin{enumerate}
    \item NNs are trained using TensorFlow.\footnote{https://www.tensorflow.org/} All parameters are assumed to be default unless stated below.
    \item We use 5 neurons per hidden layer, limiting the chance of overfitting in this low data regime. Not decreasing the number of neurons is also in line with our analysis of closing numbers.
    \item NNs with leaky ReLU activations in their hidden layers have the leaky ReLU with parameter $0.01$. Leaky ReLU is chosen over the conventional ReLU because mappings that are more local (such as the bump activation) are more sensitive to initializations \citep{nonmonotonic_activations_mlp_thesis}, and backpropagation with leaky ReLU will not suffer from dead neurons like ReLU would.
    Since we are creating relatively small NNs, the impact of dead neurons may potentially be large.
    The parameter is chosen to be relatively small at 0.01 to simulate ReLU more closely.
    \item NNs that have bump activations in their hidden layers
    are initialized with $\sigma=0.5$, which is not trainable by default. 
    \item NNs with RBF activations in their hidden layers or output layer have a fixed weight of 1.0 with learnable centers that are initialized with Glorot uniform initializer.
    Fixed weights with learnable centers make a fair comparison so that NNs will have the same number of parameters.%
    \footnote{Note that each neuron from a regular dense layer will have a total $d+1$ parameters in the affine scaling $\mathbf w^T \mathbf{z} +b$ for $\mathbf{w}, \mathbf{z} \in \mathbb R^d$, but the distance from the hyperplane is determined by the degrees of freedom, which will still be $d$. 
    Meanwhile, the center parameter in an RBF neuron is in the same space as the input $\mathbf z$, so it is $d$-dimensional.
    Hence, we view the number of parameters in a single RBF layer and a single fully connected layer to be the same.}
    \item Equality separators have an output of a bump activation with $\sigma=0.5$ that is fixed at $\sigma=0.5$ by default and have a default loss function of binary cross entropy (ie. logistic loss in this binary classification problem) that can also be switched out for mean-squared error, while halfspace separators have an sigmoid output activation and have a loss function of binary cross entropy.
    \item Weights are initialized with the default Glorot uniform initializer.
    \item NNs are trained with the Adam optimizer under a constant learning rate of $0.001$.
    \item Training is done for 1000 epochs with early stopping. Early stopping monitors validation loss with a patience of 10 and best weights are restored. Validation split is 0.1.
\end{enumerate}

\paragraph{Results}
In addition to the results presented in the main paper, we include results for equality separators optimized with MSE and equality separators with bump activations in their hidden layers with learnable $\sigma$.
AUPRs of the test dataset are reported in Table \ref{tab:supervised_ad_full_2layers}.
From our results, equality separators with RBF activations perform the best with a perfect AUPR during train and test times, regardless whether they were optimized under LL or MSE.

\paragraph{Separation between normal and anomalous data}
Furthermore, as previously discussed, we observe a larger separation between the positive class and the negative test class for these equality separators with RBF activations optimized under LL.
Looking at individual models, about half of these equality separators achieved high separation between the normal and anomaly class, where the difference between the mean output of normal data and anomalies during testing is at least 0.30 (Figure \ref{fig:AD_heatmap_ES2r}).
In the penultimate layer (``logit space''), this separation corresponds to normal data being a distance within about 0.05 away from the hyperplane, while anomalies are a distance of about 0.42 to 0.60 away.
Even though the output of anomalies is still relatively high (above 0.5), from the equality separation perspective, these anomalies are more than 8 times further than normal data, so setting reliable thresholds is more straightforward.

The difference between the means of the positive class and negative class during test time was $0.27\pm0.11$ for equality separators with RBF activations optimized under LL, compared to $0.16\pm0.11$ for those optimized under MSE.
In the penultimate layer of the NNs, the average difference corresponds to a distance of the negative class from the hyperplane of 0.40 for LL optimization and 0.30 for MSE optimization (an output prediction of $0.5$ corresponds to a distance of 0.59 away from the hyperplane).
Since this experiment had no overlap between the positive and negative class, LL optimization was more appropriate, obtaining perfect AUPR like MSE optimization but a greater separation between the positive and negative class.
A greater separation suggests one-class decision boundaries that are better.

The next best models achieving close to perfect AUPR are the equality separators with bump activations and RBF separators with leaky ReLU or RBF activations.
NNs with bump activations in their hidden layers consistently achieved mean AUPRs of above 0.90, suggesting the versality of the bump activation as a semi-local activation function.

On a whole, it seems that a mix of globality and locality in activation functions helps to increase one-class classification performance.
While equality separation with the local RBF activation was the most effective, equality separators with bump activations (purely semi-local) and RBF separators with leaky ReLU (local paired with global) also have competitive AUPR.
When using leaky ReLU in the hidden layers, increasing the locality of the output activation increases performance, while the less global bump and RBF activations also increase performance of halfspace separators.
Meanwhile, NNs with bump activations in their hidden layers consistently achieved over 0.90 mean AUPR.
Our experiments corroborate with our intuition in Section \ref{section:connection_to_AD} that more local mappings (such as the equality separator and its smooth version, the bump activation) can implicitly encode a one-class classification objective.
The semi-local property of equality separators allows it to retain locality for AD, while improving the separation between the normal and anomaly class of NNs with only RBF activations.

\begin{table}[]
    \caption{Test time AUPR for 
    supervised AD using neural networks with 2 hidden layers.
    For NNs, we modify the output layer (halfspace separation (HS), equality separation (ES) and RBF separation (RS)) in the rows and activation functions in the columns.
    For equality separators, we include information on the loss optimized on (logistic loss LL or mean-squared error MSE).
    If $\sigma$ is learnable in the bump activations in the hidden layers, we include an added ``s'' suffix. (This table includes results from updated experiments done in the main paper.)
    \\}
    \centering
    \begin{tabular}{lccc}
\toprule
NN$\backslash$Activation & Leaky ReLU & Bump & RBF\\
\midrule
HS  &  0.62$\pm$0.12 &  0.92$\pm$0.13 & 0.83$\pm$0.18\\
RS  &  0.97$\pm$0.05 &  0.91$\pm$0.14 & 0.97$\pm$0.06\\
ES (LL)  &  0.83$\pm$0.15 &  {0.99$\pm$0.04} & \textbf{1.00$\pm$0.00}\\
ES (MSE)  &  0.77$\pm$0.18 &  {0.98$\pm$0.09} & \textbf{1.00$\pm$0.00}\\
ES-s (LL)  &  - &  {0.98$\pm$0.05} & -\\
ES-s (MSE)  &  - &  {0.95$\pm$0.11} & -\\
\bottomrule
    \end{tabular}
    \label{tab:supervised_ad_full_2layers}
\end{table}

\paragraph{Ablation: 3 hidden layers}
We also investigate NNs with 3 hidden layers, which is 1 more hidden layer than our original experiments.
We report test time AUPRs in Table \ref{tab:supervised_ad_full_3layers}.
In general, NNs with 3 hidden layers achieved worse performance than their counterparts with 2 hidden layers.
Coupled with train AUPRs of above 0.97, we can see that NNs with 3 hidden layers generally overfit to the negative class during training.
Equality separator NNs with RBF activations with 3 hidden layers optimized under LL maintained a perfect AUPR and achieved a separation between the positive and negative class during test time of $0.22\pm0.13$, which has a mean that is only slightly lower than their counterpart with 2 hidden layers.
Those optimized under MSE also maintained a close to perfect AUPR with a separation of $0.16\pm0.14$, which is similar to their counterparts with 2 hidden layers.

In our experiments, NNs with bump activations in their hidden layers were consistently resistant to severe overfitting while maintaining a relatively good test performance of mean AUPR at least 0.89, with a mean test AUPR always within 1 standard deviation of their counterparts with 2 hidden layers.
On the other hand, NNs with RBF activations in their hidden layers maintained a high test AUPR except for halfspace separators.
Meanwhile, NNs with leaky ReLU generally maintained performance within 1 standard deviation.
In general, our results suggest that bump activations are more robust to overfitting while being able to achieve good performance for anomaly detection.
Future work to more thoroughly investigate this overfitting phenomenon can be done to understand how to further mitigate overfitting, especially without prior exhaustive knowledge of all possible negative examples.

\begin{table}[]
    \caption{Test time AUPR for 
    supervised AD using neural networks with 3 hidden layers.
    For NNs, we modify the output layer (halfspace separation (HS), equality separation (ES) and RBF separation (RS)) in the rows and activation functions in the columns.
    For equality separators, we include information on the loss optimized on (logistic loss LL or mean-squared error MSE).
    If $\sigma$ is learnable in the bump activations in the hidden layers, we include an added ``s'' suffix.
    \\}
    \centering
    \begin{tabular}{lccc}
\toprule
NN$\backslash$Activation & Leaky ReLU & Bump & RBF\\
\midrule
HS  &  0.69$\pm$0.16 &  0.92$\pm$0.14 & 0.63$\pm$0.17\\
RS  &  0.96$\pm$0.06 &  0.89$\pm$0.14 & 0.94$\pm$0.13\\
ES (LL)  &  0.83$\pm$0.15 &  {0.95$\pm$0.10} & \textbf{1.00$\pm$0.00}\\
ES (MSE)  &  0.84$\pm$0.16 &  {0.90$\pm$0.13} & 0.99$\pm$0.03 \\
ES-s (LL)  &  - &  {0.98$\pm$0.04} & -\\
ES-s (MSE)  &  - &  {0.90$\pm$0.16} & -\\
\bottomrule
    \end{tabular}
    \label{tab:supervised_ad_full_3layers}
\end{table}

\paragraph{Learnable $\sigma$ in bump activations}
Having a learnable $\sigma$ parameter in the hidden layers of the bump activation generally do not affect AUPR significantly for equality separator NNs optimized under LL or MSE.
Moreover, a pattern emerges for both 2- and 3-hidden layered NNs -- for individual NNs that performed well, the $\sigma$ value in the first hidden layer tends to increase while the $\sigma$ value in the latter hidden layer tends to decrease, with the final value of $\sigma$ usually being higher at the first hidden layer than the latter layers after convergence (Figure \ref{fig:sigma_behaviour}).
A larger $\sigma$ in the earlier layer signifies allowing more information through the layer because there is a larger margin, while a lower $\sigma$ in the latter layers suggests a greater selectivity in inputs for neuron activation.
Intuitively, the first hidden layer increases the dimensionality of the input for feature extraction (from 2D inputs to 5D hidden layer), while the latter hidden layers perform variable selection based on the increased dimensionality.

Such behavior of $\sigma$ is insightful to how gradient-based optimizers learn equality separation (even in hidden layers and with fixed $\sigma$).
Since backpropagation through the latter hidden layers decreases the value of $\sigma$ and decreases the receptive field of the bump activation, we can infer that backpropagation is also likely to decrease the receptive field of latter layers with fixed $\sigma$ in their bump activation.
This behavior corroborates with the observation 
in \cite{gaussian_activation_ood_detection} that bump activations lead to more closed decision regions.
Such behavior is especially useful in applications like anomaly detection, where equality separation and bump activations seems to be inclined towards encoding the inductive bias of one-class classification.

\begin{figure}
    \centering
    \begin{subfigure}[b]{0.45\linewidth}
    \centering
    \includegraphics[width=0.9\linewidth]{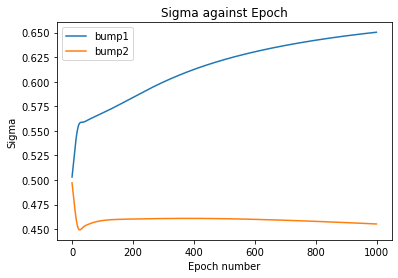}
    \caption{Neural network with 2 hidden layers.}
    \label{fig:sigma_behaviour_2layers}
    \end{subfigure}
    \hspace{0.05\linewidth}
    \begin{subfigure}[b]{0.45\linewidth}
    \centering
    \includegraphics[width=0.9\linewidth]{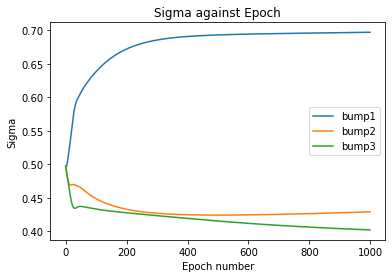}
    \caption{Neural network with 3 hidden layers.}
    \label{fig:sigma_behaviour_3layers}
    \end{subfigure}
    \caption{Example behaviour of $\sigma$ in the hidden layers of an equality separator NN with bump activations and learnable $\sigma$ in its hidden layers.
    The numbered suffix of the label of the curves refers to the index of the hidden layer (starting from 1).}
    \label{fig:sigma_behaviour}
\end{figure}

\paragraph{Model averaging for increased reliability}

NNs with bump activations generally achieved high mean AUPR with some variance.
Particularly, equality separators with bump activations in their hidden layers achieved the highest mean AUPR (close to perfect), for both 2- and 3-hidden layered NNs.
High mean AUPR with a high variance suggest that individual NNs are able to achieve good separation of the positive and negative classes during test time.
Therefore, NNs may be slightly sensitive to initializations.
To mitigate potentially bad initializations, we can view these NNs as weak learners for anomaly detection.
Then, an ensemble model should hopefully produce a better model for anomaly detection.

The simplest approach of averaging the predictions of the 20 models we trained achieves a perfect test AUPR for both the 2-hidden layered NN (Figure \ref{fig:SAD_Heat_ES2fb_ensemble}) and 3-hidden layered NN (Figure \ref{fig:SAD_Heat_ES3fb_ensemble}).
By averaging the prediction of the numerous trained models, the bad initializations were cancelled out and the ensemble model performs more reliably.

Other methods can be applied to ensemble the models, such as bootstrap aggregation and boosting.
Future work can also look into how to initialize and optimize these NNs with bump activations well to mitigate bad initializations.
Other work to directly mitigate overfitting during training of equality separators can be explored to increase the separation of normal data from anomalies to move towards a separation like the shallow equality separator with RBF kernel, such as what initializations pair well with NN optimization.

\begin{figure}[h]
    \centering
    \begin{subfigure}{0.45\linewidth}
    \centering
    \includegraphics[width=1\linewidth, keepaspectratio]{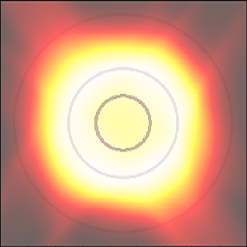}
    \caption{2 hidden layers.}
    \label{fig:SAD_Heat_ES2fb_ensemble}
    \end{subfigure}
    \hspace{0.05\linewidth}
    \begin{subfigure}{0.45\linewidth}
    \centering
    \includegraphics[width=1\linewidth, keepaspectratio]{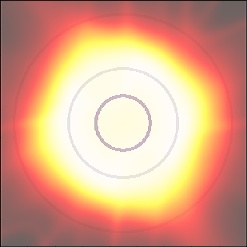}
    \caption{3 hidden layers.}
    \label{fig:SAD_Heat_ES3fb_ensemble}
    \end{subfigure}
    \caption{The heat map of the average of the output prediction of 20 equality separator neural networks with bump activations and fixed $\sigma=0.5$.}
    \label{fig:SAD_Heat_ES3fb_ensemble}
\end{figure}

\subsection{Supervised anomaly detection with NSL-KDD dataset}
\label{appendix:kdd}

We report additional details of experiments done in Section \ref{section:experiments_real_world} for NSL-KDD dataset.
Since the dataset size is significantly larger than our synthetic dataset in Section \ref{section:nonlinear_ad}, we opt to test NNs with a homogenous type of activation function (i.e. global, semi-local or local) and repeat experiments 3 times.
We report more results in Table \ref{tab:supervised_ad_kdd_FULL_TABLE}.

\paragraph{Supervised methods}
We run cross validation for SVM to obtain optimal hyperparameters as per the last section.
The one-vs-set machine (OVS) \citep{towards_open_set_recognition} uses a linear kernel SVM to define the first hyperplane then finds the second parallel hyperplane where normal data are within both hyperplanes.
However, OVS gives a 0-1 classification decision, not a continuous-valued anomaly score, so AUPR may not be the best metric here.
To produce a continuous-valued score, we introduce the equality separation paradigm by setting the hyperplane in the middle of the 2 parallel hyperplanes as the hyperplane used for equality separation.
We denote this as OVS-ES.
In our experiments, we tried both linear and RBF kernels for each of these 3 methods (SVM, OVS and OVS-ES).

There is no variance in scores because SVM solves a convex optimization problem, and all methods are SVM-based.
Thresholded one-vs-set RBF (OVS-RBF), the original thresholded one-vs-set linear (OVS-Lin) and one-vs-set linear ES (OVS-Lin-ES) have the same results and AUPR train results are low (0.703), which also reflects on the low AUPR for DoS (seen) attacks, so we ignore their results in the main text.

\paragraph{Unsupervised methods}
For unsupervised AD methods, we test on the 2 options of dealing with labels like before:
feed all the data into the algorithm, or just feed in normal data into the algorithm and remove the anomalies (because unsupervised methods assume that most/all data fed in is normal).
In the main paper, we report results of only the former.
Here, we report the former with a suffix ``-A'' and the latter with a suffix ``-N'' in Table \ref{tab:supervised_ad_kdd_FULL_TABLE}.
We report the best result in bold, but also report results one standard deviation away in bold to account for variability across experiments with different seeds.

We see OCSVM significantly outperforming all the other methods on the unseen anomalies for privilege escalation and remote access attacks.
However, OCSVMs are unable to capitalize label information, and it is expected that the supervised SVM comfortably outperforms OCSVM as well.
Once again, we notice a trade-off of poorer detection seen anomalies with the former and poorer detection of unseen anomalies for the latter.
When all data is fed into an unsupervised method, it is intuitive that the model may view anomalous data as normal and misclassify them, leading to poorer detection of seen anomalies.
However, when only normal data is fed into the unsupervised method, it is interesting to note that the resultant model is more conservative about the unseen anomalies.
Particularly, we notice IsoF obtaining the best AUPR for DoS and probe attacks.

As an aside, we note that LOF is a distance-based algorithm, which is a likely reason for its failure in this high-dimensional setting.
We also note that LOF and methods that require kernels (OCSVM) take a significantly longer time to train in these high-dimensional settings too, which we observe in our experiments as well.
Additionally, OCSVM and LOF have no randomness so there is no variance in the AUPR.

\paragraph{Binary classifiers}
We details the hyperparameters for training our binary classifier NNs below.
\begin{enumerate}
    \item We train NNs for 500 epochs with an early stopping patience of 10 epochs, with best weights restored. NNs were far from reaching the upper limit of 500 epochs.
    \item We used the Adam optimizer with learning rate of $3\times 10^{-4}$.
    \item We use 1 hidden layer with 60 neurons, about half of the input dimension. A smaller width in the hidden layer mimicks the need for dimensionality reduction in real-world datasets. This is challenging because the decision region formed will not be closed with respect to more than half of the dimensions for equality separators and halfspace separators.
    \item RBF separators and halfspace separators are trained with LL while equality separators are trained with MSE. In this appendix, we experiment with both LL and MSE and report results, noting that there is not much effect of the respective classifiers strengths and weaknesses.
    We note that equality separators trained with MSE also achieve similar performance to those trained with LL.
    \item Some results were slightly sensitive, so we repeated experiments 5 times to obtain a more reliable result.
    \item We use batch normalization before the activation function is applied (less RBF, where the only option is to apply it after).
    \item Other hyperparameters are kept the same as our synthetic experiment (e.g. leaky ReLU for global activations, Glorot uniform initializer for weights).
\end{enumerate}
Our choice was partially guided by the observation that it is easy to overfit the training data (i.e. training loss keeps decreasing while validation loss does not).
This observation corroborates with learning theory of how the generalization error bound is looser with more expressive hypothesis classes.
In particular, we experimented with 1 hidden layer with 60 neurons, 1 hidden layer with 200 neurons, 2 hidden layers with 60 neurons each and 2 hidden layers with 200 neurons each.

We also notice that using logistic loss or MSE does not affect the results too much for standard ERM binary classifiers.

\paragraph{NS}
For NS, we maintain the same training regime but detail / modify the following hyperparameters.
\begin{enumerate}
    \item The number of anomalies sampled was 10 times the size of the training dataset.
    \item We trained our equality separators with MSE. Since anomalies can be anywhere, including close to the normal data, we use MSE to be more conservative about wrong classifications, given the potential label noise when artificially sampling anomalies.
\end{enumerate}

\paragraph{SAD}
For SAD, we pre-train an autoencoder on our training data, then use the encoder as an intialization to train it for supervised AD as in \citet{DeepSAD}.
Training details are as stated:
\begin{enumerate}
    \item We use encoder dimensions of 90-60 (i.e. a hidden layer of 90 neurons between the latent layer of 60 neurons and the input/output). We employ this for all architectures because we observe underfitting without the intermediate hidden layer. We also validate performance by observing training error converge to 0 when the autoencoder only has 1 hidden layer, where the hidden layer has the same number of neurons as the input/output.
    \item Depending on the type of separator we test, the autoencoder employs either leaky ReLU (global), bump (semi-local) or RBF (local) activations homogenously.
    \item Autoencoder and classification head is trained as per \citet{DeepSAD} with hyperparameters as previously stated.
    \item $\ell_2$ regularization for classification \citep{DeepSAD} was tried with $10^{-1}$ and $10^{-6}$ weights and found to be insensitive. We reported results using a weight of $10^{-1}$.
\end{enumerate}

We note that we do not imply that RBFs are inappropriate for AD.
In fact, its locality property makes it well-suited.
However, RBFs are tricky to train \citep{goodfellow_fast_gradient_sign} and may require many neurons to work \citep{nonmonotonic_activations_mlp_thesis}, especially in higher dimensional settings (which are mostly the case in reality).
For instance, using RBFs on binary classifiers trained with MSE or on SAD have 0.000 standard deviation in the AUPR.
We noticed that learning often converged at a loss a few orders of magnitude higher than the other models, which corroborates with our intuition that NNs with RBFs are hard to train and have probably converged to a highly sub-optimal minimum.
Equality separators achieve a good balance between its ability to form closed decision regions and its capacity to learn efficiently.

\begin{table}
    \caption{Test AUPR (mean $\pm$ standard deviation) on NSL-KDD dataset, grouped according to different methods.
    The different methods are random classifiers (calculated in expectation), shallow supervised (Sup.) and unsupervised (Unsup.) methods, binary classifier NNs, Negative Sampling (-NS) and Deep Semi-Supervised Anomaly Detection (-SAD).
    For shallow supervised methods, we test linear (Lin) and RBF kernels.
    We compare using different kinds of activations in NNs: global (HS), semi-local (ES) and local (RS).
    AUPR is reported vis-a-vis each attack and all attacks (overall)..
    Suffix for unsupervised AD methods determines if only normal data is used (-N) or all data is used (-A).
    DoS attacks are seen during training, while the rest are not.
        }
        \centering
    \begin{tabular}{cl|ccccc}
        \toprule
&Model$\backslash$Attack  & DoS     & Probe  & Privilege  & Access  & Overall     \\
\midrule
 &Random  &  0.435  &  0.200  &  0.007  &  0.220  &  0.569 \\
\midrule
\parbox[t]{0.5mm}{\multirow{6}{*}{\rotatebox[origin=c]{90}{Sup.}}} & {SVM-RBF}   &  \textbf{0.959$\pm$0.000}  &  \textbf{0.787$\pm$0.000}  &  0.037$\pm$0.000  &  0.524$\pm$0.000  &  0.948$\pm$0.000 \\
&  {SVM-Lin}   &  {0.876$\pm$0.000}  &  0.494$\pm$0.000  &  0.021$\pm$0.000  &  0.151$\pm$0.000  &  {0.823$\pm$0.000} \\
&  {OVS-RBF-ES (ours)}   &  {0.956$\pm$0.000}  &  0.785$\pm$0.000  &  0.038$\pm$0.000  &  0.535$\pm$0.000  &  \textbf{0.949$\pm$0.000} \\
&  {OVS-RBF}   &  {0.717$\pm$0.000}  &  0.600$\pm$0.000  &  \textbf{0.503$\pm$0.000}  &  \textbf{0.610$\pm$0.000}  &  0.785$\pm$0.000 \\
&  {OVS-Lin-ES (ours)}   &  {0.717$\pm$0.000}  &  0.600$\pm$0.000  &  \textbf{0.503$\pm$0.000}  &  \textbf{0.610$\pm$0.000}  &  0.785$\pm$0.000 \\
&  {OVS-Lin}   &  {0.717$\pm$0.000}  &  0.600$\pm$0.000  &  \textbf{0.503$\pm$0.000}  &  \textbf{0.610$\pm$0.000}  &  0.785$\pm$0.000 \\
\midrule
\parbox[t]{0.5mm}{\multirow{6}{*}{\rotatebox[origin=c]{90}{Unsup.}}} & OCSVM-N &  0.835$\pm$0.000  &  0.849$\pm$0.000  &  0.382$\pm$0.000  &  0.745$\pm$0.000  &  0.920$\pm$0.000 \\
& OCSVM-A &  0.779$\pm$0.000 &  0.827$\pm$0.000 &  \textbf{0.405$\pm$0.000} &  \textbf{0.760$\pm$0.000} &  0.897$\pm$0.000 \\
& IsoF-N  &  \textbf{0.964$\pm$0.006}  &  \textbf{0.960$\pm$0.003}  &  0.039$\pm$0.007  &  0.438$\pm$0.015  &  \textbf{0.957$\pm$0.002} \\
& IsoF-A  &  0.765$\pm$0.073  & 0.850$\pm$0.066 &  0.089$\pm$0.044  &  0.392$\pm$0.029  &  0.865$\pm$0.031 \\
& LOF-N  &   0.759$\pm$0.000  &  0.501$\pm$0.000  &  0.046$\pm$0.000  &  0.451$\pm$0.000  &  0.824$\pm$0.000 \\
& LOF-A  &   0.495$\pm$0.000  &  0.567$\pm$0.000  &  0.039$\pm$0.000  &  0.455$\pm$0.000  &  0.718$\pm$0.000 \\
\midrule
\parbox[t]{0.5mm}{\multirow{3}{*}{\rotatebox[origin=c]{90}{ERM}}}& HS     &  0.944$\pm$0.016  &  \textbf{0.739$\pm$0.018}  &  0.016$\pm$0.006  &  0.180$\pm$0.033  &  0.877$\pm$0.010 \\
& ES (ours)  &  \textbf{0.970$\pm$0.006}  &  \textbf{0.726$\pm$0.093}  &  \textbf{0.047$\pm$0.017}  &  \textbf{0.446$\pm$0.124}  &  \textbf{0.939$\pm$0.014} \\
& RS  &  0.356$\pm$0.002  &  0.213$\pm$0.002  &  0.010$\pm$0.000  &  0.228$\pm$0.002  &  0.406$\pm$0.001 \\
\midrule
\parbox[t]{0.5mm}{\multirow{3}{*}{\rotatebox[origin=c]{90}{ERM}}}& HS, MSE     &  0.921$\pm$0.015  &  0.682$\pm$0.017  &  0.010$\pm$0.006  &  0.133$\pm$0.008  &  0.846$\pm$0.008 \\
& ES, MSE (ours)  &  \textbf{0.974$\pm$0.001}  &  \textbf{0.717$\pm$0.107}  &  \textbf{0.045$\pm$0.014}  &  \textbf{0.510$\pm$0.113}  &  \textbf{0.941$\pm$0.014} \\
& RS, MSE  &  0.356$\pm$0.000  &  0.215$\pm$0.000  &  0.010$\pm$0.000  &  0.231$\pm$0.000  &  0.406$\pm$0.000 \\
\midrule
\parbox[t]{0.5mm}{\multirow{3}{*}{\rotatebox[origin=c]{90}{NS}}} & HS-NS         &  \textbf{0.936$\pm$0.001}  &  0.642$\pm$0.030  &  0.006$\pm$0.000  &  0.139$\pm$0.001  &  0.845$\pm$0.006 \\
& ES-NS (ours)  &  \textbf{0.945$\pm$0.009}  &  \textbf{0.659$\pm$0.013}  &  \textbf{0.023$\pm$0.011}  &  0.206$\pm$0.013  &  \textbf{0.881$\pm$0.007} \\
& RS-NS         &  0.350$\pm$0.002  &  0.207$\pm$0.002  &  0.009$\pm$0.000  &  \textbf{0.223$\pm$0.002}  &  0.401$\pm$0.002
\\
\midrule
\parbox[t]{0.5mm}{\multirow{3}{*}{\rotatebox[origin=c]{90}{SAD}}} & HS-SAD         & 0.955$\pm$0.003  & 0.766$\pm$0.011  &  \textbf{0.100$\pm$0.000}	  & 0.447$\pm$0.000  &  0.935$\pm$0.007 \\
& ES-SAD (ours)  & \textbf{0.960$\pm$0.002}   & \textbf{0.795$\pm$0.004}   & 0.047$\pm$0.002   & \textbf{0.509$\pm$0.009}   & \textbf{0.952$\pm$0.002}  \\
& RS-SAD         & 0.935$\pm$0.000   &  0.678$\pm$0.000  &  0.022$\pm$0.000  &  0.142$\pm$0.000  &  0.846$\pm$0.000 \\
\bottomrule
    \end{tabular}
    \label{tab:supervised_ad_kdd_FULL_TABLE}
\end{table}

\subsection{Supervised anomaly detection with thyroid dataset}

We report additional details of experiments done in Section \ref{section:experiments_real_world} for the thyroid dataset.
All details are the same as Section \ref{appendix:kdd} unless otherwise stated.

The thyroid dataset is a medical dataset on proper functioning of the thyroid (a gland in the neck), with 21 features and 3 classes, 1 of which is normal while the anomalous classes are hyperfunction and subnormal.
For Thyroid dataset, we train on hyperfunction anomalies and test on both hyperfunction and subnormal anomalies.
\begin{enumerate}
    \item We choose 1 hidden layer with 200 neurons with batch normalization. Using a validation split of 0.15, we confirm that we do not overfit to the seen data during training.
    \item We initialize $\sigma = 7$ to increase the margin width so initializations are not as sensitive. (We also note that $\sigma=5, 10$ produce similar results.)
    \item We employ learning rate exponential decay with Adam optimizer under LL to speed up convergence, with an initial learning rate of 0.01 and decay of 0.96.
    \item We train NNs for 500 epochs with early stopping patience of 50 epochs, with best weights restored.
\end{enumerate}

We report more detailed results in Table \ref{tab:thyroid_appendix}, which includes a logistic regression baseline.
We see that equality separation is the best at detecting both seen and unseen anomalies.

\begin{table}
    \centering
    \caption{AUPR for thyroid by diagnosis (i.e. anomaly type).}
            \begin{tabular}{lcccc}
        \toprule
        Diagnosis & Logistic Regression & HS & ES (ours) & RS  \\
        \midrule
Hyperfunction (Seen)    & 0.889$\pm$0.000    & 0.907$\pm$0.013 & \textbf{0.934$\pm$0.004} & 0.030$\pm$0.002 \\
Subnormal (Unseen) & 0.340$\pm$0.001   & 0.459$\pm$0.040 & \textbf{0.596$\pm$0.007} & 0.082$\pm$0.002 \\
Overall & 0.530$\pm$0.001  & 0.615$\pm$0.032 & \textbf{0.726$\pm$0.004} & 0.103$\pm$0.001 \\
        \bottomrule
    \end{tabular}
    \label{tab:thyroid_appendix}
\end{table}

\subsection{Supervised anomaly detection with MVTec dataset}

We report additional details of experiments done in Section \ref{section:experiments_real_world} for MVTec defect dataset.
All details are the same as Section \ref{appendix:kdd} unless otherwise stated.

Each object has about 300 training samples, so we use (frozen) DINOv2 \citep{dinov2} to embed raw image data into features rather than training our own computer vision model from scratch.
DINOv2 is a foundation Vision Transformer (ViT) model pre-trained on many images, so we assume that the feature representations obtained from the model are useful, or at least more useful than a model trained from scratch.
Specifically, we use the classification CLS token embedding, since it is not patch-dependent and is widely used for downstream classification tasks.

There are far fewer data samples than the ambient dimensionality of the data, so we opt for simple methods.
Our halfspace, equality and RBF separators are all output-layered (i.e. no hidden layers) and act as linear probes, except the RBF separator which acts like SAD where DINOv2 plays the role of the encoder (of a pre-trained autoencoder) in \citet{DeepSAD} to extract meaningful features from the data.
Hence, the halfspace separator acts as our baseline.
In high dimensions, sampling to cover the whole space is much harder and computationally heavier, so we exclude negative sampling (NS).

Each anomaly class has very few anomalies (usually around 10), so we use all anomalies from 1 anomaly type during training (we pick the first type of anomaly in alphabetical order).
In other words, we do not test on seen anomalies and only test on unseen anomalies.
As the best proxy (or at least an upper bound) for performance on seen anomalies, we monitor the training loss and AUPR as well.
For instance, we observe that RBF separators do not converge and have low AUPR, which suggests that test AUPR on seen anomalies is likely to be worse.
Meanwhile, the two other models (halfspace and equality separators) had converging losses and good train AUPR.
Out of these 2 models, equality separators had better test AUPR (i.e. better separation between normal and unseen anomalies) on a whole.

\end{appendices}

\end{document}